\documentclass{article}

\usepackage{microtype}
\usepackage{graphicx}
\usepackage{algorithm}
\usepackage{algpseudocode}
\usepackage{subcaption}
\usepackage{booktabs} 
\usepackage{multirow}

\usepackage[preprint,nohyperref]{icml2026}

\usepackage{amsmath}
\usepackage{amssymb}
\usepackage{mathtools}
\usepackage{amsthm}
\usepackage[capitalize,noabbrev]{cleveref}

\theoremstyle{plain}
\newtheorem{theorem}{Theorem}[section]
\newtheorem{proposition}[theorem]{Proposition}
\newtheorem{lemma}[theorem]{Lemma}
\newtheorem{corollary}[theorem]{Corollary}
\theoremstyle{definition}

\newtheorem{assumption}[theorem]{Assumption}
\theoremstyle{remark}

\usepackage[textsize=tiny]{todonotes}

\icmltitlerunning{Enhance the Safety in Reinforcement Learning by ADRC Lagrangian Methods}

\begin{document}

\twocolumn[
  \icmltitle{Enhance the Safety in Reinforcement Learning by ADRC Lagrangian Methods}



  \icmlsetsymbol{equal}{*}

  \begin{icmlauthorlist}
    \icmlauthor{Mingxu Zhang}{equal,yyy}
    \icmlauthor{Huicheng Zhang}{equal,yyy}
    \icmlauthor{Jiaming Ji}{sch}
    \icmlauthor{Yaodong Yang}{sch}
    \icmlauthor{Ying Sun}{yyy}
  \end{icmlauthorlist}

  \icmlaffiliation{yyy}{AI Thrust, The Hong Kong University of Science and Technology (Guangzhou), Guangzhou, China}
  \icmlaffiliation{sch}{School of Artificial Intelligence, Peking University, Beijing, China}

  \icmlcorrespondingauthor{Ying Sun}{sunyinggilly@gmail.com}

  \printAffiliationsAndNotice{}

  \icmlkeywords{Machine Learning, ICML}

  \vskip 0.3in
]
\begin{abstract}
Safe reinforcement learning (Safe RL) seeks to maximize rewards while satisfying safety constraints, typically addressed through Lagrangian-based methods. However, existing approaches, including PID and classical Lagrangian methods, suffer from oscillations and frequent safety violations due to parameter sensitivity and inherent phase lag. To address these limitations, we propose ADRC-Lagrangian methods that leverage Active Disturbance Rejection Control (ADRC) for enhanced robustness and reduced oscillations. Our unified framework encompasses classical and PID Lagrangian methods as special cases while significantly improving safety performance. Extensive experiments demonstrate that our approach reduces safety violations by up to 74\%, constraint violation magnitudes by 89\%, and average costs by 67\%, establishing superior effectiveness for Safe RL in complex environments.
\end{abstract}

\section{Introduction}
Reinforcement Learning (RL) aims to maximize rewards as agents interact with environments, finding applications in fields such as robotics \citep{sun2023zeroshotmultilevelfeaturetransmission,luo2024precisedexterousroboticmanipulation,li2024marladonacooperativeteam} and the post-training of Large Language Models (LLMs) \citep{RLHF, RLAIF, DPO}. 
However, in real-world scenarios like autonomous driving \citep{auto-driving}, safety requirements are often of paramount importance. 
It is essential not only to maximize rewards but also to ensure compliance with safety constraints. 
To address this challenge, Safe RL \citep{garcia2015comprehensive} has emerged as a paradigm dedicated to reliable and robust policy learning in complex and dynamic environments. Safe RL is typically formulated as a Constrained Markov Decision Process (CMDP). Among the various approaches for solving CMDPs, Lagrangian methods play a pivotal role by transforming constrained optimization problems into unconstrained ones through the introduction of a dual variable, the Lagrange multiplier. This transformation enables the adaptation of any RL algorithm into a Safe RL framework \citep{cpo,chow2018lyapunov,chow2019lyapunovbasedsafepolicyoptimization}, leading to the development of numerous novel Safe RL algorithms \citep{liu2024ddmlagdiffusionbaseddecisionmaking,chen2024adaptiveprimaldualmethodsafe}.

Classical Lagrangian updates can be interpreted as pure integral controllers on the constraint violation signal. While simple, this mechanism reacts slowly to the rapid distributional shifts caused by policy updates and the stochasticity of cost estimates, leading to lag, overshoot, and persistent oscillations in safety performance. Attempts to mitigate these issues with PID-based extensions \citep{pidlagrange} reduce oscillations by adding proportional and derivative terms, but their behavior remains fragile: performance is highly sensitive to the chosen gains and rarely transfers robustly across tasks \citep{pidcontrolbook1,PIDcontrolbook2}. These limitations point to a deeper challenge—existing methods lack a way to explicitly counteract the drifting disturbances that underlie oscillatory training.

To overcome this challenge, we turn to the broader toolbox of \emph{adaptive control}, whose central goal is to maintain stability and performance under unknown and time-varying dynamics. Among the many adaptive strategies, Active Disturbance Rejection Control (ADRC) \citep{han1998adrc,PID_ADRC} is particularly well suited to the Safe RL setting. Unlike classical adaptive methods that often require a parametric model or extensive gain tuning, ADRC treats all uncertainty—including model error, noise, and nonstationarity—as a lumped disturbance, and employs a lightweight observer to estimate and cancel it online. This observer-based design makes ADRC both model-free and robust to changing dynamics, while its reliance only on observable signals (such as cost returns) makes it a natural match for reinforcement learning. In contrast, alternatives such as MRAC or Lyapunov-based adaptive schemes typically assume access to richer state information or stronger structural knowledge, which is impractical in high-dimensional RL environments.

Building on this idea, we introduce ADRC Lagrangian methods, which augment the classical dual update with an observer that estimates and cancels the disturbance acting on the constraint return. By combining this observer with a smooth reference trajectory for the cost threshold, our update suppresses transient overshoot while directly compensating for nonstationarity and noise. The resulting method is simple to implement, model-free, and optimizer-agnostic, yet it fundamentally changes the dynamics of constraint regulation: we derive a theoretical lower bound on the observer gain that provides safe defaults across diverse environments, we show that classical and PID Lagrangian updates emerge as strict special cases of our formulation, and we establish through frequency-domain analysis that our approach achieves smaller disturbance-estimation error and reduced phase lag. Empirically, these properties translate into substantial improvements in safety throughout training: on OmniSafe benchmarks our method reduces violation rates by up to 74\%, lowers violation magnitudes by 89\%, and decreases average costs by 67\%, all while maintaining competitive reward performance. These results demonstrate that bringing the ADRC perspective into Safe RL yields both principled theoretical guarantees and significant empirical gains.

In conclusion, our main contributions are:
\begin{itemize}
    \item We are the first to introduce ADRC into Safe RL, dynamically adjusting the Lagrange multiplier to improve constraint satisfaction and training stability.
    \item We theoretically establish that both PID and classical Lagrangian methods are special cases of our ADRC Lagrangian methods. Moreover, through frequency-domain analysis, we demonstrate that our method significantly reduces phase lag compared to traditional approaches, leading to faster and more stable constraint satisfaction.
    \item Comprehensive experiments validate the effectiveness of our method, showing significant improvements in reducing oscillations during training across diverse benchmarks.
\end{itemize}

\section{Related Work}

\paragraph{Safe RL}
Safe reinforcement learning (RL) aims to find optimal policies that maximize rewards while satisfying safety constraints \citep{garcia2015comprehensive,cpo,wachi2020safe,pcpo}. Common approaches include safe exploration techniques to ensure safety during training \citep{sui2015safe,wang2023enforcing} and the primal-dual framework, which employs Lagrangian multipliers to address constrained optimization \citep{ray2019benchmarking,ding2020natural,chow2018lyapunov,chow2019lyapunovbasedsafepolicyoptimization}. Recent advances have improved the tuning of these multipliers through gradient-based methods \citep{DDPG,rcpo,zhang2020first}, PID-based updates \citep{pidlagrange}, adaptive primal-dual methods \citep{chen2024adaptiveprimaldualmethodsafe}, and variational inference approaches \citep{liu2022constrained,huang2022constrained}, enhancing algorithm stability and performance \citep{yao2024constraint}. On-policy algorithms can be broadly categorized into Lagrangian methods, such as PDO \citep{chow2018lyapunov}, and convex optimization methods like CPO \citep{cpo} and CVPO \citep{liu2022constrained}. Recent developments, including APPO \citep{dai2023augmented} and CUP \citep{cup}, specifically address oscillatory behaviors and improve constraint feasibility. Furthermore, accurate estimation of objective and constraint functions is crucial, as it significantly influences the efficiency and reliability of policy updates \citep{altman2021constrained}. Additionally, methods targeting training stability, such as policy inertia learning \citep{chen2021addressing} and soft-switching gradient manipulation \citep{gu2024balance}, have effectively reduced oscillations, highlighting their importance for Safe RL.

\paragraph{Control Theory} 
Control theory includes traditional controllers such as Proportional-Integral-Derivative (PID) controllers \citep{aastrom2006pid,ang2005pid}, which, despite their widespread use, are sensitive to parameter variations and external disturbances \citep{pidcontrolbook1,PIDcontrolbook2}. To overcome these limitations, Model Reference Adaptive Control (MRAC), which is a milestone of adaptive control, has been developed, enabling dynamic parameter adjustment to maintain performance under system uncertainties \citep{nguyen2018model,parks1966liapunov}. By incorporating the MIT rule \citep{MITrule}, MRAC-PID controllers leverage backpropagation-inspired methods \citep{bp} to adapt PID parameters in real time, addressing control gain sensitivity \citep{mrac_PID,kungwalrut2011design}. Furthermore, \citet{zhang2019theory} it has been demonstrated that PID parameters could be selected within a specific manifold to ensure global system stabilization with exponential error convergence. In parallel, Active Disturbance Rejection Control (ADRC) has emerged as a robust alternative for managing uncertainties and disturbances. First introduced by \citet{han1998adrc}, ADRC has been further developed to enhance its applicability and theoretical underpinnings \citep{han2009pid}. Recent advancements include applications in nonlinear systems \citep{guo2017active} and rigorous stability analysis using Lyapunov functions \citep{Zhong_2020}, solidifying ADRC as an effective and versatile control strategy.

\section{Preliminaries}
\label{sec:preliminaries}
\paragraph{Safe Reinforcement Learning}
A Markov Decision Process (MDP) $\mathcal{M}$ \citep{mdp} is defined by the tuple $(\mathcal{S}, \mathcal{A}, R, \mathbb{P}, \mu, \gamma)$, where $\mathcal{S}$ and $\mathcal{A}$ denote state and action spaces, $R$ is the reward function, $\mathbb{P}(s' \mid s,a)$ is the state transition probability, $\mu$ is the initial state distribution, and $\gamma \in (0,1)$ is the discount factor. A parameterized stationary policy $\pi_{\theta}(a \mid s)$ specifies action probabilities given state $s$. The goal of reinforcement learning (RL) is to maximize the expected return:
\begin{equation}
J(\pi_{\theta}) = \mathbb{E}_{s \sim \mu}\left[\sum_{t=0}^{\infty}\gamma^t r_{t+1}\right].
\end{equation}
Safe RL is typically formulated as a constrained MDP (CMDP) \citep{cmdp}, which extends MDPs with constraints defined by cost functions $c_i : \mathcal{S} \times \mathcal{A} \rightarrow \mathbb{R}$ and thresholds $d_i$. The cost return under policy $\pi_{\theta}$ is:
\begin{equation}
J_{c_i}(\pi_{\theta}) = \mathbb{E}{\pi_{\theta}}\left[\sum_{t=0}^{\infty} \gamma^t c_i(s_t,a_t)\right].
\end{equation}
Safe RL aims to find an optimal policy:
\begin{equation}
\pi^* = \arg\max_{\pi_{\theta}} J(\pi_{\theta}) \quad \text{s.t.} \quad J_{c_i}(\pi_{\theta}) \leq d_i, \forall i.
\end{equation}

\paragraph{Lagrangian Methods}
\label{sec:preliminary_lag}
In constrained optimization problems such as those in Safe RL, the goal is to maximize the objective function while satisfying constraints. A common approach is to apply the Lagrangian method. Specifically, for a CMDP with a single cost constraint, denote $d$ as the cost threshold, and $\lambda \geq 0$ as the Lagrangian multiplier, we define the Lagrangian function as:
\begin{equation}
\mathcal{L}(\theta, \lambda) = J(\pi_{\theta}) - \lambda (J_c(\pi_{\theta}) - d),
\end{equation}
The optimal solution aims to maximize the Lagrangian with respect to $\theta$ while minimizing it with respect to the multiplier $\lambda$. To achieve this, we apply a gradient-based approach to update $\lambda$ iteratively. Specifically, we define the constraint violation as $g(\pi_{\theta}) := J_c(\pi_{\theta}) - d$. Denote $\alpha > 0$ as the learning rate controlling the update step size; we perform gradient ascent on $\lambda$ with respect to $\mathcal{L}$:
\begin{equation}
\dot{\lambda} = \alpha g(\pi_{\theta}),
\end{equation}
Thus, optimizing $\lambda$ reduces to a standard gradient ascent problem:
\begin{equation}
\dot{\lambda} = \alpha \frac{\partial \mathcal{L}(\theta, \lambda)}{\partial \lambda} = \alpha g(\pi_{\theta}),
\end{equation}
Discretizing over time, the Lagrangian multiplier is updated iteratively by:
\begin{equation}
\lambda_t = \lambda_{t-1} + \alpha g(\pi_{\theta_{t-1}}),
\end{equation}
or equivalently, by summing constraint violations over time:
\begin{equation}
\lambda_t = \lambda_0 + \alpha \sum_{\tau=0}^{t-1} g(\pi_{\theta_\tau}) \approx \alpha \int_0^t g(\pi_{\theta_\tau}) d\tau.
\end{equation}
This shows that classical Lagrangian methods implement a pure Integral (I) controller on the constraint violation signal $g(\pi_\theta)$. With this view, to reduce oscillations during training, PID Lagrangian methods \citep{pidlagrange} generalize the integral control by adding proportional (P) and derivative (D) terms into the dynamics of $\lambda$:
\begin{equation}
\dot{\lambda} = \alpha g(\pi_{\theta}) + \beta \dot{g}(\pi_{\theta}) + \gamma \ddot{g}(\pi_{\theta}),
\end{equation}
where $\alpha$, $\beta$, and $\gamma$ are positive coefficients controlling the strength of the I, P, and D terms respectively.

Similarly, integrating this equation over time, the resulting PID update law for the multiplier is:
\begin{equation}
\label{eq:PID_law}
\lambda_t = \left( K_p g(\pi_{\theta_t}) + K_i \int_0^t g(\pi_{\theta_\tau}) d\tau + K_d \dot{g}(\pi_{\theta_t}) \right)_+,
\end{equation}
where $K_p$, $K_i$, $K_d$ are proportional, integral, and derivative gains that need to be tuned carefully.

\paragraph{Active Disturbance Rejection Control}
\label{sec:preliminary_adrc}
Compared with PID controller, Active Disturbance Rejection Control (ADRC) \citep{han1998adrc, PID_ADRC} provides a more adaptive and resilient alternative. Unlike traditional PID control, ADRC explicitly estimates and compensates for unknown disturbances through an observer-based framework, reducing reliance on precise model knowledge and hyperparameter sensitivity. The core component of ADRC is the Extended State Observer (ESO), which is designed to simultaneously estimate both the internal system states and the total disturbance affecting the system dynamics. By accurately reconstructing the disturbance in real time, the control input can proactively reject its influence, significantly enhancing system stability and performance. In practical Safe RL scenarios, where exact system dynamics are unknown and only observable quantities like costs are available, a reduced-order ESO design is commonly employed. In addition to disturbance estimation, to achieve smoother system behavior and better transient performance, the control strategy can also incorporate a designed reference trajectory that guides the evolution of the system states toward the desired setpoints. 

\section{Method}

\subsection{Closed-loop System Representation of Safe RL}
Lagrangian-based Safe RL can be viewed as a feedback system: the policy affects the cumulative cost, which drives the Lagrange multiplier that in turn influences the policy update. We capture this interaction in a simple closed-loop form,
\begin{equation}
\label{safeRL_as_closed_loop_system}
\begin{cases}
x_1 = J_c, \\
\dot{x}_1 = x_2, \\
\dot{x}_2 = f(x_1, x_2, t) + u(t), \\
u(t) = \lambda_t,
\end{cases}
\end{equation}
where $x_1$ is the cumulative cost, $x_2$ its derivative, $f(\cdot)$ aggregates all unknown and time-varying effects, and $u(t)$ is the control input given by the multiplier $\lambda_t$. This formulation highlights the root cause of oscillations in existing methods: the dynamics drift as the policy changes, while the multiplier behaves like an integral controller that lags behind disturbances. 

Our goal is to replace this fragile mechanism with an observer-augmented update inspired by ADRC. By explicitly estimating the total disturbance from cost signals and compensating it in real time, while guiding constraint satisfaction through a smooth reference trajectory, our method achieves faster and more stable regulation than classical or PID-based approaches. The designs of the observer and reference signal are presented next.

\subsection{Arranging a Transient Process}
\label{sec:transient_process}

In Safe RL, the dual update implicitly drives the cumulative cost $x_1$ toward the safety threshold d. To formalize this objective, we introduce a reference signal $y^*(t)$, which represents the target trajectory that $x_1$ should ideally follow. Since the ultimate goal is constraint satisfaction, the natural choice of reference is a constant at the threshold:
\begin{equation}
\label{eq:reference_signal}
y^*(t) = d.
\end{equation}
However, tracking this signal directly can be problematic in practice. At the beginning of training, policies are usually far from safe, so the gap $x_1(0)-d$ is large. Forcing the multiplier to eliminate this gap immediately leads to abrupt updates, which amplify estimator noise and policy nonstationarity into overshoot and repeated violations. Empirically, this appears as sharp cost spikes in early training and oscillatory swings of $\lambda_t$, even when the constraint is ultimately feasible.

To prevent these instabilities, we need to arrange a transient process that gradually shrinks the effective budget from the current cost level toward $d$ with critically damped dynamics. In the Safe RL context, this corresponds to a smooth budget schedule: early training permits a controlled violation margin that decays over time, enabling exploration while guiding the system toward feasibility.

Concretely, we filter $y^*(t)$ through a second-order system:
\begin{equation}
\label{eq:ODE}
\ddot{r} = -2c_r \dot{r} - c_r^2(r-d), 
\qquad r(0)=x_1(0),\ \dot{r}(0)=x_2(0),
\end{equation}
where $c_r>0$ controls the tightening speed. The resulting reference trajectory is
\begin{equation}
\label{eq:trancient_process}
r(t)=d+\big(x_1(0)-d\big)e^{-c_rt}+\big(x_2(0)+c_r(x_1(0)-d)\big)te^{-c_rt},
\end{equation}
which starts from the current cost level and slope, then converges smoothly and non-oscillatorily to $d$. This shaped reference avoids abrupt enforcement in early training, stabilizes the multiplier dynamics, and reduces phase lag in constraint regulation. 

A detailed derivation is provided in Appendix~\ref{sec:solve_ODE}.

\subsection{Extended State Observation for Multiplier Updates}
\label{sec:ESO}

Training in Safe RL is inherently noisy and nonstationary: the measured cost fluctuates due to stochastic transitions, estimation error, and abrupt policy changes. If the multiplier reacts to these raw signals directly, it amplifies noise and tends to oscillate. What is missing is an online estimate of the \emph{unmodeled dynamics}—the effective disturbance $f(x_1,x_2,t)$ in Eqn.~\ref{safeRL_as_closed_loop_system}—so that the update can distinguish genuine constraint trends from transient fluctuations.

To this end, we borrow the idea of an \emph{extended state observer} (ESO) from adaptive control, but use it in the simplest reduced-order form~\cite{PID_ADRC} suitable for RL. The ESO maintains an auxiliary state $\xi$ that is updated alongside observed costs:
\begin{equation}
\label{eq:ESO_Form}
\begin{cases}
\dot{\xi} = -\omega_o \xi - \omega_o^2 x_2 - \omega_o u, \\
\hat{f} = \xi + \omega_o x_2,
\end{cases}
\end{equation}
where $\hat f$ serves as a running estimate of the disturbance and $\omega_o>0$ is a gain controlling how aggressively it adapts. Intuitively, $\hat f$ behaves like a bias-correction term that smooths the effect of noise and policy shifts before they reach the multiplier.

With this estimate, the control input $u(t)$ (i.e., the multiplier update) is designed to track the transient reference $r(t)$ using proportional, derivative, and disturbance feedback:
\begin{equation}
\label{eq:u_in_ESO}
u(t) = k_{ap}(x_1 - r) + k_{ad}(x_2 - \dot{r}) + \hat{f} - \ddot{r}.
\end{equation}
Substituting~\eqref{eq:ESO_Form} into~\eqref{eq:u_in_ESO}, and identifying $u(t)$ with $\lambda_t$, we obtain the ADRC-based update law:
\begin{equation}
\label{eq:ADRC_Lag_Law}
\begin{aligned}
\lambda_t = (k_{ap} + \omega_o k_{ad})(x_1 - r) + (k_{ad} + \omega_o)(x_2 - \dot{r}) &
\\+ \omega_o k_{ap} \int_0^t (x_1(\tau) - r(\tau)) d\tau - \ddot{r}.
\end{aligned}
\end{equation}

\begin{proposition}
Classical PID Lagrangian methods are a special case of~\eqref{eq:ADRC_Lag_Law}. Under a specific mapping between $(K_p,K_d,K_i)$ and $(k_{ap},k_{ad},\omega_o)$, the ADRC update reduces exactly to the PID rule in Eqn.~\ref{eq:PID_law}.
\end{proposition}

Thus ADRC can be seen as a strict generalization of PID Lagrangian: in addition to $P$, $I$, and $D$ terms, it continuously estimates the disturbance and compensates for it in real time. This additional degree of adaptivity reduces phase lag and improves robustness to noise and nonstationarity, as analyzed in Sec.~\ref{sec:lower_bound}.

\subsection{The Lower Bound of Optimal Parameters}
\label{sec:lower_bound}
A central challenge in Safe RL is parameter sensitivity: the same multiplier update rule can behave well in one environment but oscillate or diverge in another. This is particularly acute for PID-based Lagrangian methods, whose stability depends heavily on hand-tuned gains. To make ADRC practical in Safe RL, we aim for a principled condition that guarantees stability and bounded estimation error across environments, thereby removing the need for brittle manual tuning.

We characterize the uncertainty in the closed-loop dynamics of Eqn.~\ref{safeRL_as_closed_loop_system} by bounding how disturbances can depend on the current state and vary over time. Following \citet{PID_ADRC}, we consider
\begin{equation}
\label{eq:non_linear_function_class}
\begin{aligned}
&f(x_1,x_2,t) = h(x_1,x_2) + w(t), \\
&\Big|\tfrac{\partial h}{\partial x_1}\Big| \leq L_1, \Big|\tfrac{\partial h}{\partial x_2}\Big| \leq L_2, |w(t)|,|\dot w(t)| \leq L_3,    
\end{aligned}
\end{equation}
where $L_1$ and $L_2$ bound the sensitivity of disturbances to the cost $x_1$ and its rate $x_2$, while $L_3$ bounds the magnitude and variation of exogenous fluctuations such as noise or nonstationarity. This setting captures both state-dependent effects and purely time-dependent variations.

Within this class, the admissible observer gains $\omega_o$ are constrained by a characteristic polynomial manifold
\begin{equation}
\label{eq:manifold}
\Omega=\left\{\omega\in \mathbb{R} \,\middle|\,
n_0\omega^4+n_1\omega^3+n_2\omega^2+n_3\omega+n_4=0\right\}, 
\end{equation}
where the coefficients $n_i$ depend on $(k_{ap},k_{ad})$ and the constants $(L_1,L_2,L_3)$. We define
\[
\bar{\omega_o} =
\begin{cases}
\max \{ \omega \mid \omega \in \Omega \}, & \text{if } \Omega \neq \emptyset, \\ 
0, & \text{otherwise}.
\end{cases}
\]

The final lower bound ensuring stability and bounded estimation error is therefore
\begin{equation}
\label{eq:lower_bound_of_omega}
\omega_o^* = \max \Big\{ \bar \omega_o,\ 0,\ \tfrac{L_1 - k_{ap}}{k_{ad}},\ L_2 - k_{ad} \Big\}.
\end{equation}
Intuitively, this bound ensures that the ESO adapts quickly enough to track disturbances while remaining stable. In practice, this means practitioners only need to set $\omega_o > \omega_o^*$, removing the trial-and-error search that plagues PID tuning.

Beyond stability, we also analyze how fast and accurately ADRC reacts compared to classical integral updates. Let $e(t) = x_1(t) - r(t)$ be the tracking error, $\hat f_I = k_i \int_0^t e(\tau)\, d\tau$ the disturbance estimate from integral control, and $\hat f$ the ADRC estimate. Their estimation errors are
\[
e_{f_I}(t) = \hat f_I - f(x_1,x_2,t), \qquad 
e_f(t) = \hat f - f(x_1,x_2,t).
\]
Denote the Laplace transforms of $e$, $e_f$, $e_{f_I}$, $\hat f_I$, and $\hat f$ by $E(s)$, $E_f(s)$, $E_{f_I}(s)$, $\hat F_I(s)$, and $\hat F(s)$ respectively, and let $F(s)$ be the transform of the disturbance $f$. Then $G_{e_f}(s)$ and $G_{e_{f_I}}(s)$ are the transfer functions from $F(s)$ to $E_f(s)$ and $E_{f_I}(s)$.

We establish the following result:

\begin{theorem}
\label{theorem:3}
Suppose $\omega_o > \omega_o^*$. Then, for any frequency $\omega$, ADRC Lagrangian achieves uniformly lower disturbance estimation error than integral control:
\[
\frac{|G_{e_f}(i\omega)|}{|G_{e_{f_I}}(i\omega)|} < 1.
\]
Moreover, if $\omega_o > \max\Big\{ \tfrac{k_{ap} - \omega^2}{k_{ad}}, \ \omega_o^* \Big\}$, ADRC Lagrangian also exhibits smaller phase lag:
\[
\arg(G_{e_f}(i\omega)) < \arg(G_{e_{f_I}}(i\omega)).
\]
\end{theorem}

In words, the frequency-domain guarantees imply concrete benefits in Safe RL training. A smaller disturbance estimation error means the multiplier update is less sensitive to transient noise and policy-induced fluctuations, avoiding spurious reactions to single-batch variability. A smaller phase lag means constraint violations are corrected earlier, rather than several updates later, reducing the amplitude and duration of overshoot. Together, these properties yield training dynamics with fewer repeated safety violations, smoother multiplier trajectories, and faster convergence to the feasible region. Detailed derivations are provided in Appendix~\ref{appendix:proofs}.

\paragraph{Stability Analysis under Surrogate Dynamics.}
Appendix~\ref{sec:theory-corrected} analyzes the ADRC multiplier update using a \emph{surrogate} closed-loop model $x_1(t)=J_c(\pi(t))$ tracking $r(t) \to d$. By treating policy non-stationarity and noise as a lumped disturbance $f(t)$ within this simplified dynamic, we derive an ISS-type tracking bound $\mathcal{O}(L_f/\omega_o)$ and bounded time-average constraint violations. Crucially, we establish a scaling law linking the disturbance bound $L_f$ to Safe RL hyperparameters $(\delta, N, \Delta t)$, providing a principled basis for parameter tuning.

\begin{figure*}
    \centering
    \includegraphics[width=\linewidth]{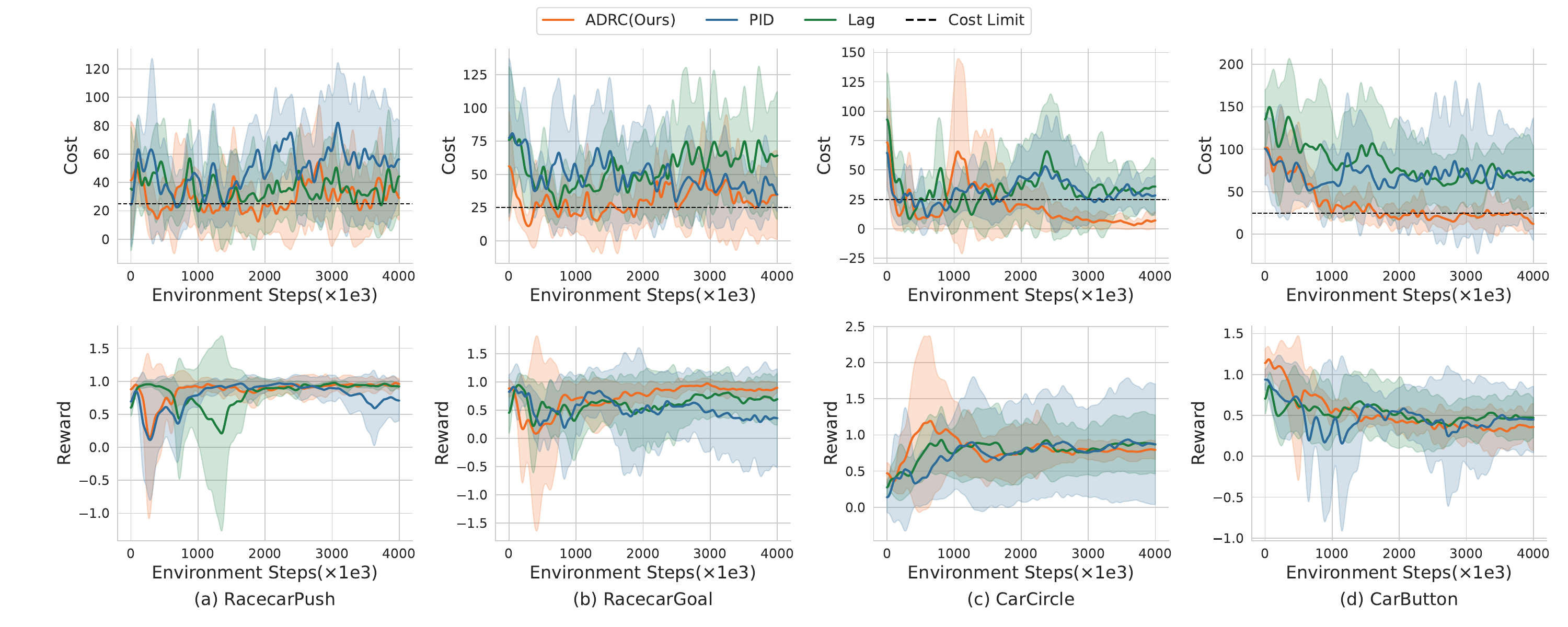}
    \caption{The training curves of PPO with various Lagrangian methods (denoted as CPPOLag, CPPOPID, CPPOADRC) across different tasks, showing episodic returns and costs over five random seeds. Solid lines represent mean values, while shaded areas denote variance. CPPOADRC demonstrates a shorter  phase and lower costs compared to baselines, while achieving competitive rewards. Additional results are provided in Appendix~\ref{sec:More_results}.}
    \label{fig:CPPO}
\end{figure*}

\subsection{ADRC Lagrangian Methods in Safe RL}
We now describe how ADRC Lagrangian methods are applied in Safe RL training. The central idea is to replace the hand-tuned, noise-sensitive dual update with an observer-augmented rule that adapts automatically to the evolving learning dynamics.

In practice, the Lagrangian multiplier $\lambda$ serves as a penalty knob: when the observed cumulative cost approaches the safety threshold, $\lambda$ should increase to discourage unsafe behavior; when costs fall well below the threshold, $\lambda$ can relax to allow more exploration. ADRC realizes this adaptivity by updating $\lambda$ according to Eqn.~\ref{eq:ADRC_Lag_Law}, where the observer $\hat f$ continuously estimates the effect of unmodeled disturbances and compensates for them in real time. This makes multiplier updates less myopic and more responsive than classical dual ascent or PID rules.

A key question is how to choose the observer gain $\omega_o$. As shown in Sec.~\ref{sec:lower_bound}, $\omega_o$ must exceed a lower bound $\omega_o^*$ that depends on environment sensitivities $(L_1,L_2)$. Since these quantities are unknown beforehand, we approximate them online using finite differences of observed costs:
\begin{equation}
\label{eq:calculate_L1_L2}
\begin{aligned}
L_1 &\approx \max_t \left| \frac{\ddot{x}_1(t+1)-\ddot{x}_1(t)}{x_1(t+1)-x_1(t)} \right|, \\
L_2 &\approx \max_t \left| \frac{\ddot{x}_1(t+1)-\ddot{x}_1(t)}{x_2(t+1)-x_2(t)} \right|,
\end{aligned}
\end{equation}
where $x_1$ is the cumulative cost and $x_2$ its derivative. These estimates allow us to adaptively compute $\omega_o$ via Eqns.~\ref{eq:manifold} and~\ref{eq:lower_bound_of_omega}, ensuring stability without manual tuning even as the training dynamics evolve, overcoming the parameter sensitivity and environment-specific retuning that plague PID-based and classical Lagrangian methods.

Finally, large values of $\lambda$ may destabilize policy learning by forcing aggressive gradient steps. Following the previous method \citep{pidlagrange}, we adopt a rescaled optimization objective:
\[
\theta^*(\lambda) = \arg\max_{\theta} \frac{1}{1+\lambda} \big(J(\pi_\theta) - \lambda J_c(\pi_\theta)\big),
\]
which tempers the effect of $\lambda$ while preserving constraint enforcement. This adjustment yields smoother policy updates and makes ADRC Lagrangian straightforward to integrate into standard Safe RL algorithms. The full training procedure is summarized in Algorithm~\ref{alg:ADRC_Multipliers}.

\begin{table*}[t]
\centering
\small
\caption{The proportion of constraint violations during training, the average violation magnitude, and the average cost for various algorithms. Our ADRC method consistently outperforms others.}
\resizebox{\textwidth}{!}{
\begin{tabular}{lcccccc}
\toprule
\textbf{Algorithm} 
& \multicolumn{3}{c}{\textbf{CarPush}} 
& \multicolumn{3}{c}{\textbf{RacecarGoal}} \\
\cmidrule(lr){2-4} \cmidrule(lr){5-7}
& \textbf{Vio. Rate (\%)} & \textbf{Magnitude} & \textbf{Avg. Cost} 
& \textbf{Vio. Rate (\%)} & \textbf{Magnitude} & \textbf{Avg. Cost} \\
\midrule
CPPOLag  & 68.45 ± 18.98 & 21.48 ± 16.81 & 43.38 ± 18.71 & 80.87 ± 19.17 & 31.18 ± 18.99 & 54.24 ± 21.01 \\
CPPOPID  & 62.36 ± 10.66 & 12.40 ± 3.08 & 33.74 ± 3.88 & 72.30 ± 24.96 & 27.11 ± 15.58 & 49.02 ± 18.75 \\
CPPOADRC & \textbf{42.23 ± 15.34} & \textbf{11.68 ± 7.99} & \textbf{29.16 ± 10.19} & \textbf{47.08 ± 21.58} & \textbf{12.31 ± 9.34} & \textbf{30.12 ± 12.74} \\
\midrule
DDPGLag  & 65.52 ± 4.60 & 12.53 ± 0.72 & 35.03 ± 0.53 & 71.44 ± 9.47 & 18.47 ± 4.13 & 40.76 ± 5.00 \\
DDPGPID  & 52.43 ± 0.12 & 7.25 ± 0.93 & 28.35 ± 0.44 & 72.05 ± 4.09 & 17.93 ± 2.81 & 40.29 ± 2.99 \\
DDPGADRC & \textbf{47.36 ± 1.90} & \textbf{2.88 ± 0.57} & \textbf{21.55 ± 0.29} & \textbf{68.41 ± 5.77} & \textbf{17.81 ± 1.34} & \textbf{39.41 ± 2.14} \\
\midrule
TD3Lag   & 80.94 ± 5.87 & 20.68 ± 3.94 & 43.43 ± 4.23 & 75.26 ± 2.13 & 19.93 ± 1.40 & 42.57 ± 1.59 \\
TD3PID   & 70.47 ± 10.44 & 17.00 ± 1.29 & 38.90 ± 2.83 & 73.31 ± 6.06 & 19.19 ± 2.72 & 41.57 ± 3.22 \\
TD3ADRC  & \textbf{40.62 ± 8.51} & \textbf{2.85 ± 0.31} & \textbf{20.65 ± 2.46} & \textbf{71.24 ± 3.00} & \textbf{17.55 ± 3.57} & \textbf{39.71 ± 3.94} \\
\midrule
TRPOLag  & 54.86 ± 6.74 & 10.46 ± 4.56 & 30.97 ± 4.69 & 64.31 ± 23.37 & 22.89 ± 18.69 & 43.67 ± 21.56 \\
TRPOPID  & 44.79 ± 2.84 & 7.34 ± 1.31 & 25.84 ± 0.88 & 49.33 ± 15.00 & 11.70 ± 7.07 & 30.94 ± 8.81 \\
TRPOADRC & \textbf{29.11 ± 3.70} & \textbf{3.44 ± 1.21} & \textbf{20.48 ± 0.99} & \textbf{34.03 ± 8.06} & \textbf{6.16 ± 2.37} & \textbf{22.02 ± 3.61} \\
\bottomrule
\end{tabular}
}
\label{table:carpush_racecargoal}
\end{table*}

\begin{table*}[htbp]
\centering
\small
\caption{Performance comparison under RacecarGoal for TRPO and PPO with different $c_r$ values.}
\resizebox{\textwidth}{!}{
\begin{tabular*}{\textwidth}{@{\extracolsep{\fill}}lccc|ccc@{}}
\toprule
\multirow{2}{*}{\textbf{Method}} 
& \multicolumn{3}{c|}{\textbf{TRPO (RacecarGoal)}} 
& \multicolumn{3}{c}{\textbf{PPO (RacecarGoal)}} \\
\cmidrule(lr){2-4} \cmidrule(lr){5-7}
& \textbf{Vio. Rate (\%)} & \textbf{Magnitude} & \textbf{Avg. Cost}
& \textbf{Vio. Rate (\%)} & \textbf{Magnitude} & \textbf{Avg. Cost} \\
\midrule
Lag          & 87.33 & 37.36 & 61.53 & 84.35 & 30.16 & 53.38 \\
PID          & 44.60 & 7.04  & 26.15 & 79.25 & 23.88 & 46.44 \\
\midrule
$c_r=0.05$   & \textbf{33.98} & \textbf{5.25} & \textbf{20.83} & \textbf{34.38} & \textbf{3.90} & \textbf{22.45} \\
$c_r=0.1$    & \textbf{29.05} & \textbf{3.44} & \textbf{18.95} & \textbf{33.08} & \textbf{5.78} & \textbf{21.22} \\
$c_r=0.15$   & \textbf{31.25} & \textbf{5.34} & \textbf{21.16} & \textbf{52.88} & \textbf{10.69} & \textbf{31.37} \\
$c_r=0.2$    & \textbf{40.65} & \textbf{6.10} & \textbf{23.67} & \textbf{48.95} & \textbf{8.77} & \textbf{26.91} \\
$c_r=0.25$   & \textbf{38.50} & \textbf{6.71} & \textbf{23.40} & \textbf{62.83} & \textbf{13.37} & \textbf{33.99} \\
\bottomrule
\end{tabular*}
}
\label{table:racecargoal_c_r}
\end{table*}

\section{Experiments}

In this section, we conduct a series of experiments to evaluate the performance of our ADRC Lagrangian method. Specifically, we aim to answer the following questions:
(1) Does it reduce training oscillations, having smaller phase-lag of the response, thereby minimizing constraint violations compared to baseline methods?
(2) How robust is the method when facing different parameters that we set?
(3) Can the ADRC Lagrangian method be applied universally to any Lagrangian-based safe RL algorithm?
(4) How does the ADRC-based Lagrangian method perform upon convergence?
(5) How does our method compare against existing state-of-the-art Safe RL approaches?

We will address Questions 1 and 3 in Section~\ref{experiment1}, Question 2 in Section~\ref{sec:abla}, Question 4 in Appendix~\ref{sec:converge_performance} and Question 5 in Section~\ref{sec:compared_sa} and Appendix~\ref{sec:compared_with_saferl}.
\vspace{-0.5em}
\subsection{Experimental Setups}

\label{sec:baseline_environment}
 \paragraph{Environments}
We use the OmniSafe \citep{omnisafe} to do the experiments, which provides a comprehensive and reliable benchmark for safe RL algorithms. We conduct our experiments using four safe RL algorithms with various combinations of agents and tasks. For more detail about the environment, please refer to Appendix~\ref{sec:tasks_env}.
\paragraph{Algorithms and Baseline}
We utilize two categories of algorithms that have been implemented in Omnisafe: on-policy methods (PPO, TRPO) and off-policy methods (DDPG, TD3). We use the classical Lagrangian method and the PID Lagrangian method as baselines.

\subsection{Performance Evaluation}
\label{experiment1}
Figure~\ref{fig:CPPO} illustrates the learning curves of PPO algorithms using various Lagrangian methods across different tasks. Our ADRC methods demonstrate superior constraint satisfaction and a shorter response lag while maintain competitive reward compared to the baseline. To better compare with the baseline, Table~\ref{tab:cppo_partial} presents the violation rates, violation magnitude, and average cost during the training phase. 

To illustrate the universal applicability of our ADRC Lagrangian methods to any Lagrangian-based safe RL algorithm, we conducted additional experiments on various algorithms in the RacecarGoal and CarPush environments. The learning curves for the RacecarGoal tasks are presented in Figure~\ref{fig:racecar-goal}, and the performance metrics are summarized in Table~\ref{table:carpush_racecargoal}. Further experimental results, including detailed analyses, are provided in Appendix~\ref{sec:all_exp}.

\begin{table*}[t]
\centering
\small
\caption{Sensitivity analysis of tuning parameters $k_{ap}$ and $k_{ad}$ in the RacecarGoal environment.}
\label{tab:tuning_kap_kad}
\renewcommand{\arraystretch}{0.95}
\setlength{\tabcolsep}{6pt}

\begin{tabular}{lccc @{\hspace{18pt}} lccc}
\toprule
\multicolumn{4}{c}{\textbf{Varying $k_{ap}$}} & \multicolumn{4}{c}{\textbf{Varying $k_{ad}$}} \\
\cmidrule(r){1-4}\cmidrule(l){5-8}
\textbf{Setting} & \textbf{Vio. Rate (\%)} & \textbf{Mag.} & \textbf{Avg. Cost}
& \textbf{Setting} & \textbf{Vio. Rate (\%)} & \textbf{Mag.} & \textbf{Avg. Cost} \\
\midrule
$k_{ap}=1$    & \textbf{69.98} & \textbf{18.62} & \textbf{39.87}
& $k_{ad}=1$    & \textbf{28.55} & \textbf{6.17}  & \textbf{20.23} \\
$k_{ap}=0.1$  & \textbf{33.08} & \textbf{5.78}  & \textbf{21.22}
& $k_{ad}=0.1$  & \textbf{38.25} & 9.12           & \textbf{23.92} \\
$k_{ap}=0.01$ & \textbf{20.43} & \textbf{2.50}  & \textbf{15.42}
& $k_{ad}=0.01$ & \textbf{39.08} & \textbf{5.75}  & \textbf{23.33} \\
\midrule
PID            & 79.25          & 23.88          & 46.44
& PID            & 44.60          & 7.04           & 26.15 \\
Lag            & 84.35          & 30.16          & 53.38
& Lag            & 87.33          & 37.36          & 61.53 \\
\bottomrule
\end{tabular}
\end{table*}

\subsection{Parameter Sensitivity Analysis}
\label{sec:abla}
To demonstrate whether our ADRC Lagrangian methods are robust to different value of parameters, we test the parameter $c_r$ in Section~\ref{sec:tun_c_r}, we test the control gain $k_{ap}$ and $k_{ad}$ in Section~\ref{sec:tun_kap_kad}.
\subsubsection{Tuning Parameter $c_r$}
\label{sec:tun_c_r}
We selected five different values for $c_r$ and conducted the experiments using TRPO under the RacecarGoal benchmarks. 
As shown in Table \ref{table:racecargoal_c_r}, the results show that all selected values of the parameter $c_r$ outperform the baseline, demonstrating robustness to parameter variations. More experimental results can be found at Appendix \ref{sec:appendix_tuning_c_r}.

\subsubsection{Tuning Parameters $k_{ap}$ and $k_{ad}$}
\label{sec:tun_kap_kad}
We investigate the impact of the tuning parameters $k_{ap}$ and $k_{ad}$ under the RacecarGoal environment. For $k_{ap}$, experiments were conducted using PPO; for $k_{ad}$, we used TRPO. In both cases, we evaluated three orders of magnitude: 1, 0.1, and 0.01. As shown in Table~\ref{tab:tuning_kap_kad}, the results demonstrate that all tested values of $k_{ap}$ and $k_{ad}$ significantly outperform existing methods, including the PID method, the Lagrangian method, and the Classical Lagrangian method. For additional experimental details, please refer to Appendices~\ref{sec:kap} and~\ref{sec:kad}.

\subsection{Ablation Study}
{To evaluate the effectiveness of our proposed dynamic parameter adjustment and transient process, we conducted ablation studies, with results summarized in Table~\ref{tab:ablation_CarPush_CPPO_main}. In this table, “Delete $r(t)$” refers to the removal of the dynamic adjustment component $r(t)$, while “Delete $w_0$” refers to the exclusion of the transient weight $w_0$ from the algorithm. The results show that removing either component results in a clear performance degradation in terms of violation rate, violation magnitude, and average cost. However, even with these removals, the performance of our approach remains superior to the baseline PID method, demonstrating the robustness of our framework. Additionally, the complete ADRC method achieves the best results across all metrics, further highlighting the significance of combining both $r(t)$ and $w_0$ in achieving optimal performance. For further details and results, please refer to Appendix~\ref{sec:ablation_in_appendix}.}
\begin{table}[H]
\centering
\small
\caption{Ablation study of CPPO algorithm under RacecarGoal.}
\begin{tabular}{lccc}
\toprule
\textbf{Method} & \textbf{Vio. Rate(\%)} & \textbf{Magnitude} & \textbf{Avg. Cost} \\
\midrule
Delete $r(t)$ & 65.40 & 13.99 & 36.38 \\
Delete $w_0$  & 54.08 & 15.23 & 34.66 \\
ADRC (Ours)   & \textbf{33.08} & \textbf{5.78} & \textbf{21.22  } \\
\midrule
PID           & 79.25 & 23.88 & 46.44 \\
Lag           & 84.35 & 30.16 & 53.38 \\
\bottomrule
\end{tabular}
\label{tab:ablation_CarPush_CPPO_main}
\end{table}

\subsection{Comparison with State-of-the-Art Safe RL Algorithms}
\label{sec:compared_sa}
To demonstrate the broader applicability and effectiveness of our ADRC-Lagrangian framework beyond baseline, we conduct comprehensive comparisons with state-of-the-art safe RL algorithms. Our evaluation includes both Lagrangian-based methods (RCPO and PDO)~\cite{rcpo,PDO} and non-Lagrangian approaches such as CUP~\cite{cup} and IPO~\cite{ipo}. This comparison validates our method's superiority across different safe RL paradigms and confirms that the benefits stem from ADRC's adaptive control principles rather than merely being artifacts of the Lagrangian framework.

As detailed in Appendix~\ref{sec:compared_with_saferl}, ADRC variants consistently improve training stability by reducing violation rates, violation magnitudes, and average costs, while maintaining or even enhancing final task performance. 

\vspace{-0.2cm}
\subsection{Converge Performance Analysis}
To assess the final performance of the trained policies rather than intermediate training behavior, we conducted experiments on the Swimmer and Hopper environments from the Velocity tasks suite. The results compare ADRC-based and PID-based methods under CPPO and TRPO frameworks.

\begin{table}[h]
\centering
\caption{Final performance on the Swimmer environment.}
\begin{tabular}{lccc}
\toprule
Algorithm & Avg Reward & Avg Cost & Vio. Rate (\%) \\
\midrule
CPPOPID & \textbf{30.10} & 22.44 & 28.02 \\
CPPOADRC & 29.39 & \textbf{16.77} & \textbf{14.16} \\
TRPOPID & 28.72 & 21.34 & 37.85 \\
TRPOADRC & \textbf{36.32} & \textbf{19.03} & \textbf{12.16} \\
\bottomrule
\end{tabular}
\label{tab:velocity_swimmer}
\end{table}

As shown in Table~\ref{tab:velocity_swimmer}, ADRC consistently improves constraint satisfaction over PID under both CPPO and TRPO, yielding lower average cost and violation rate; under TRPO, it also achieves a substantially higher average reward.
\vspace{-0.2cm}
\section{Conclusion}
\label{sec:conclusion}

In this paper, we introduce an effective method to optimize the Lagrangian multiplier update process in safe RL, reducing oscillation during training.
First, we define the safe RL learning process as a closed-loop system.
Next, we introduce ADRC, a robust and innovative controller that estimates and compensates for overall disturbances.
We consider the current cost as the control objective and design a second-order closed-loop system to regulate this cost, ensuring compliance with the safety constraint.
Additionally, we employed a reduced-order ESO \citep{PID_ADRC} to estimate the unknown nonlinear function affecting agent costs, revealing that prior approaches, including PID Lagrangian and classic Lagrangian methods, in form, are special cases of our approach.
Theoretical proofs and experimental results demonstrate the effectiveness and superiority of our method over existing approaches. While our method is validated extensively in simulated environments, applying it to real-world robotics or safety-critical systems remains an important direction for future work.
\newpage
\nocite{langley00}
\nocite{trpo}
\nocite{ppo}
\nocite{DDPG}
\nocite{TD3}
\bibliography{example_paper}
\bibliographystyle{icml2026}

\newpage
\appendix
\onecolumn

\section{Process of Solving ODE}
\label{sec:solve_ODE}
We consider the ODE:
\begin{equation}
\ddot{r} = -2c_r \dot{r} - c_r^2 (r - d), \quad r(0) = x_1(0), \quad \dot{r}(0) = x_2(0),
\label{eq:original_ode}
\end{equation}
which can be rewritten in the standard form:
\begin{equation}
\ddot{r} + 2c_r \dot{r} + c_r^2 r = c_r^2 d.
\label{eq:standard_form}
\end{equation}
First, we solve the associated homogeneous equation:
\begin{equation}
\ddot{r} + 2c_r \dot{r} + c_r^2 r = 0.
\label{eq:homogeneous_ode}
\end{equation}
Assuming a solution of the form \( r_h(t) = e^{\lambda t} \) and substituting into Eqn.~\ref{eq:homogeneous_ode}, we obtain the characteristic equation:
\begin{equation}
\lambda^2 + 2c_r \lambda + c_r^2 = 0.
\label{eq:characteristic_eq}
\end{equation}
Solving for \( \lambda \) yields:
\begin{equation}
\lambda = -c_r.
\label{eq:lambda_solution}
\end{equation}
Thus, the general solution for the homogeneous equation is:
\begin{equation}
r_h(t) = (A + Bt) e^{-c_r t},
\label{eq:homogeneous_solution}
\end{equation}
where \( A \) and \( B \) are constants that determined by the initial value.

For the nonhomogeneous equation Eqn.~\ref{eq:standard_form}, we assume a particular solution \( r_p(t) = C \). Substituting into Eqn.~\ref{eq:standard_form} gives:
\begin{equation}
C = d.
\label{eq:particular_solution}
\end{equation}
The general solution to Eqn.~\ref{eq:standard_form} is:
\begin{equation}
r(t) = (A + Bt) e^{-c_r t} + d.
\label{eq:general_solution}
\end{equation}
To determine \( A \) and \( B \), we use the initial conditions. From \( r(0) = x_1(0) \):
\begin{equation}
A + d = x_1(0) \quad \Rightarrow \quad A = x_1(0) - d.
\label{eq:A_value}
\end{equation}
The derivative \( \dot{r}(t) \) is:
\begin{equation}
\dot{r}(t) = \left(B - c_r (A + Bt)\right) e^{-c_r t}.
\label{eq:derivative_r}
\end{equation}
Substitute \( t = 0 \) into Eqn.~\ref{eq:derivative_r} and use \( \dot{r}(0) = x_2(0) \):
\begin{equation}
B - c_r A = x_2(0). \rightarrow B = x_2(0) + c_r (x_1(0) - d).
\label{eq:initial_condition_2}
\end{equation}
Substitute \( A \) and \( B \) into Eqn.~\ref{eq:general_solution} to obtain the final solution:
\begin{equation}
r(t) = d + (x_1(0) - d) e^{-c_r t} + \left(x_2(0) + c_r (x_1(0) - d)\right) t e^{-c_r t}.
\label{eq:final_solution}
\end{equation}

\section{Simplify the ESO}
\label{sec:simplify_ESO}
We consider the control law:
\begin{equation}
u = k_{ap}(x_1 - r) + k_{ad}(x_2 - \dot{r}) + \hat{f} - \ddot{r}, \quad k_{ap} > 0, \quad k_{ad} > 0,
\label{eq:control_law}
\end{equation}
where \( k_{ap} \) and \( k_{ad} \) are tuning parameters, and the term \( \hat{f} \) compensates for disturbances.

Substitute Eqn.~\ref{eq:control_law} into the Eqn.~\ref{eq:ESO_Form}:
\begin{equation}
\dot{\xi} = -\omega_o \xi - \omega_o^2 x_2 - \omega_o \left(k_{ap}(x_1 - r) + k_{ad}(x_2 - \dot{r}) + \hat{f} - \ddot{r}\right).
\label{eq:ESO_substitute}
\end{equation}

Simplify Eqn.~\ref{eq:ESO_substitute}:
\begin{equation}
\dot{\xi} = -\omega_o \xi - \omega_o^2 x_2 + \omega_o k_{ap}(x_1 - r) + \omega_o k_{ad}(x_2 - \dot{r}) + \omega_o \hat{f} - \omega_o \ddot{r}.
\label{eq:ESO_simplified}
\end{equation}

Given \( \hat{f} = \xi + \omega_o x_2 \), we have \( \xi = \hat{f} - \omega_o x_2 \) and \( \dot{\xi} = \dot{\hat{f}} - \omega_o \dot{x}_2 \). Substitute this into Eqn.~\ref{eq:ESO_simplified}:
\begin{equation}
\dot{\hat{f}} - \omega_o \dot{x}_2 = -\omega_o (\hat{f} - \omega_o x_2) - \omega_o^2 x_2 + \omega_o k_{ap}(x_1 - r) + \omega_o k_{ad}(x_2 - \dot{r}) + \omega_o \hat{f} - \omega_o \ddot{r}.
\label{eq:dot_f_substitute}
\end{equation}

Simplify further:
\begin{equation}
\dot{\hat{f}} = \omega_o k_{ap}(x_1 - r) + \omega_o k_{ad}(x_2 - \dot{r}) - \omega_o \ddot{r} + \omega_o \dot{x}_2.
\label{eq:dot_f_final}
\end{equation}

Integrating Eqn.~\ref{eq:dot_f_final}, we have:
\begin{equation}
\hat{f}=\omega_ok_{ad}(x_1-r)+\omega_o(x_2-\dot{r})+\omega_ok_{ap}\int_0^t(x_1(\tau)-r(\tau))d\tau.
\label{eq:f_integrated}
\end{equation}

Substitute Eqn.~\ref{eq:f_integrated} back into Eqn.~\ref{eq:control_law}:
\begin{equation}
u = (k_{ap} + \omega_o k_{ad})(x_1 - r) + (k_{ad} + \omega_o)(x_2 - \dot{r}) + \omega_o k_{ap} \int_0^t (x_1(\tau) - r(\tau)) d\tau - \ddot{r}.
\label{eq:control_law_final}
\end{equation}

\section{Theoretical Details}
\label{sec:proof}
\subsection{Convergence and Error Bounds}
\label{appendix:proofs}

For completeness, we provide the detailed stability conditions and error analysis that were summarized in Sec.~\ref{sec:lower_bound}. Recall the disturbance class
\begin{equation}
\label{eq:non_linear_function_class_appendix}
\begin{aligned}
\mathcal{F} = \Big\{ f \ \Big| \ 
& f(x_1,x_2,t) = h(x_1,x_2) + w(t), \\
& \Big|\frac{\partial h}{\partial x_1}\Big| \leq L_1, \quad 
  \Big|\frac{\partial h}{\partial x_2}\Big| \leq L_2, \\
& |w(t)| \leq L_3, \quad 
  |\dot{w}(t)| \leq L_3, \quad 
  \lim_{t \to \infty} w(t) = k \Big\},
\end{aligned}
\end{equation}
where $L_1,L_2$ bound state-dependent sensitivity, and $L_3$ bounds the magnitude and rate of purely time-dependent fluctuations.

\paragraph{Stability manifold and lower bound.}
To guarantee convergence, the observer gain $\omega_o$ must lie in a feasible region determined by the characteristic polynomial
\begin{equation}
\label{eq:manifold_appendix}
\Omega = \Big\{\omega \in \mathbb{R} \ \Big|\ n_0 \omega^4 + n_1 \omega^3 + n_2 \omega^2 + n_3 \omega + n_4 = 0 \Big\},
\end{equation}
where the coefficients $n_i$ depend on $(k_{ap},k_{ad})$ and the constants $L_1,L_2,L_3$. Let
\[
\bar{\omega}_o =
\begin{cases}
\max \{ \omega \mid \omega \in \Omega \}, & \text{if } \Omega \neq \emptyset, \\ 
0, & \text{otherwise}.
\end{cases}
\]
The admissible observer gains are then those satisfying
\begin{equation}
\label{eq:lower_bound_appendix}
\omega_o > \omega_o^* = \max \left\{ \bar{\omega}_o, \ 0, \ \tfrac{L_1 - k_{ap}}{k_{ad}}, \ L_2 - k_{ad} \right\}.
\end{equation}

Suppose $f \in \mathcal{F}$ and $\omega_o > \omega_o^*$. Then:
\begin{itemize}
    \item (\emph{Convergence}) For any initial condition and any cost limit $d \in \mathbb{R}$, the system converges:
\[
\lim_{t\to\infty} x_1(t) = d, 
\qquad
\lim_{t\to\infty} x_2(t) = 0.
\]
    \item (\emph{Bounded estimation error}) Let $e(t) = x_1(t) - r(t) $ be the tracking error and $e_f(t) = \hat f - f(x_1,x_2,t)$ the disturbance estimation error. Then there exist constants $\eta_1,\eta_2$ such that
\[
|\ddot{e}(t) + k_{ad}\dot{e}(t) + k_{ap} e(t)| = |e_f(t)| 
\ \leq \ \eta_1 e^{-\omega_o t} + \tfrac{\eta_2}{\omega_o}, \quad t \geq 0.
\]
\end{itemize}

The first result shows that as long as the observer gain exceeds the lower bound $\omega_o^*$, the cumulative cost $x_1$ converges to the constraint threshold $d$ without oscillation. The second result shows that the estimation error is always bounded, decays over time, and can be reduced by choosing larger $\omega_o$. Together, these properties justify our claim in the main text that ADRC Lagrangian guarantees convergence and robustness without fragile manual tuning. The detailed proofs follow directly from~\citet{PID_ADRC} and related ADRC analyses.

\subsection{Proof of Theorem \ref{theorem:3}}
\begin{proof}

From Theorems demonstrated by~\citet{PID_ADRC}, we know that both \( f \) and \( \dot{f} \) are bounded. The error dynamics are given by:
\begin{equation}
\begin{cases}
\dot{e} = e_d, \\
\dot{e}_d = -k_{ap} e - k_{ad} e_d + e_f.
\end{cases}
\label{eq:error_dynamics}
\end{equation}

Taking the second derivative of \( e \), we have:
\begin{equation}
\ddot{e} = -k_{ap} e - k_{ad} e_d + e_f,
\label{eq:second_derivative_e}
\end{equation}
or equivalently:
\begin{equation}
e_f = \ddot{e} + k_{ad} \dot{e} + k_{ap} e.
\label{eq:ef_expression}
\end{equation}

Applying the Laplace transform to Eqn.~\ref{eq:ef_expression}, we obtain:
\begin{equation}
E(s) = \frac{1}{s^2 + k_{ad}s + k_{ap}} E_f(s).
\label{eq:E_Laplace}
\end{equation}

And we know that, the dynamics of \( e_f \) are given by:
\begin{equation}
\dot{e}_f = -\omega_o e_f - \dot{f}.
\label{eq:ef_dot_dynamics}
\end{equation}

Taking the Laplace transform of Eqn.~\ref{eq:ef_dot_dynamics}, we have:
\begin{equation}
E_f(s) = G_{e_f}(s) F(s), \quad G_{e_f}(s) = \frac{s}{s + \omega_o}.
\label{eq:G_ef}
\end{equation}

Similarly, applying the Laplace transform to the integral form of \( e_f \), we obtain:
\begin{equation}
E_{f_I}(s) = \frac{s^2 + k_d s + k_p}{s^2 + k_{ad} s + k_{ap}} E_f(s),
\label{eq:E_fI}
\end{equation}
and the transfer function for \( E_{f_I}(s) \) can be expressed as:
\begin{equation}
E_{f_I}(s) = G_{e_{f_I}}(s) F(s), \quad G_{e_{f_I}}(s) = \frac{s^3 + k_d s^2 + k_p s}{(s + \omega_o)(s^2 + k_{ad} s + k_{ap})}.
\label{eq:G_efI}
\end{equation}

The ratio of the squared magnitudes of \( G_{e_f}(i \omega) \) and \( G_{e_{f_I}}(i \omega) \) is given by:
\begin{equation}
\frac{|G_{e_f}(i \omega)|^2}{|G_{e_{f_I}}(i \omega)|^2} = \frac{(k_{ap} - \omega^2)^2 + k_{ad}^2 \omega^2}{(k_p - \omega^2)^2 + k_d^2 \omega^2} < 1.
\label{eq:magnitude_ratio}
\end{equation}

As \( t \to \infty \), we have:
\begin{equation}
\lim_{t \to \infty} \frac{e_f(t)}{e_{f_I}(t)} = \lim_{s \to 0} \frac{s E_f(s)}{s E_{f_I}(s)} = \frac{k_{ap}}{k_{ap} + \omega_o k_{ad}}.
\label{eq:ratio_limit}
\end{equation}

This completes the first part of this theorem.

Now, consider the phase angle of a transfer function \( G(i \omega) \), defined as:
\begin{equation}
\arg(G(i \omega)) = \tan^{-1} \left( \frac{\text{Im}(G(i \omega))}{\text{Re}(G(i \omega))} \right).
\label{eq:phase_angle}
\end{equation}

For \( G_{e_f}(i \omega) \) and \( G_{e_{f_I}}(i \omega) \), we have:
\begin{equation}
\arg(G_{e_f}(i \omega)) = \tan^{-1} \left( \frac{\omega}{\omega_o} \right),
\label{eq:arg_G_ef}
\end{equation}
and
\begin{equation}
\arg(G_{e_{f_I}}(i \omega)) = \tan^{-1} \left( \frac{k_{ad} \omega}{k_{ap} - \omega^2} \right).
\label{eq:arg_G_efI}
\end{equation}

For any \( \omega \), if we choose \( \omega_o > \max\left\{ \frac{k_{ap} - \omega^2}{k_{ad}}, \omega_o^* \right\} \), it follows that:
\begin{equation}
\frac{\omega}{\omega_o} < \frac{k_{ad} \omega}{k_{ap} - \omega^2}.
\label{eq:angle_comparison}
\end{equation}

Thus, we conclude:
\begin{equation}
\tan^{-1} \left( \frac{\omega}{\omega_o} \right) < \tan^{-1} \left( \frac{k_{ad} \omega}{k_{ap} - \omega^2} \right),
\end{equation}
or equivalently:
\begin{equation}
\arg(G_{e_f}(i \omega)) < \arg(G_{e_{f_I}}(i \omega)).
\label{eq:phase_relation}
\end{equation}

This completes the second part of this theorem.
    
\end{proof}

\subsection{Surrogate Analysis of ADRC Cost Regulation}
\label{sec:theory-corrected}

\paragraph{Scope}
This section provides a control-theoretic surrogate analysis for the cost-regulation channel induced by the ADRC multiplier update.
We analyze the \emph{population} discounted cost return
\[
x_1(t) \;:=\; J_c(\pi(t)),
\]
where $\pi(t)$ is a smooth interpolation of the discrete policy iterates $\{\pi_{\theta_k}\}_{k\ge0}$ produced by a trust-region backbone (TRPO/PPO).
The guarantees below control the evolution of $x_1(t)$, rather than per-trajectory costs.
We do not claim almost-sure episode-level safety, nor global convergence in nonconvex policy space; instead, we derive robustness and bounded-violation guarantees under a standard ADRC disturbance-regularity envelope, and we explicitly connect the disturbance envelope to Safe RL hyperparameters.

\paragraph{Discrete training index and sampling interpretation.}
Let $k\in\{0,1,2,\dots\}$ index policy updates with update interval $\Delta t>0$.
We denote $x_{1,k}:=x_1(k\Delta t)=J_c(\pi_{\theta_k})$ and the multiplier by $\lambda_k$.
In implementation, $\lambda_k$ is held constant within iteration $k$ (zero-order hold), which corresponds to a piecewise-constant signal
$\lambda(t)=\lambda_k$ for $t\in[k\Delta t,(k+1)\Delta t)$ in the surrogate analysis.

\subsubsection{Surrogate cost channel and error dynamics}

\paragraph{Surrogate channel with correct feedback direction.}
In Lagrangian-based Safe RL, a larger multiplier $\lambda$ increases the penalty on cost in the primal objective, which acts as negative feedback on the cost return.
Accordingly, we analyze the sign-normalized relative-degree-2 surrogate channel
\begin{equation}
\label{eq:surrogate-channel}
\dot x_1(t)=x_2(t),\qquad 
\dot x_2(t)= f(t)\;-\;\lambda(t),
\end{equation}
where $f(t)$ lumps all unmodeled effects (policy nonstationarity, approximation error of the abstraction, and sampling effects).
Eq.~\eqref{eq:surrogate-channel} is a local abstraction of the primal-dual interaction; it is \emph{not} a physical model of the environment.

\paragraph{Reference tracking and ADRC law.}
Let $r(t)$ be the critically damped reference trajectory from Sec.~\ref{sec:transient_process}, converging to $d$ (or $d-\varepsilon$ for a margin variant).
Define the tracking error and its derivative (consistent with the main text):
\[
e(t):=x_1(t)-r(t),\qquad e_d(t):=\dot e(t)=x_2(t)-\dot r(t).
\]
Under the ADRC design in Sec.~\ref{sec:ESO}, the multiplier is updated by
\begin{equation}
\label{eq:adrc-law-appendix}
\lambda(t)=k_{ap}\big(x_1-r\big)+k_{ad}\big(x_2-\dot r\big)+\hat f(t)-\ddot r(t),
\qquad k_{ap}>0,\;k_{ad}>0,
\end{equation}
where $\hat f(t)$ is the ESO estimate of $f(t)$.

\paragraph{Reduced-order ESO and estimation error dynamics.}
For the negative-feedback channel \eqref{eq:surrogate-channel}, the reduced-order ESO takes the sign-consistent form
\begin{equation}
\label{eq:eso-appendix}
\dot{\xi}(t)=-\omega_o \xi(t)-\omega_o^2 x_2(t)+\omega_o \lambda(t),
\qquad 
\hat f(t)=\xi(t)+\omega_o x_2(t),
\qquad \omega_o>0.
\end{equation}
Define the disturbance estimation error $e_f(t):=\hat f(t)-f(t)$.
Combining \eqref{eq:surrogate-channel} and \eqref{eq:eso-appendix} yields the standard ADRC error equation
\begin{equation}
\label{eq:ef-dynamics}
\dot e_f(t)=-\omega_o e_f(t)-\dot f(t).
\end{equation}

\paragraph{Closed-loop tracking-error dynamics.}
Substituting \eqref{eq:adrc-law-appendix} into \eqref{eq:surrogate-channel} and using $e=x_1-r$, $e_d=x_2-\dot r$, we obtain
\begin{equation}
\label{eq:tracking-dynamics}
\ddot e(t)+k_{ad}\dot e(t)+k_{ap}e(t)=-e_f(t).
\end{equation}
Thus the tracking channel is a stable second-order system driven by the estimation error $e_f(t)$.

\subsubsection{Disturbance regularity and an RL-grounded high-probability envelope}

\begin{assumption}[Disturbance regularity]
\label{ass:df-bound}
The lumped disturbance $f(t)$ in \eqref{eq:surrogate-channel} is absolutely continuous and satisfies
\[
|\dot f(t)| \le L_f
\]
over the time interval of interest.
\end{assumption}

Assumption~\ref{ass:df-bound} is standard in ADRC analyses: it formalizes that the total uncertainty seen by the observer cannot vary arbitrarily fast.
We now justify that $L_f$ is \emph{algorithmically controlled} in trust-region Safe RL and provide a finite-horizon high-probability envelope that links $L_f$ to $(\delta,N,\Delta t)$.

\begin{assumption}[Bounded discounted cost return]
\label{ass:bounded-return}
There exists $B_c<\infty$ such that the discounted trajectory cost return
$
C(\tau)=\sum_{t=0}^\infty \gamma^t c_t
$
satisfies $0\le C(\tau)\le B_c$ almost surely.
\end{assumption}

\begin{assumption}[Trust-region update]
\label{ass:tr}
The backbone update satisfies a trust-region constraint
$
D_{\mathrm{KL}}(\pi_{\theta_k}\,\|\,\pi_{\theta_{k+1}})\le \delta
$
(TRPO) or an implicit KL control induced by PPO clipping.
\end{assumption}

\begin{assumption}[Population cost drift under trust regions]
\label{lem:cost-drift}
Under Assumptions~\ref{ass:bounded-return}--\ref{ass:tr},
\[
|J_c(\pi_{\theta_{k+1}})-J_c(\pi_{\theta_k})|
\;\le\;
\underbrace{\frac{2B_c}{1-\gamma}\sqrt{2\delta}}_{=:D_{\mathrm{TR}}}.
\]
\end{assumption}
\begin{proof}[Sketch]
The discounted performance-difference bound controls return differences by total-variation distance; Pinsker's inequality gives
$\mathrm{TV}\le \sqrt{D_{\mathrm{KL}}/2}\le \sqrt{\delta/2}$, yielding the stated bound after collecting constants.
\end{proof}

\begin{lemma}[Finite-horizon high-probability envelope for disturbance variation]
\label{lem:Lf-envelope}
Assume (i) $\lambda_k\in[0,\lambda_{\max}]$ (projection and/or a max-penalty cap as in implementation),
(ii) the cost estimate $\widehat{J}_c(\pi_{\theta_k})$ is computed from $N$ i.i.d.\ trajectories, and
(iii) Assumptions~\ref{ass:bounded-return}--\ref{ass:tr} hold.
Fix a horizon $K$ and confidence $\eta\in(0,1)$, and define the (empirical) second-difference proxy
\[
\widehat f_k
\;:=\;
\frac{\widehat J_c(\pi_{\theta_{k+1}})-2\widehat J_c(\pi_{\theta_k})+\widehat J_c(\pi_{\theta_{k-1}})}{\Delta t^2}
\;+\;\lambda_k.
\]
Then, with probability at least $1-\eta$,
\[
\max_{1\le k\le K-1}\frac{|\widehat f_{k+1}-\widehat f_k|}{\Delta t}
\;\le\;
\frac{4}{\Delta t^3}\Big(D_{\mathrm{TR}}+2\varepsilon_N\Big)
\;+\;
\frac{2\lambda_{\max}}{\Delta t},
\qquad
\varepsilon_N:=B_c\sqrt{\frac{\log(2K/\eta)}{2N}}.
\]
Consequently, on this event one may take
$
L_f = \mathcal{O}\!\left(\frac{\sqrt{\delta}}{\Delta t^3}+\frac{1}{\Delta t^3}\sqrt{\frac{\log(K/\eta)}{N}}+\frac{\lambda_{\max}}{\Delta t}\right)
$
as a conservative envelope in Assumption~\ref{ass:df-bound}.
\end{lemma}

\begin{proof}[Sketch]
Hoeffding's inequality plus a union bound yields
$|\widehat J_c(\pi_{\theta_k})-J_c(\pi_{\theta_k})|\le \varepsilon_N$ for all $k\le K$ with prob.\ $\ge 1-\eta$.
Lemma~\ref{lem:cost-drift} bounds the population drift by $D_{\mathrm{TR}}$.
Thus the per-iteration increment of the observed cost satisfies
$|\widehat J_c(\pi_{\theta_{k+1}})-\widehat J_c(\pi_{\theta_k})|\le D_{\mathrm{TR}}+2\varepsilon_N$,
which bounds the third finite difference by $4(D_{\mathrm{TR}}+2\varepsilon_N)$.
The cap $\lambda_k\in[0,\lambda_{\max}]$ implies $|\lambda_{k+1}-\lambda_k|\le 2\lambda_{\max}$.
Combining these bounds yields the inequality.
\end{proof}

\paragraph{Interpretation.}
Lemma~\ref{lem:Lf-envelope} provides the missing bridge: the ``disturbance speed'' is controlled by the trust-region radius $\delta$
(or PPO clipping), batch size $N$, and update interval $\Delta t$, plus an explicit dependence on $\lambda_{\max}$ due to projection/capping.

\subsubsection{Robustness: ESO error bound and ISS-type tracking tube}

\begin{lemma}[ESO estimation-error bound]
\label{lem:eso-bound}
Under Assumption~\ref{ass:df-bound}, the solution of \eqref{eq:ef-dynamics} satisfies for all $t\ge 0$:
\begin{equation}
\label{eq:ef-bound}
|e_f(t)| \;\le\; e^{-\omega_o t}|e_f(0)| \;+\; \frac{L_f}{\omega_o}.
\end{equation}
\end{lemma}

\begin{theorem}[ISS-type tracking bound]
\label{thm:iss}
Assume $k_{ap}>0$, $k_{ad}>0$ so that $s^2+k_{ad}s+k_{ap}$ is Hurwitz.
Let $h(t)$ be the impulse response of $H(s)=\frac{1}{s^2+k_{ad}s+k_{ap}}$ and define
$\|h\|_{L_1}:=\int_0^\infty |h(\tau)|d\tau < \infty$.
Then there exist constants $C_0>0$ and $\rho>0$ (depending on $(k_{ap},k_{ad})$ and initial conditions) such that
\begin{equation}
\label{eq:tracking-tube}
|e(t)| \;\le\; C_0 e^{-\rho t} \;+\; \|h\|_{L_1}\left(e^{-\omega_o t}|e_f(0)|+\frac{L_f}{\omega_o}\right).
\end{equation}
In particular,
\begin{equation}
\label{eq:limsup-tube}
\limsup_{t\to\infty}|e(t)| \;\le\; \|h\|_{L_1}\,\frac{L_f}{\omega_o}
\;=\; \mathcal{O}\!\left(\frac{L_f}{\omega_o}\right).
\end{equation}
\end{theorem}

\paragraph{A computable constant.}
If $k_{ad}^2\ge 4k_{ap}$, then $h(t)\ge 0$ and $\|h\|_{L_1}=H(0)=1/k_{ap}$, yielding the explicit tube radius
$
\limsup_{t\to\infty}|e(t)| \le \frac{1}{k_{ap}}\frac{L_f}{\omega_o}.
$

\subsubsection{Safety: bounded time-average violation and eventual feasibility}

\begin{theorem}[Bounded time-average population violation]
\label{thm:avg-viol}
Let $x_1(t)=J_c(\pi(t))$ and suppose $r(t)\to d$ as in Sec.~\ref{sec:transient_process}.
Under the conditions of Theorem~\ref{thm:iss},
\begin{equation}
\label{eq:avg-viol}
\limsup_{T\to\infty}\frac{1}{T}\int_0^T \big(x_1(t)-d\big)_+\,dt
\;\le\;
\|h\|_{L_1}\,\frac{L_f}{\omega_o}.
\end{equation}
\end{theorem}

\begin{proof}[Sketch]
Since $x_1=r+e$,
\[
(x_1-d)_+ = (r+e-d)_+ \le (r-d)_+ + |e|.
\]
Because $r(t)\to d$ exponentially, $\lim_{T\to\infty}\frac{1}{T}\int_0^T (r-d)_+dt=0$.
The remaining term is controlled by the tracking tube in Theorem~\ref{thm:iss}.
\end{proof}

\begin{corollary}[Eventual feasibility with a fixed safety margin]
\label{cor:margin}
If the reference is chosen to converge to $d-\varepsilon$ with $\varepsilon>0$ (replace $d$ by $d-\varepsilon$ in the transient process),
and if
$
\varepsilon > \|h\|_{L_1}\,{L_f}/{\omega_o},
$
then there exists $T_\varepsilon<\infty$ such that $x_1(t)\le d$ for all $t\ge T_\varepsilon$ in the surrogate channel.
\end{corollary}

\subsubsection{Reward discussion: what ADRC preserves}

ADRC modifies only the \emph{dual} schedule $\{\lambda_k\}$ and does not change the primal trust-region mechanism.
Therefore, for each fixed $\lambda_k$, standard TRPO/PPO analyses apply to the \emph{penalized} objective
$
J(\pi)-\lambda_k J_c(\pi)
$
(or the rescaled variant used in our implementation).
In other words, ADRC does not invalidate the backbone optimizer's trust-region reasoning; it provides a smoother, less lagged multiplier signal,
which empirically stabilizes the reward--cost tradeoff.

\subsubsection{Projection and saturation}

In implementation, $\lambda_k$ is projected to $\lambda_k\ge 0$ (and optionally capped by $\lambda_{\max}$).
Projection is non-expansive: for any $a,b\in\mathbb{R}$,
$
|\Pi_+(a)-\Pi_+(b)|\le |a-b|
$
with $\Pi_+(x)=\max\{0,x\}$.
Thus projection does not amplify multiplier perturbations; it only clamps the signal and ensures boundedness.
A full hybrid-systems analysis with saturation is beyond the scope of this appendix.

\section{Implementation Detail}
\label{sec:details}
This section outlines the details of the proposed method through the pseudo-code presented in Algorithm \ref{alg:ADRC_Multipliers}. The algorithm describes the procedure for adjusting the Lagrange multipliers using ADRC during training, ensuring robust performance and adaptability to varying conditions.
\begin{algorithm}
\small 
\caption{ADRC-Controlled Lagrange Multiplier}\label{alg:ADRC_Multipliers}
\begin{algorithmic}[1]
\Require Choosed parameters $k_{ap}, k_{ad} \geq 0$
\State Integral: $I \gets 0$
\State Previous Cost: $J_{C, \text{prev}} \gets 0$
\For{each iteration \( t \)}
    \State Receive current cost $J_C$, reference cost $r$, its time derivative $\dot{r}$, $\ddot{r}$ and the optimal gain $\omega_o$.
    \State $\Delta \gets J_C - r$
    \State $\partial \gets (J_C - J_{C, \text{prev}}-\dot{r})_+$
    \State $I \gets (I + \Delta)_+$
    \State $K_P \gets k_{ap} + \omega_o k_{ad}$
    \State $K_I \gets \omega_o k_{ap}$
    \State $K_D \gets \omega_o +k_{ap}$
    \State $\lambda \gets (K_P \Delta + K_I I + K_D \partial-\ddot{r})_+$
    \State $J_{C, \text{prev}} \gets J_C$
\EndFor
\end{algorithmic}
\end{algorithm}

\subsection{Hyper-parameters}
\label{sec:parameters}
For the on-policy algorithms TRPO and PPO, we adopt the default parameters provided by Omnisafe \citep{omnisafe}, as detailed in Table~\ref{tab:onpolicy_para}. These parameters are consistently applied across all tasks.
\begin{table}[htbp]
\centering
\caption{Parameter Comparison: ADRC, PID, and Lag Methods}
\label{tab:onpolicy_para}
\renewcommand{\arraystretch}{1.1}
\setlength{\tabcolsep}{12pt}
\begin{tabular}{lccc}
\toprule
\textbf{Parameter}                 & \textbf{ADRC} & \textbf{PID} & \textbf{Lag} \\ 
\midrule
$k_p(k_{ap})$                                 & 0.1           & 0.1          & -            \\
$k_i$                & -             & 0.01         & 0.035        \\
$k_d(k_{ad})$                                & 0.01          & 0.01         & -            \\
Delay                              & 10            & 10           & -            \\
EMA $\alpha$ (Proportional Term)   & 0.95          & 0.95         & -            \\
EMA $\alpha$ (Derivative Term)     & 0.95          & 0.95         & -            \\
Sum Normalization                  & True          & True         & -            \\
Derivative Normalization           & False         & False        & -            \\
Cost Limit                         & 25.0          & 25.0         & 25.0         \\
Max Penalty Coefficient            & 100.0         & 100.0        & -            \\
Initial Lagrangian Multiplier      & 0.001         & 0.001        & 0.001        \\
Hidden Layer Sizes (Actor)         & {[}64, 64{]}  & {[}64, 64{]} & {[}64, 64{]} \\
Activation Function (Actor)        & tanh          & tanh         & tanh         \\
Hidden Layer Sizes (Critic)        & {[}64, 64{]}  & {[}64, 64{]} & {[}64, 64{]} \\
Activation Function (Critic)       & tanh          & tanh         & tanh         \\
Critic Learning Rate               & 0.0003        & 0.0003       & 0.0003       \\
Linear Learning Rate Decay         & True          & True         & True         \\
Clip Ratio                         & 0.2           & 0.2          & 0.2          \\
Target KL                          & 0.02          & 0.02         & 0.02         \\
Use Max Gradient Norm              & True          & True         & True         \\
Max Gradient Norm                  & 40.0          & 40.0         & 40.0         \\
\bottomrule
\end{tabular}
\label{table:parameter_comparison}
\end{table}

\section{Experimental Details}
\label{sec:experimental_details}

\subsection{Baseline}
To comprehensively evaluate the effectiveness of our proposed ADRC method, we compare it against four well-established reinforcement learning algorithms. These include two off-policy algorithms, TD3 and DDPG, as well as two on-policy algorithms, PPO and TRPO. These algorithms were chosen due to their widespread adoption and proven performance across various RL tasks, providing a robust foundation for benchmarking.

\subsection{Tasks Specification}
\label{sec:tasks_env}
To demonstrate the effectiveness and generalizability of our proposed methods, we conduct comprehensive experiments across diverse environments. We select three distinct agents, namely Car, Racecar, and Ant, each governed by different physical dynamics. 
\begin{figure}[htbp]
    \centering
    \includegraphics[width=\linewidth]{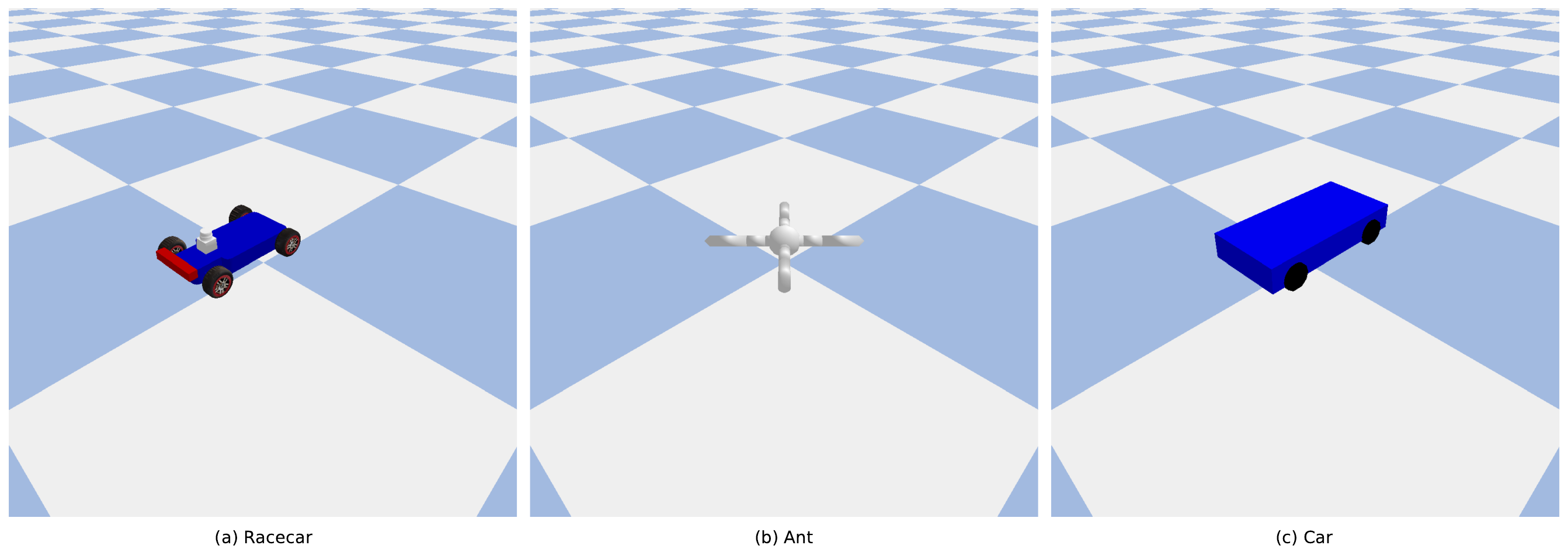}
    \caption{Illustration of the three distinct agents used in our experiments. Car: A simple wheeled agent with low degrees of freedom. Racecar: A dynamic and agile wheeled agent with higher motion complexity. Ant: A multi-legged bionic agent with high degrees of freedom and non-linear dynamics. These agents represent diverse physical characteristics, allowing us to comprehensively evaluate the performance of our method under various physical dynamics.}
    \label{fig:agents}
\end{figure}

As illustrated in Figure~\ref{fig:agents}, the three agents represent diverse physical characteristics, enabling us to evaluate the performance of our method comprehensively across varying physical dynamics.

We consider four tasks in our experiments, as shown in Figure~\ref{fig:tasks}:

\begin{itemize}
    \item \textbf{Goal Task}
    The robot must navigate to a specified goal region while avoiding hazards.
    \item \textbf{Button Task}
    The robot must press the correct button while avoiding hazards and gremlins, and must not press any wrong buttons.
    \item \textbf{Push Task}
    The robot must push a box to the goal region while avoiding hazards. A pillar is present but does not penalize collisions.
    \item \textbf{Circle Task}
    The robot moves around a circular track, without additional objects or hazards. This is mainly for testing circular navigation behavior.
\end{itemize}

\begin{figure}[ht]
    \centering
    \includegraphics[width=\linewidth]{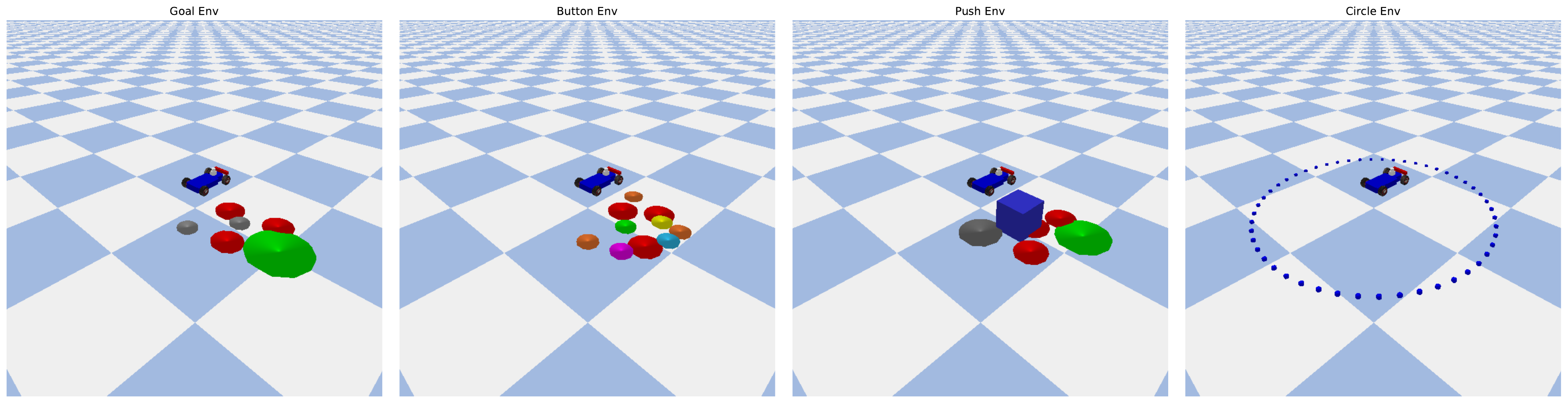}
    \caption{Four different tasks used in our experiments. 
    (a) \textit{Goal Task}: The agent must reach the goal area (blue sphere) without entering dangerous zones (red circles). 
    (b) \textit{Button Task}: The agent must press the correct button (green) and avoid pressing wrong ones (yellow, purple, etc.) or colliding with gremlins. 
    (c) \textit{Push Task}: The agent must push the box to the goal location (green circle) while avoiding hazards (red). 
    (d) \textit{Circle Env}: The agent moves around a simple circular track.}
    \label{fig:tasks}
\end{figure}

\subsection{Evaluation Metrics}
\label{sec:metrics}

To comprehensively assess the performance of the proposed reinforcement learning algorithms, we employ several evaluation metrics. These metrics evaluate both the agent's ability to minimize costs and its adherence to safety constraints.

\paragraph{Reward and Cost}  
The primary performance metrics are the \textbf{reward} and \textbf{cost}, which respectively measure the benefits and penalties accumulated by the agent over the course of an episode. For an episode consisting of $T$ time steps:

\begin{itemize}
    \item The \textbf{return reward}, $R$, is defined as:
    \[
    R = \sum_{t=1}^{T} r_t,
    \]
    where $r_t$ is the reward received at time step $t$. This metric reflects the agent's ability to achieve its objective efficiently.

    \item The \textbf{return cost}, $C$, is calculated as:
    \[
    C = \sum_{t=1}^{T} c_t,
    \]
    where $c_t$ is the cost incurred at time step $t$. This metric assesses the penalties associated with the agent's actions, capturing its safety and resource efficiency.
\end{itemize}

\paragraph{Violation Rate (Vio. Rate)}  
The Violation Rate quantifies the proportion of episodes during training in which the agent breaches predefined safety constraints. It is expressed as:
\[
\text{Vio Rate} = \frac{N_{\text{violations}}}{N_{\text{total\_episodes}}},
\]
where $N_{\text{violations}}$ is the number of episodes in which the agent's cumulative cost $C$ exceeds the allowable threshold, and $N_{\text{total\_episodes}}$ is the total number of training episodes. A lower violation rate indicates better safety performance.

\paragraph{Constraint Violation Magnitude(Magnitude)}  
The Violation Magnitude measures the severity of constraint violations in episodes where breaches occur. It is calculated as the average amount by which the return cost exceeds the allowable threshold across all violating episodes:
\[
\text{Violation Magnitude} = \frac{1}{N_{\text{violations}}} \sum_{i=1}^{N_{\text{violations}}} \max(0, C_i - d),
\]
where $C_i$ is the return cost of the $i$-th violating episode and $d$ is the cost threshold that we set. Smaller magnitudes indicate less severe constraint violations.

\paragraph{Average Cost(Avg. Cost)}  
To evaluate the overall performance during training, we calculate the Average Cost across all episodes:
\[
\text{Average Cost} = \frac{1}{N_{\text{total\_episodes}}} \sum_{i=1}^{N_{\text{total\_episodes}}} C_i,
\]
where $C_i$ is the return cost of the $i$-th episode.

By analyzing these metrics, we can comprehensively assess the effectiveness of each algorithm in achieving a balance between reward maximization and safety constraint adherence.

\section{More Experimental Results}

\subsection{Tables and Figures Referenced in the Main Text}

\begin{table}[H]
\centering
\tiny
\renewcommand{\arraystretch}{0.9} 
\caption{Constraint violation rate (Vio.), violation magnitude (Mag.), and average cost (Cost) during PPO training with various Lagrangian methods.}
\resizebox{\columnwidth}{!}{%
\begin{tabular}{llccc}
\toprule
\textbf{Task} & \textbf{Method} & \textbf{Vio. (\%)} & \textbf{Mag.} & \textbf{Cost} \\
\midrule
\multirow{3}{*}{CarButton} 
& Lag  & 89.77 $\pm$ 19.38 & 59.77 $\pm$ 39.05 & 84.05 $\pm$ 40.18 \\
& PID  & 85.09 $\pm$ 16.67 & 45.27 $\pm$ 46.80 & 68.95 $\pm$ 47.94 \\
& ADRC & \textbf{50.16 $\pm$ 17.08} & \textbf{14.80 $\pm$ 3.96} & \textbf{34.20 $\pm$ 5.75} \\
\midrule
\multirow{3}{*}{CarCircle} 
& Lag  & 46.74 $\pm$ 20.16 & 15.85 $\pm$ 17.40 & 33.54 $\pm$ 20.09 \\
& PID  & 52.78 $\pm$ 17.62 & 12.29 $\pm$ 8.14 & 31.24 $\pm$ 10.31 \\
& ADRC & \textbf{21.35 $\pm$ 13.09} & \textbf{7.74 $\pm$ 6.46} & \textbf{18.85 $\pm$ 9.30} \\
\midrule
\multirow{3}{*}{RacecarGoal} 
& Lag  & 80.87 $\pm$ 19.17 & 31.18 $\pm$ 18.99 & 54.24 $\pm$ 21.01 \\
& PID  & 72.30 $\pm$ 24.96 & 27.11 $\pm$ 15.58 & 49.02 $\pm$ 18.75 \\
& ADRC & \textbf{47.08 $\pm$ 21.58} & \textbf{12.31 $\pm$ 9.34} & \textbf{30.12 $\pm$ 12.74} \\
\midrule
\multirow{3}{*}{RacecarPush} 
& Lag  & 57.91 $\pm$ 22.12 & 15.45 $\pm$ 12.92 & 35.54 $\pm$ 15.56 \\
& PID  & 70.84 $\pm$ 23.97 & 28.16 $\pm$ 16.84 & 49.67 $\pm$ 19.96 \\
& ADRC & \textbf{47.28 $\pm$ 17.05} & \textbf{12.50 $\pm$ 7.05} & \textbf{29.35 $\pm$ 11.03} \\
\bottomrule
\end{tabular}%
}
\label{tab:cppo_partial}
\end{table}

As shown in Table~\ref{tab:cppo_partial}, the ADRC methods significantly reduce violation rate, a smaller violation magnitude, indicating reduced oscillations and a shorter phase-lag in response. The calculation of metrics are detailed in Appendix~\ref{sec:metrics}.

\begin{figure}[H]
    \centering
    \includegraphics[width=1\linewidth]{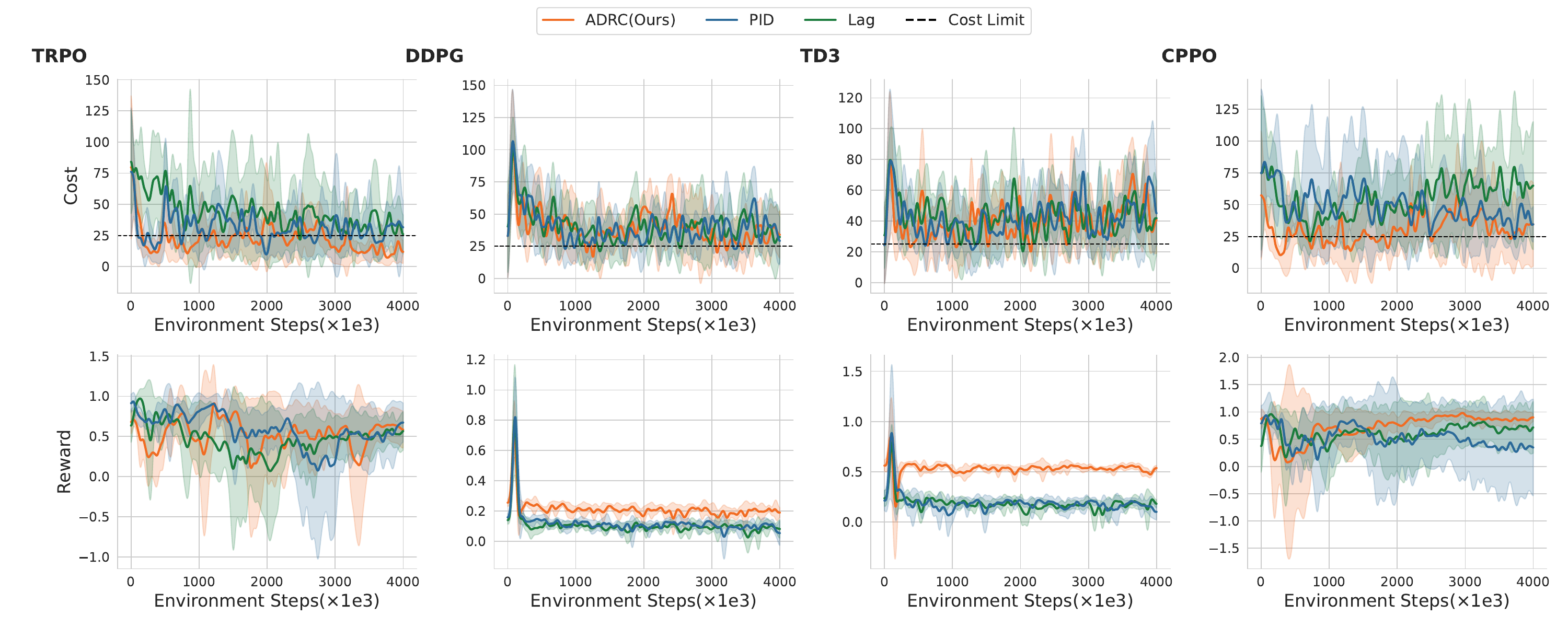}
    \caption{Training curves of Racecargoal task.}
    \label{fig:racecar-goal}
\end{figure}

Figure~\ref{fig:racecar-goal} shows training curves for different constraint-handling methods (ADRC, PID, and Classical Lagrangian) across four RL algorithms (TRPO, DDPG, TD3, and PPO) in the RacecarGoal environment. Results show that ADRC consistently maintains cost below the limit while achieving competitive or superior rewards compared to other methods across different RL backbones.

\label{sec:More_results}
\subsection{Main Results}
\label{sec:all_exp}
To ensure clarity and readability, we present the training curves for each environment separately, avoiding the complexity of overlaying multiple curves on a single plot. This approach allows for a more intuitive comparison of performance across different settings. For a comprehensive evaluation of our method’s effectiveness, we conducted experiments across three agents—Ant, Racecar, and Car—and four reinforcement learning tasks: Goal, Circle, Button, and Push. This setup resulted in a total of 12 experimental groups. For each group, we ran experiments with 5 different random seeds to account for variability and ensure statistical robustness. Furthermore, we benchmarked our method against four widely used reinforcement learning algorithms: TRPO, PPO, DDPG, and TD3, covering both on-policy and off-policy approaches. This rigorous experimental design provides a thorough validation of our method’s adaptability and performance across diverse scenarios.

\subsubsection{Ant Environments}

Figures \ref{fig:antbutton_performance} to Figure~\ref{fig:antpush_performance} present the training curves for the Ant environment across four tasks: Button, Circle, Goal, and Push. Each plot illustrates the episodic returns and costs averaged over five random seeds, with solid lines representing the mean and shaded areas denoting the variance.
\begin{figure}[htbp]
    \centering
    \includegraphics[width=\linewidth]{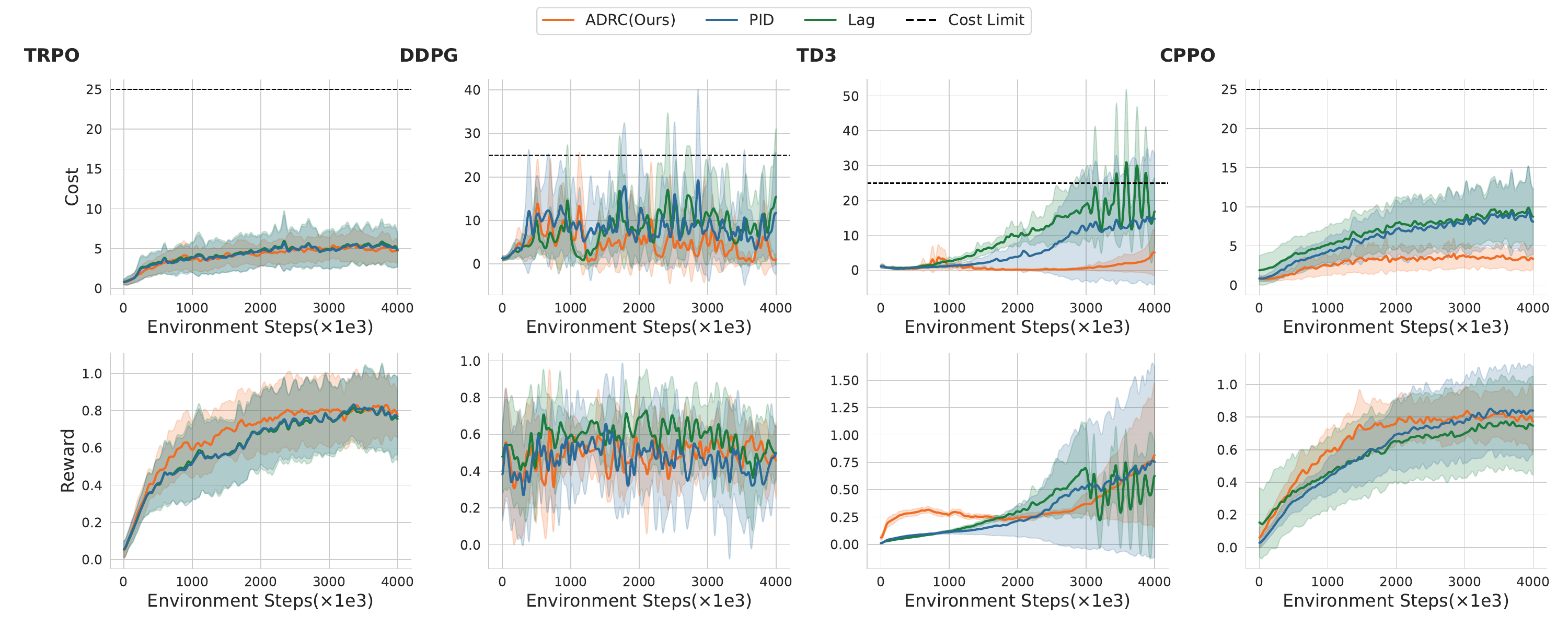}
    \caption{The training curves of AntButton with various Lagrangian methods across different algorithms.}
    \label{fig:antbutton_performance}
\end{figure}

\begin{figure}[htbp]
    \centering
    \includegraphics[width=\linewidth]{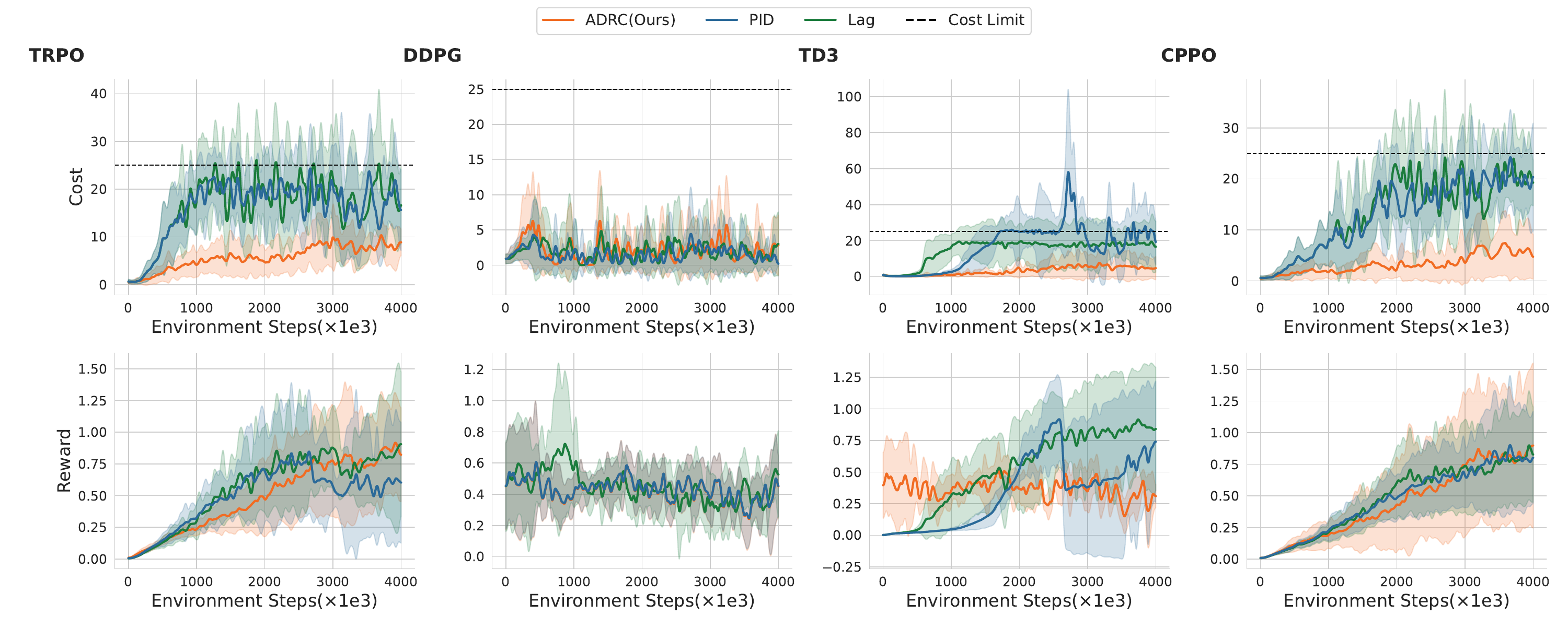}
    \caption{The training curves of AntCircle with various Lagrangian methods across different algorithms.}
    \label{fig:antcircle_performance}
\end{figure}

\begin{table*}[htbp]
\centering
\small
\caption{Comparison of violation rate, magnitude, and average cost on AntButton and AntCircle.}
\resizebox{\textwidth}{!}{
\begin{tabular*}{\textwidth}{@{\extracolsep{\fill}}lcccccc@{}}
\toprule
\textbf{Algorithm} 
& \multicolumn{3}{c}{\textbf{AntButton}} 
& \multicolumn{3}{c}{\textbf{AntCircle}} \\
\cmidrule(lr){2-4} \cmidrule(lr){5-7}
& \textbf{Vio. Rate (\%)} & \textbf{Magnitude} & \textbf{Avg. Cost} 
& \textbf{Vio. Rate (\%)} & \textbf{Magnitude} & \textbf{Avg. Cost} \\
\midrule
CPPOLag  & 0.22 ± 0.03 &  ± 0.83 &  6.09 ± 3.94 
         & 13.98 ± 3.11 & 6.59 ± 1.89 & 17.30 ± 3.00 \\
CPPOPID  & 0.30 ± 0.28 & 0.01 ± 0.01 & 6.95 ± 2.11
         & 10.82 ± 3.77 & 0.62 ± 0.36 & 12.95 ± 0.92 \\
CPPOADRC & \textbf{0.01 ± 0.01} & \textbf{0.00 ± 0.00} & \textbf{2.69 ± 0.72}
         & \textbf{0.00 ± 0.00} & \textbf{0.00 ± 0.00} & \textbf{3.23 ± 2.24} \\
\midrule
DDPGLag  & 5.72 ± 4.67 & 0.30 ± 0.24 & 7.35 ± 2.97
         & 0.07 ± 0.15 & 0.00 ± 0.00 & 1.93 ± 0.68 \\
DDPGPID  & 6.22 ± 7.10 & 0.47 ± 0.65 & 7.92 ± 3.86
         & 0.04 ± 0.07 & 0.00 ± 0.00 & 1.59 ± 0.50 \\
DDPGADRC & \textbf{1.93 ± 1.15} & \textbf{0.12 ± 0.10} & \textbf{4.87 ± 0.64}
         & 0.15 ± 0.27 & 0.00 ± 0.00 & 2.15 ± 0.41 \\
\midrule
TD3Lag   & 3.31 ± 5.67 & 0.32 ± 0.64 & 6.22 ± 4.96
         & 31.26 ± 2.86 & 1.76 ± 0.22 & 15.02 ± 0.99 \\
TD3PID   & 2.24 ± 5.01 & 0.11 ± 0.25 & 3.14 ± 4.22
         & 21.32 ± 11.16 & 3.61 ± 2.06 & 13.74 ± 4.75 \\
TD3ADRC  & \textbf{0.00 ± 0.00} & \textbf{0.00 ± 0.00} & \textbf{0.86 ± 0.41}
         & \textbf{0.02 ± 0.02} & \textbf{0.00 ± 0.00} & \textbf{2.49 ± 2.61} \\
\midrule
TRPOLag  & 0.01 ± 0.02 & 0.00 ± 0.00 & 4.40 ± 1.48
         & 20.80 ± 3.52 & 1.64 ± 0.37 & 16.58 ± 2.01 \\
TRPOPID  & 0.01 ± 0.02 & 0.00 ± 0.00 & 4.29 ± 1.35
         & 17.73 ± 2.21 & 1.06 ± 0.20 & 16.14 ± 1.31 \\
TRPOADRC & \textbf{0.00 ± 0.00} & \textbf{0.00 ± 0.00} & \textbf{4.15 ± 0.58}
         & \textbf{0.14 ± 0.28} & \textbf{0.00 ± 0.01} & \textbf{5.74 ± 1.93} \\
\bottomrule
\end{tabular*}
}
\label{table:ant_button_circle}
\end{table*}

\begin{figure}[htbp]
    \centering
    \includegraphics[width=\linewidth]{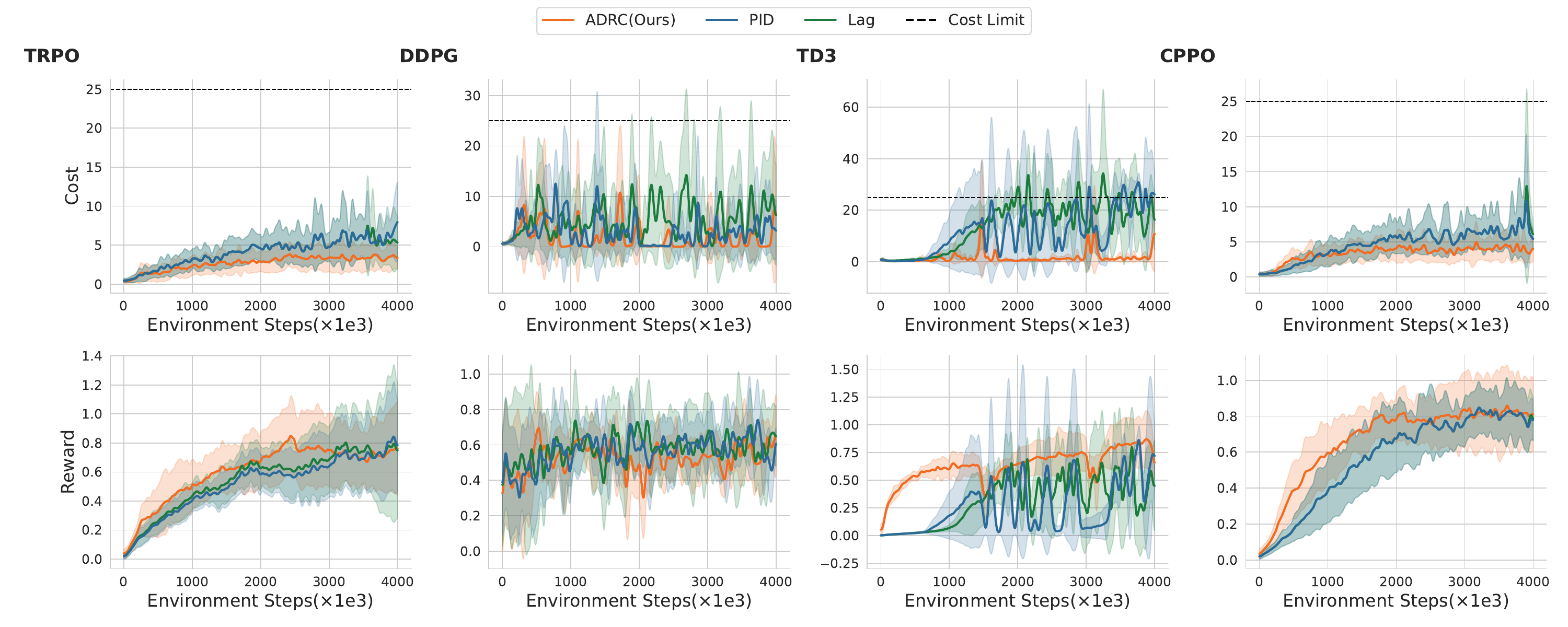}
    \caption{The training curves of AntGoal with various Lagrangian methods across different algorithms.}
    \label{fig:antgoal_performance}
\end{figure}

\begin{figure}[htbp]
    \centering
    \includegraphics[width=\linewidth]{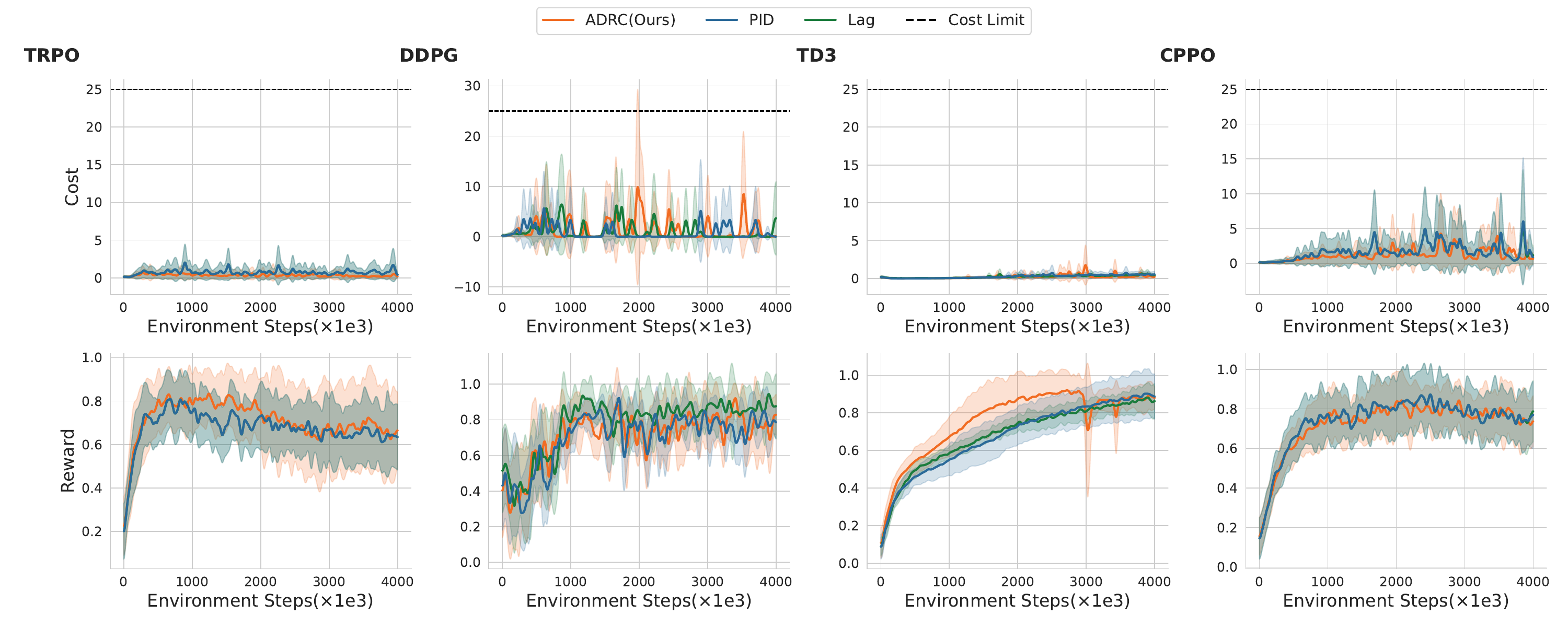}
    \caption{The training curves of AntPush with various Lagrangian methods across different algorithms.}
    \label{fig:antpush_performance}
\end{figure}

\begin{table*}[htbp]
\centering
\small
\caption{Comparison of violation rate, magnitude, and average cost on AntGoal and AntPush.}
\resizebox{\textwidth}{!}{
\begin{tabular*}{\textwidth}{@{\extracolsep{\fill}}lcccccc@{}}
\toprule
\textbf{Algorithm} 
& \multicolumn{3}{c}{\textbf{AntGoal}} 
& \multicolumn{3}{c}{\textbf{AntPush}} \\
\cmidrule(lr){2-4} \cmidrule(lr){5-7}
& \textbf{Vio. Rate (\%)} & \textbf{Magnitude} & \textbf{Avg. Cost} 
& \textbf{Vio. Rate (\%)} & \textbf{Magnitude} & \textbf{Avg. Cost} \\
\midrule
CPPOLag  & 0.41 ± 0.66 & 0.05 ± 0.01 & 4.68 ± 1.05 
         & 0.48 ± 0.48 & 3.65 ± 3.93 & 0.61 ± 0.55 \\
CPPOPID  & 0.31 ± 0.46 & 0.03 ± 0.06 & 4.63 ± 0.96
         & 0.36 ± 0.54 & 0.02 ± 0.03 & 1.70 ± 0.96 \\
CPPOADRC & \textbf{0.00 ± 0.00} & \textbf{0.00 ± 0.00} & \textbf{3.43 ± 0.47}
         & \textbf{0.09 ± 0.13} & \textbf{0.00 ± 0.01} & \textbf{1.25 ± 0.37} \\
\midrule
DDPGLag  & 3.09 ± 1.37 & 0.32 ± 0.20 & 5.34 ± 1.18
         & 0.03 ± 0.04 & 0.00 ± 0.00 & 1.05 ± 0.47 \\
DDPGPID  & 1.42 ± 0.58 & 0.19 ± 0.07 & 3.03 ± 1.01
         & 0.17 ± 0.24 & 0.02 ± 0.02 & \textbf{0.82 ± 0.51} \\
DDPGADRC & \textbf{1.19 ± 1.27} & \textbf{0.15 ± 0.15} & \textbf{1.88 ± 1.19}
         & 0.67 ± 1.19 & 0.11 ± 0.20 & 1.50 ± 0.80 \\
\midrule
TD3Lag   & 20.21 ± 4.27 & 1.94 ± 0.64 & 13.46 ± 1.39
         & 0.00 ± 0.00 & 0.00 ± 0.00 & \textbf{0.12 ± 0.07} \\
TD3PID   & 18.74 ± 12.17 & 2.00 ± 1.75 & 11.81 ± 5.51
         & 0.00 ± 0.00 & 0.00 ± 0.00 & 0.16 ± 0.14 \\
TD3ADRC  & \textbf{1.54 ± 1.50} & \textbf{0.20 ± 0.22} & \textbf{2.04 ± 1.34}
         & 0.00 ± 0.00 & 0.00 ± 0.00 & 0.18 ± 0.16 \\
\midrule
TRPOLag  & 0.25 ± 0.50 & 0.02 ± 0.04 & 4.10 ± 0.89
         & 0.00 ± 0.00 & 0.00 ± 0.00 & 0.73 ± 0.26 \\
TRPOPID  & 0.15 ± 0.30 & 0.01 ± 0.02 & 4.13 ± 0.94
         & 0.00 ± 0.00 & 0.00 ± 0.00 & 0.73 ± 0.26 \\
TRPOADRC & \textbf{0.00 ± 0.00} & \textbf{0.00 ± 0.00} & \textbf{2.70 ± 0.69}
         & 0.00 ± 0.00 & 0.00 ± 0.00 & \textbf{0.34 ± 0.15} \\
\bottomrule
\end{tabular*}
}
\label{table:ant_goal_push}
\end{table*}

To provide a more thorough and quantitative evaluation of our method, we report the results of experiments conducted on four challenging environments, AntButton, AntCircle, AntPush and AntGoal in Table~\ref{table:ant_button_circle} and Table~\ref{table:ant_goal_push}. The metrics compared include violation rate (\%), magnitude of violations, and average cost. Across all experiments, our ADRC method consistently outperforms baseline approaches (PID and Lagrange) in achieving lower violation rates and magnitudes, while maintaining competitive or reduced average costs. These results, validated across four RL algorithms (TRPO, PPO, DDPG, TD3), demonstrate the effectiveness and robustness of ADRC in handling constraint-aware reinforcement learning tasks.

\subsubsection{Car Environments}

Figures \ref{fig:carbutton_performance} to Figure~\ref{fig:carpush_performance} present the training curves for the Car environment across four tasks: Button, Circle, Goal, and Push. Each plot illustrates the episodic returns and costs averaged over five random seeds, with solid lines representing the mean and shaded areas denoting the variance.

\begin{figure}[htbp]
    \centering
    \includegraphics[width=\linewidth]{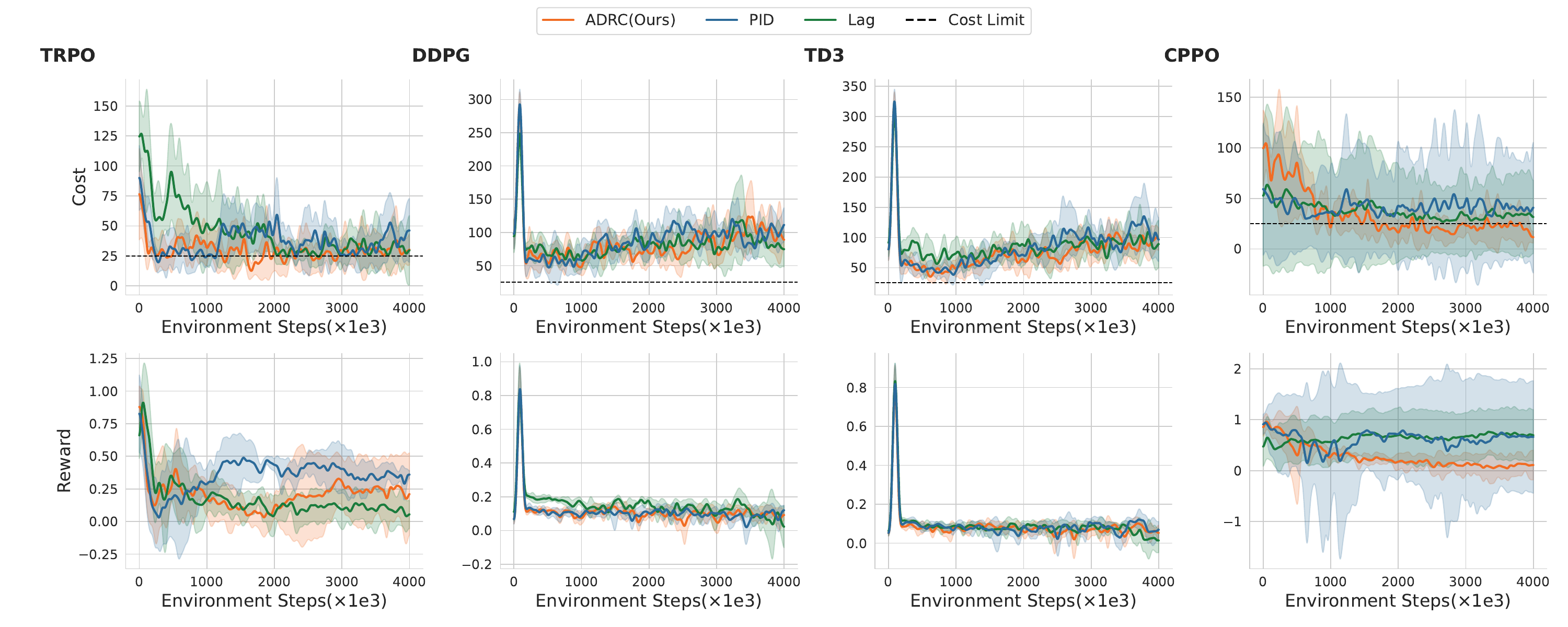}
    \caption{The training curves of CarButton with various Lagrangian methods across different algorithms.}
    \label{fig:carbutton_performance}
\end{figure}
\begin{figure}[htbp]
    \centering
    \includegraphics[width=\linewidth]{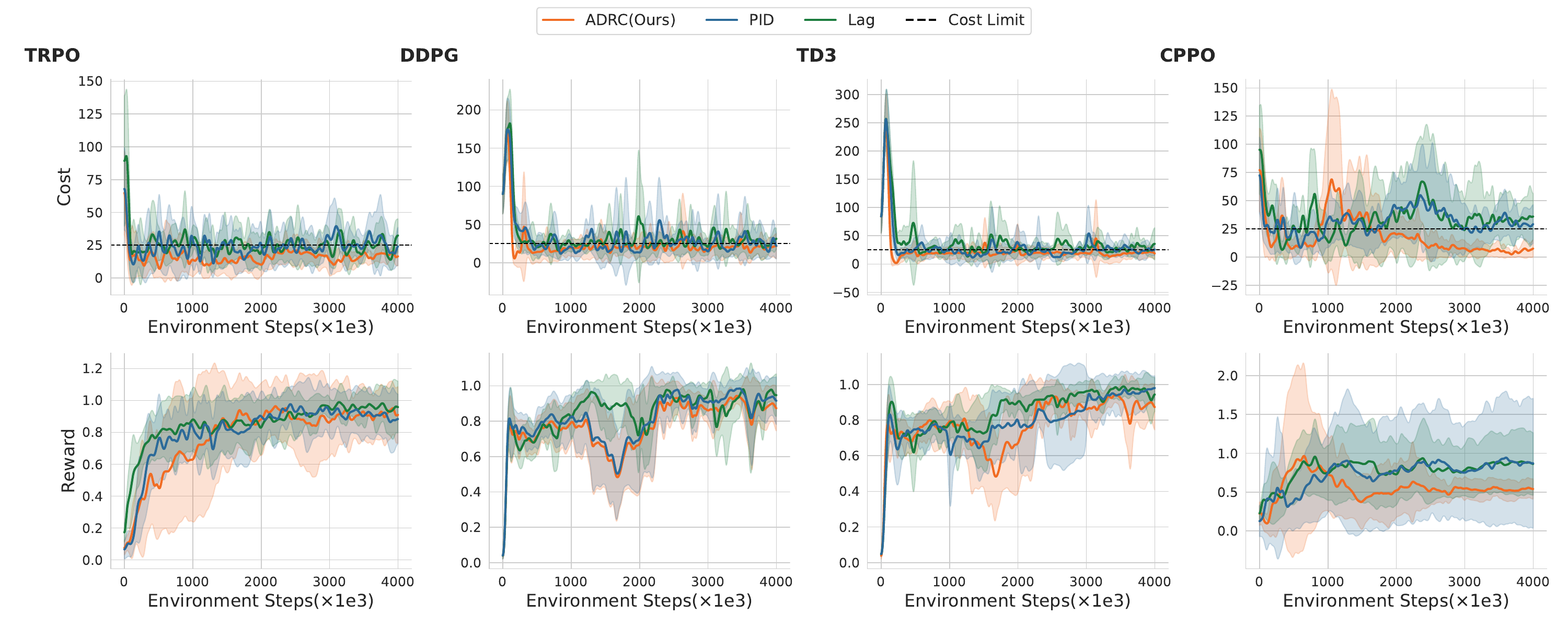}
    \caption{The training curves of CarCircle with various Lagrangian methods across different algorithms.}
    \label{fig:carcircle_performance}
\end{figure}

To provide a more thorough and quantitative evaluation of our method, we report the results of experiments conducted on four challenging environments, AntButton, AntCircle, AntPush and AntGoal in Table~\ref{table:car_button_circle} and Table~\ref{table:car_goal_push}. The metrics compared include violation rate (\%), magnitude of violations, and average cost. Across all experiments, our ADRC method consistently outperforms baseline approaches (PID and Lagrange) in achieving lower average cost and violation magnitudes, while maintaining competitive.

\begin{table*}[htbp]
\centering
\small
\caption{Comparison of violation rate, magnitude, and average cost on CarButton and CarCircle.}
\resizebox{\textwidth}{!}{
\begin{tabular}{lcccccc}

\toprule
\textbf{Algorithm} 
& \multicolumn{3}{c}{\textbf{CarButton}} 
& \multicolumn{3}{c}{\textbf{CarCircle}} \\
\cmidrule(lr){2-4} \cmidrule(lr){5-7}
& \textbf{Vio. Rate (\%)} & \textbf{Magnitude} & \textbf{Avg. Cost} 
& \textbf{Vio. Rate (\%)} & \textbf{Magnitude} & \textbf{Avg. Cost} \\
\midrule
CPPOLag  & 89.77 ± 19.38 & 59.77 ± 39.05 & 84.05 ± 40.18
         & 46.73 ± 20.16 & 15.85 ± 17.40 & 33.54 ± 20.09 \\
CPPOPID  & 85.09 ± 16.67 & 45.27 ± 46.80 & 68.95 ± 47.94
         & 52.77 ± 17.62 & 12.29 ± 8.14 & 31.24 ± 10.31 \\
CPPOADRC & \textbf{50.16 ± 17.08} & \textbf{14.80 ± 3.96} & \textbf{34.20 ± 5.75}
         & \textbf{21.35 ± 13.09} & \textbf{7.74 ± 6.46} & \textbf{18.85 ± 9.30} \\
\midrule
DDPGLag  & 99.93 ± 0.11 & 58.18 ± 12.27 & 83.17 ± 12.27
         & 51.85 ± 0.74 & 13.55 ± 2.86 & 33.49 ± 2.45 \\
DDPGPID  & 98.87 ± 2.05 & 64.10 ± 6.43 & 89.05 ± 6.51
         & 39.77 ± 4.31 & 13.56 ± 1.49 & 30.89 ± 1.44 \\
DDPGADRC & 99.49 ± 0.08 & \textbf{58.62 ± 7.25} & \textbf{83.60 ± 7.25}
         & 51.50 ± 5.10 & \textbf{7.73 ± 1.79} & \textbf{23.82 ± 1.33} \\
\midrule
TD3Lag   & 99.97 ± 0.04 & 62.20 ± 10.10 & 87.20 ± 10.11
         & 53.00 ± 1.06 & 17.18 ± 1.25 & 38.06 ± 1.57 \\
TD3PID   & 99.03 ± 1.00 & 59.44 ± 11.92 & 84.40 ± 11.95
         & 39.20 ± 1.56 & 12.52 ± 1.64 & 30.87 ± 1.18 \\
TD3ADRC  & \textbf{99.04 ± 0.73} & \textbf{50.84 ± 9.18} & \textbf{75.80 ± 9.21}
         & \textbf{48.41 ± 6.50} & \textbf{7.49 ± 1.54} & \textbf{24.26 ± 1.28} \\
\midrule
TRPOLag  & 74.24 ± 11.05 & 21.90 ± 7.38 & 44.84 ± 8.09
         & 40.04 ± 1.30 & 6.45 ± 1.28 & 25.52 ± 0.51 \\
TRPOPID  & 69.04 ± 17.53 & 14.60 ± 4.66 & 37.11 ± 6.10
         & 40.67 ± 3.51 & 6.58 ± 1.15 & 24.24 ± 1.44 \\
TRPOADRC & \textbf{53.28 ± 15.44} & \textbf{8.53 ± 4.33} & \textbf{29.57 ± 6.18}
         & \textbf{17.71 ± 2.92} & \textbf{1.92 ± 0.64} & \textbf{16.22 ± 1.75} \\
\bottomrule
\end{tabular}
}
\label{table:car_button_circle}
\end{table*}

\begin{figure}[htbp]
    \centering
    \includegraphics[width=\linewidth]{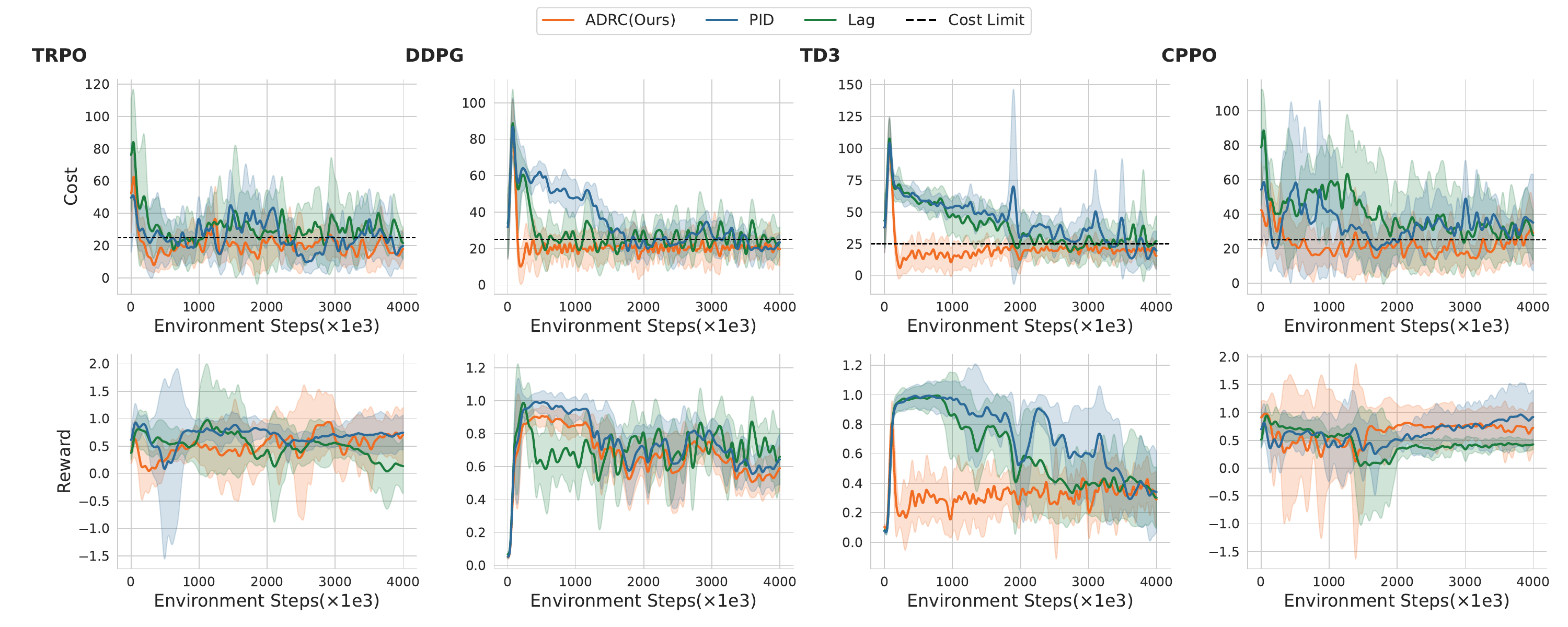}
    \caption{The training curves of CarGoal with various Lagrangian methods across different algorithms.}
    \label{fig:cargoal_performance}
\end{figure}

\begin{figure}[htbp]
    \centering
    \includegraphics[width=\linewidth]{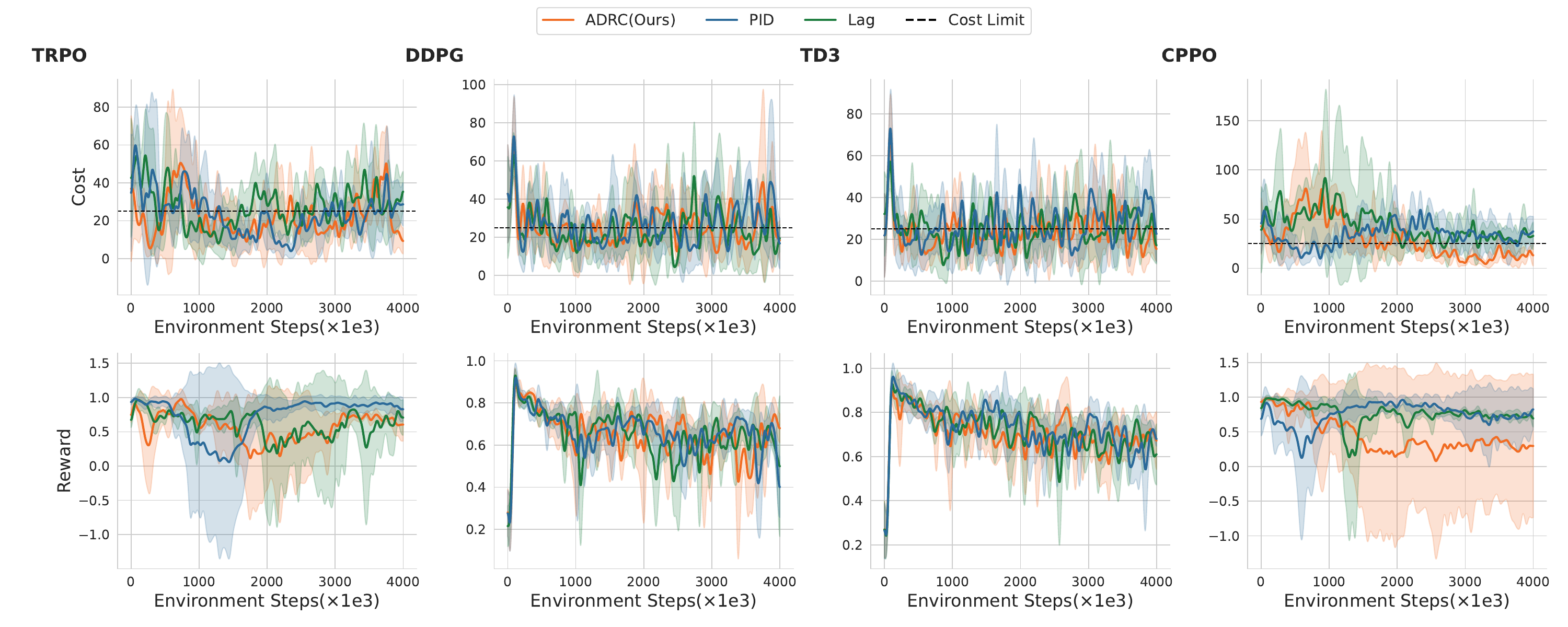}
    \caption{The training curves of CarPush with various Lagrangian methods across different algorithms}
    \label{fig:carpush_performance}
\end{figure}

\begin{table*}[htbp]
\centering
\small
\caption{Comparison of violation rate, magnitude, and average cost on CarGoal and CarPush.}
\resizebox{\textwidth}{!}{
\begin{tabular}{lcccccc}

\toprule
\textbf{Algorithm} 
& \multicolumn{3}{c}{\textbf{CarGoal}} 
& \multicolumn{3}{c}{\textbf{CarPush}} \\
\cmidrule(lr){2-4} \cmidrule(lr){5-7}
& \textbf{Vio. Rate (\%)} & \textbf{Magnitude} & \textbf{Avg. Cost} 
& \textbf{Vio. Rate (\%)} & \textbf{Magnitude} & \textbf{Avg. Cost} \\
\midrule
CPPOLag  & 65.20 ± 22.57 & 18.54 ± 11.41 & 40.05 ± 14.00
         & 68.45 ± 18.98 & 21.48 ± 16.81 & 43.38 ± 18.71 \\
CPPOPID  & 62.04 ± 18.18 & 12.95 ± 7.45 & 34.41 ± 9.04
         & 62.36 ± 10.66 & 12.40 ± 3.08 & 33.74 ± 3.88 \\
CPPOADRC & \textbf{34.97 ± 13.88} & \textbf{4.72 ± 2.14} & \textbf{21.99 ± 5.32}
         & \textbf{42.23 ± 15.34} & \textbf{11.68 ± 7.99} & \textbf{29.16 ± 10.19} \\
\midrule
DDPGLag  & 52.43 ± 0.12 & 7.25 ± 0.93 & 28.35 ± 0.44
         & 48.81 ± 1.62 & 8.20 ± 1.38 & 27.37 ± 1.19 \\
DDPGPID  & 65.52 ± 4.60 & 12.53 ± 0.72 & 35.03 ± 0.53
         & 48.81 ± 1.62 & 8.20 ± 1.38 & 27.37 ± 1.19 \\
DDPGADRC & \textbf{47.36 ± 1.90} & \textbf{2.88 ± 0.57} & \textbf{21.55 ± 0.29}
         & \textbf{42.43 ± 9.36} & \textbf{7.07 ± 1.69} & \textbf{25.70 ± 2.85} \\
\midrule
TD3Lag   & 70.47 ± 10.44 & 17.00 ± 1.29 & 38.90 ± 2.83
         & 46.97 ± 2.20 & 6.27 ± 1.53 & 25.51 ± 1.36 \\
TD3PID   & 80.94 ± 5.87 & 20.68 ± 3.94 & 43.43 ± 4.23
         & 49.83 ± 0.77 & 7.65 ± 1.15 & 26.66 ± 1.14 \\
TD3ADRC  & \textbf{40.62 ± 8.51} & \textbf{2.85 ± 0.31} & \textbf{20.65 ± 2.46}
         & \textbf{39.92 ± 1.66} & \textbf{5.15 ± 0.82} & \textbf{23.64 ± 0.93} \\
\midrule
TRPOLag  & 54.86 ± 6.74 & 10.46 ± 4.56 & 30.97 ± 4.69
         & 48.24 ± 6.24 & 8.51 ± 2.35 & 27.58 ± 3.16 \\
TRPOPID  & 44.79 ± 2.84 & 7.34 ± 1.31 & 25.84 ± 0.88
         & 40.05 ± 1.69 & 6.66 ± 1.82 & 23.92 ± 1.74 \\
TRPOADRC & \textbf{29.12 ± 3.70} & \textbf{3.44 ± 1.21} & \textbf{20.48 ± 0.99}
         & \textbf{34.75 ± 8.43} & \textbf{6.32 ± 2.13} & \textbf{22.40 ± 2.34} \\
\bottomrule
\end{tabular}
}
\label{table:car_goal_push}
\end{table*}

\subsubsection{Racecar Environments}

Figures \ref{fig:racecarbutton_performance} to Figure~\ref{fig:racecarpush_performance} present the training curves for the Ant environment across four tasks: Button, Circle, Goal, and Push. Each plot illustrates the episodic returns and costs averaged over five random seeds, with solid lines representing the mean and shaded areas denoting the variance.

\begin{figure}[htbp]
    \centering
    \includegraphics[width=\linewidth]{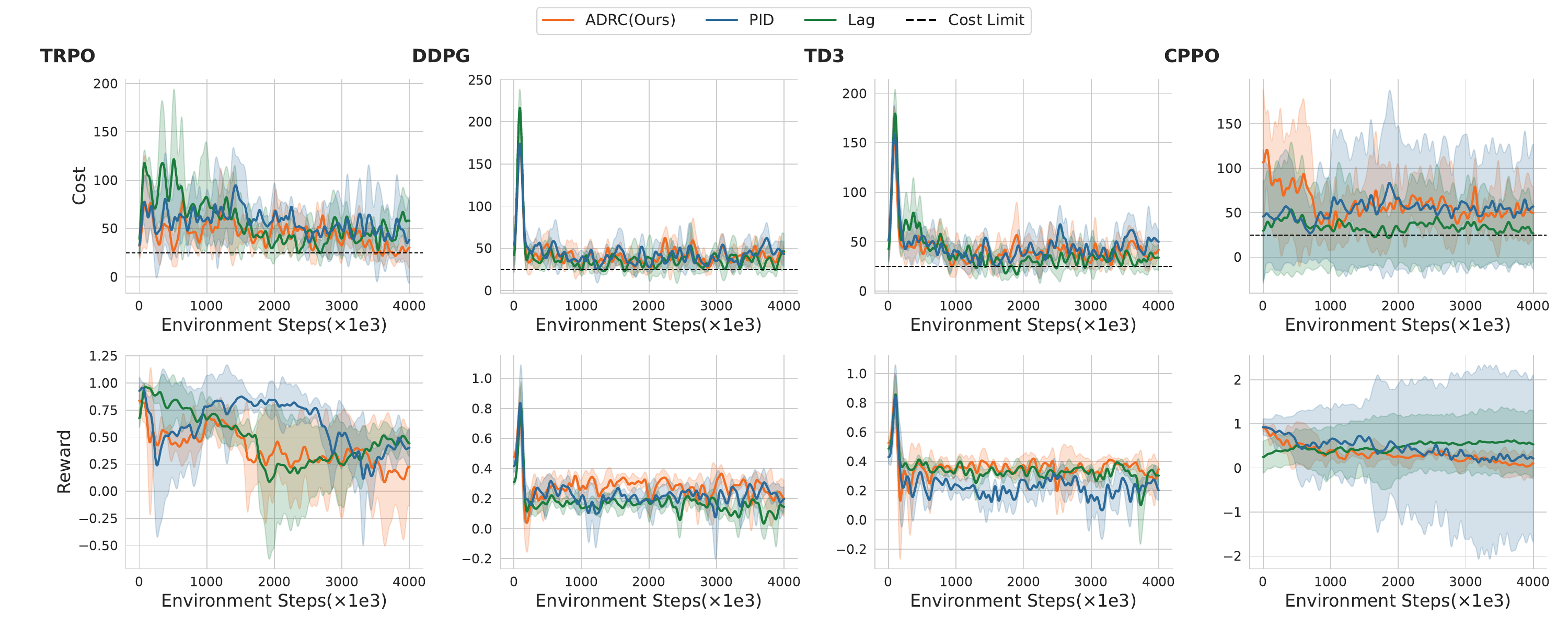}
    \caption{The training curves of RacecarButton with various Lagrangian methods across different algorithms.}
    \label{fig:racecarbutton_performance}
\end{figure}

\begin{figure}[htbp]
    \centering
    \includegraphics[width=\linewidth]{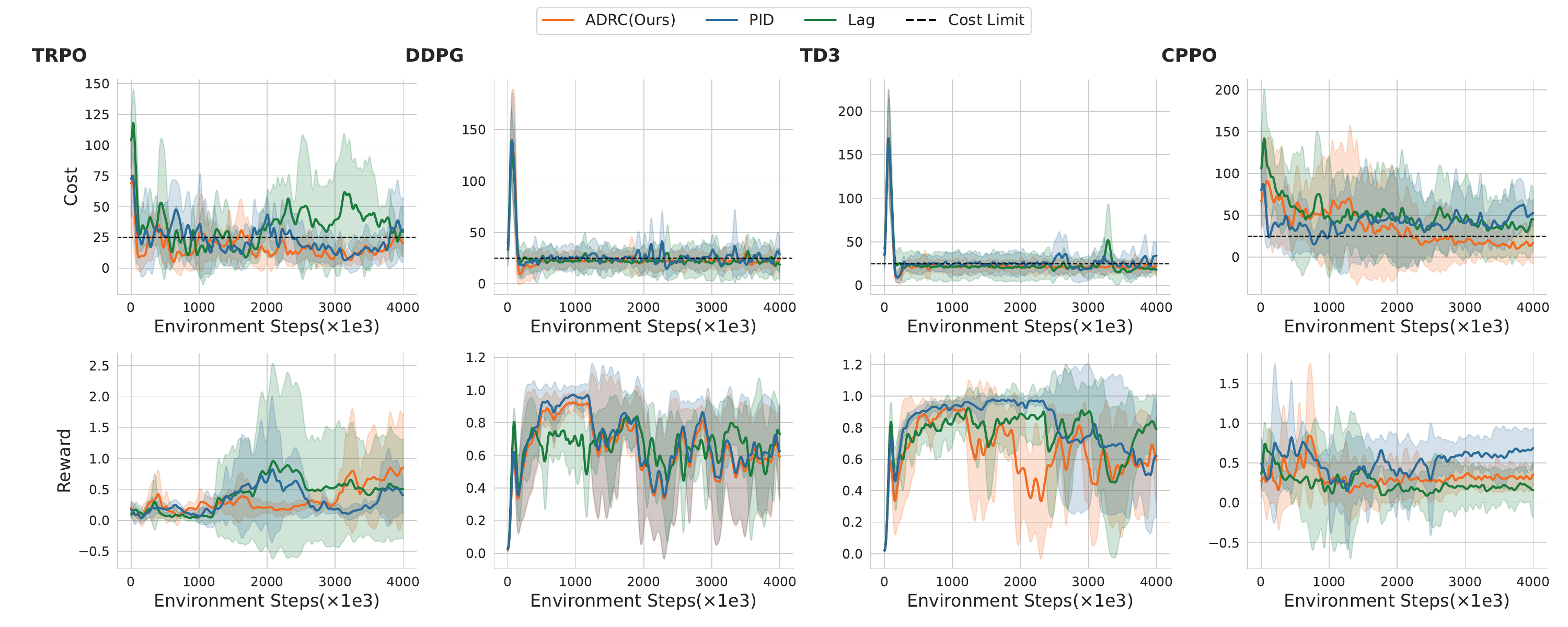}
    \caption{The training curves of RacecarCircle with various Lagrangian methods across different algorithms.}
    \label{fig:racecarcircle_performance}
\end{figure}

\begin{table*}[htbp]
\centering
\small
\caption{Comparison of violation rate, magnitude, and average cost on RacecarButton and RacecarCircle.}
\resizebox{\textwidth}{!}{
\begin{tabular}{lcccccc}

\toprule
\textbf{Algorithm} 
& \multicolumn{3}{c}{\textbf{RacecarButton}} 
& \multicolumn{3}{c}{\textbf{RacecarCircle}} \\
\cmidrule(lr){2-4} \cmidrule(lr){5-7}
& \textbf{Vio. Rate (\%)} & \textbf{Magnitude} & \textbf{Avg. Cost} 
& \textbf{Vio. Rate (\%)} & \textbf{Magnitude} & \textbf{Avg. Cost} \\
\midrule
CPPOLag  & 97.38 ± 2.85 & 65.75 ± 16.98 & 90.53 ± 17.10
         & 58.32 ± 31.47 & 30.72 ± 37.77 & 50.41 ± 41.68 \\
CPPOPID  & 97.37 ± 4.24 & 78.44 ± 43.36 & 103.21 ± 43.64
         & 56.82 ± 31.13 & 21.77 ± 25.28 & 40.95 ± 29.63 \\
CPPOADRC & \textbf{81.18 ± 20.97} & \textbf{35.76 ± 29.48} & \textbf{59.11 ± 31.27}
         & \textbf{39.94 ± 38.28} & \textbf{21.78 ± 31.79} & \textbf{35.72 ± 39.00} \\
\midrule
DDPGLag  & 75.88 ± 4.41 & 15.77 ± 3.10 & 39.47 ± 3.34
         & 50.52 ± 0.25 & 7.87 ± 0.74 & 25.03 ± 0.67 \\
DDPGPID  & 88.26 ± 2.48 & 20.12 ± 2.35 & 44.62 ± 2.47
         & 46.17 ± 1.40 & 7.94 ± 1.09 & 27.09 ± 0.56 \\
DDPGADRC & 85.18 ± 6.77 & \textbf{18.49 ± 2.43} & \textbf{42.87 ± 2.74}
         & 50.92 ± 2.27 & \textbf{5.00 ± 0.79} & \textbf{24.51 ± 0.33} \\
\midrule
TD3Lag   & 72.49 ± 2.60 & 16.31 ± 2.92 & 39.65 ± 3.12
         & 49.42 ± 0.19 & 8.22 ± 0.52 & 24.95 ± 0.52 \\
TD3PID   & 85.99 ± 7.60 & 21.46 ± 4.96 & 45.75 ± 5.43
         & 48.72 ± 1.24 & 8.57 ± 1.07 & 27.76 ± 1.30 \\
TD3ADRC  & 84.48 ± 4.33 & \textbf{18.37 ± 2.74} & \textbf{42.65 ± 2.93}
         & \textbf{42.08 ± 1.39} & \textbf{4.56 ± 0.79} & \textbf{24.21 ± 0.82} \\
\midrule
TRPOLag  & 87.31 ± 4.65 & 33.07 ± 7.84 & 57.34 ± 7.59
         & 51.39 ± 10.47 & 16.10 ± 10.65 & 34.48 ± 11.14 \\
TRPOPID  & 87.04 ± 3.69 & 32.94 ± 3.94 & 56.86 ± 3.65
         & 37.99 ± 3.47 & 7.92 ± 2.78 & 23.45 ± 1.84 \\
TRPOADRC & \textbf{80.06 ± 4.19} & \textbf{21.38 ± 5.00} & \textbf{45.00 ± 5.21}
         & \textbf{23.64 ± 3.83} & \textbf{3.54 ± 0.43} & \textbf{17.30 ± 1.32} \\
\bottomrule
\end{tabular}
}
\label{table:racecar_button_circle}
\end{table*}

\begin{figure}[htbp]
    \centering
    \includegraphics[width=\linewidth]{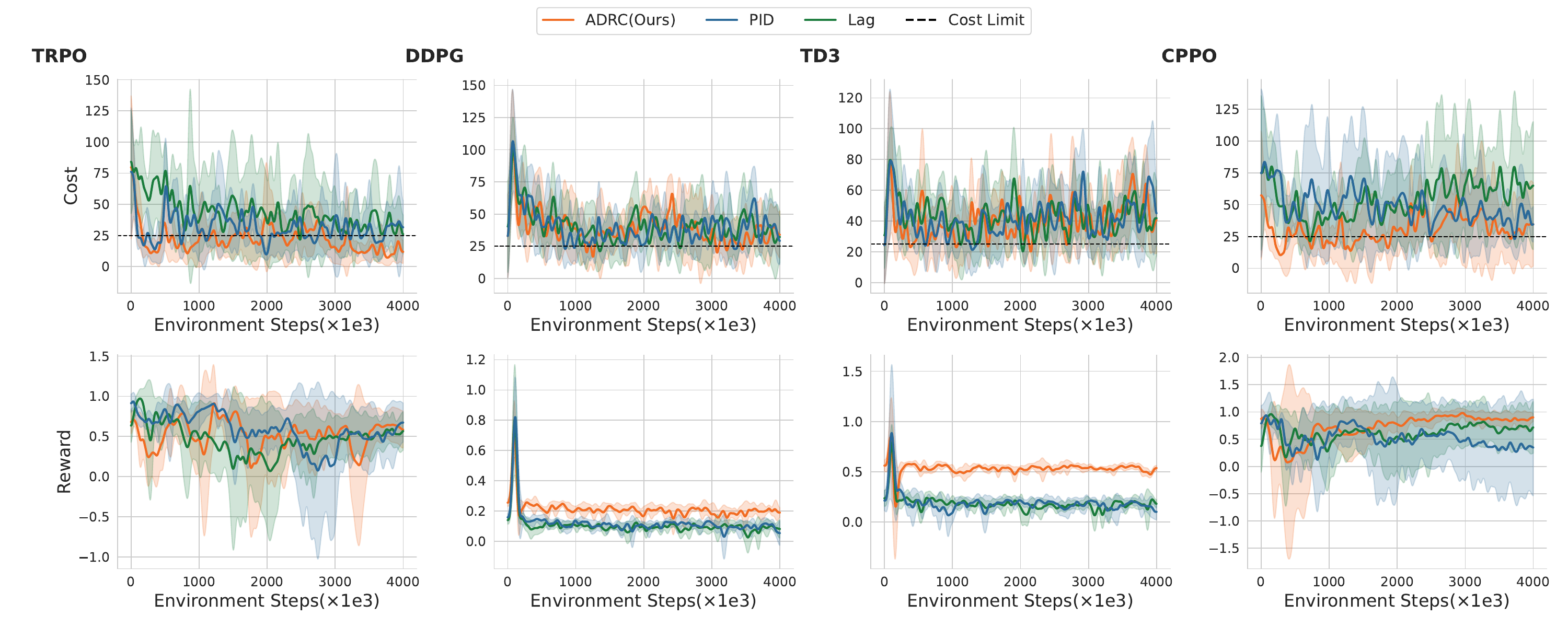}
    \caption{The training curves of RacecarGoal with various Lagrangian methods across different algorithms.}
    \label{fig:racecargoal_performance}
\end{figure}

\begin{figure}[htbp]
    \centering
    \includegraphics[width=\linewidth]{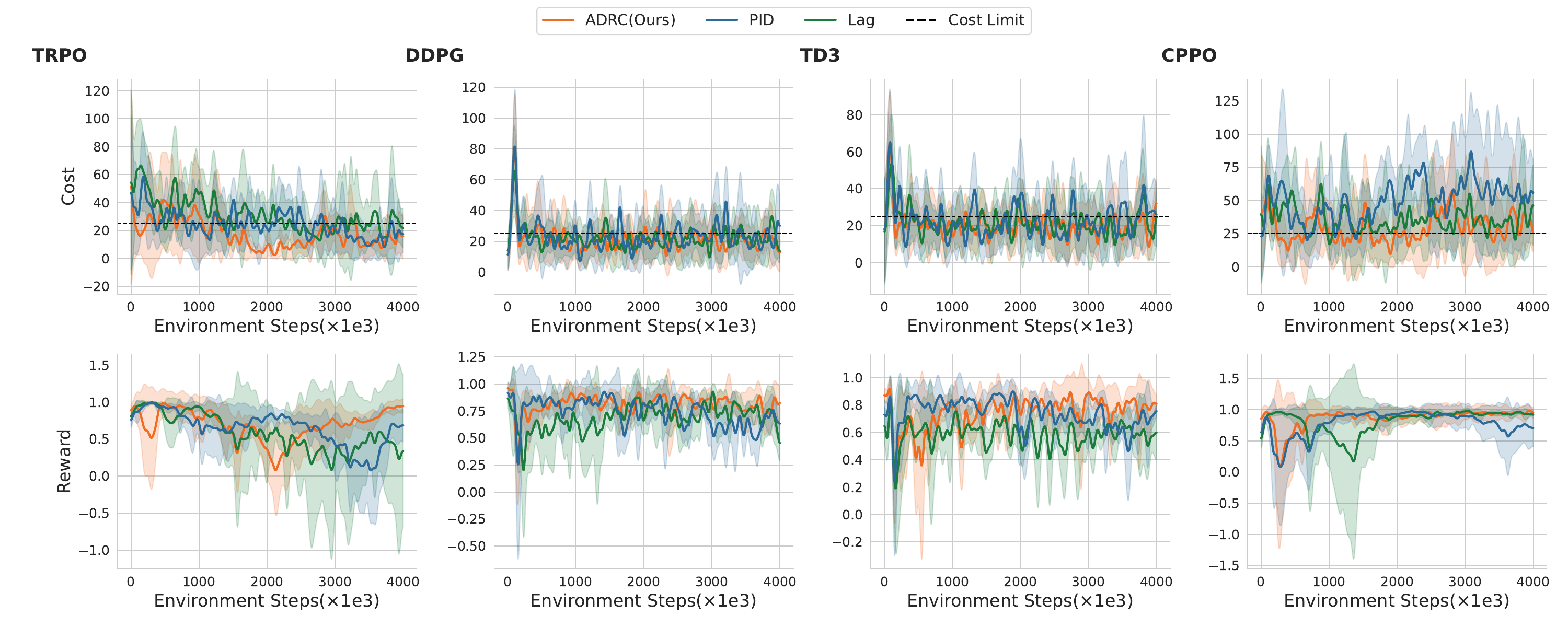}
    \caption{The training curves of RacecarPush with various Lagrangian methods across different algorithms.}
    \label{fig:racecarpush_performance}
\end{figure}

\begin{table*}[htbp]
\centering
\small
\caption{Comparison of violation rate, magnitude, and average cost on RacecarGoal and RacecarPush.}
\resizebox{\textwidth}{!}{
\begin{tabular}{lcccccc}
\toprule
\textbf{Algorithm} 
& \multicolumn{3}{c}{\textbf{RacecarGoal}} 
& \multicolumn{3}{c}{\textbf{RacecarPush}} \\
\cmidrule(lr){2-4} \cmidrule(lr){5-7}
& \textbf{Vio. Rate (\%)} & \textbf{Magnitude} & \textbf{Avg. Cost} 
& \textbf{Vio. Rate (\%)} & \textbf{Magnitude} & \textbf{Avg. Cost} \\
\midrule
CPPOLag  & 80.87 ± 19.17 & 31.18 ± 18.99 & 54.24 ± 21.01
         & 57.91 ± 22.12 & 15.45 ± 12.92 & 35.54 ± 15.56 \\
CPPOPID  & 72.30 ± 24.96 & 27.11 ± 15.58 & 49.02 ± 18.75
         & 70.84 ± 23.97 & 28.16 ± 16.84 & 49.67 ± 19.96 \\
CPPOADRC & \textbf{47.08 ± 21.59} & \textbf{12.31 ± 9.34} & \textbf{30.12 ± 12.74}
         & \textbf{47.28 ± 17.05} & \textbf{12.50 ± 7.05} & \textbf{29.35 ± 11.03} \\
\midrule
DDPGLag  & 71.44 ± 9.47 & 18.47 ± 4.13 & 40.76 ± 5.00
         & 39.58 ± 2.98 & 5.00 ± 0.49 & 22.52 ± 0.58 \\
DDPGPID  & 72.05 ± 4.09 & 17.93 ± 2.81 & 40.29 ± 2.99
         & 39.31 ± 1.84 & 6.89 ± 0.81 & 23.46 ± 0.98 \\
DDPGADRC & \textbf{68.41 ± 5.77} & \textbf{17.81 ± 1.34} & \textbf{39.41 ± 2.14}
         & \textbf{38.08 ± 3.96} & \textbf{5.10 ± 0.98} & \textbf{21.99 ± 1.13} \\
\midrule
TD3Lag   & 75.26 ± 2.13 & 19.93 ± 1.40 & 42.57 ± 1.59
         & 36.76 ± 1.53 & 5.25 ± 0.92 & 21.64 ± 0.97 \\
TD3PID   & 73.31 ± 6.06 & 19.19 ± 2.72 & 41.57 ± 3.22
         & 38.87 ± 1.99 & 6.76 ± 1.65 & 23.32 ± 2.45 \\
TD3ADRC  & \textbf{71.24 ± 3.00} & \textbf{17.55 ± 3.57} & \textbf{39.71 ± 3.94}
         & \textbf{36.21 ± 3.11} & \textbf{4.25 ± 1.50} & \textbf{20.70 ± 1.49} \\
\midrule
TRPOLag  & 64.31 ± 23.37 & 22.89 ± 18.69 & 43.67 ± 21.56
         & 51.18 ± 11.61 & 12.82 ± 3.76 & 31.19 ± 5.63 \\
TRPOPID  & 49.33 ± 15.00 & 11.70 ± 7.07 & 30.94 ± 8.81
         & 40.06 ± 1.74 & 7.98 ± 1.92 & 24.46 ± 1.55 \\
TRPOADRC & \textbf{34.03 ± 8.06} & \textbf{6.16 ± 2.37} & \textbf{22.02 ± 3.61}
         & \textbf{26.06 ± 11.09} & \textbf{5.41 ± 2.66} & \textbf{18.28 ± 5.41} \\
\bottomrule
\end{tabular}
}
\label{table:racecar_goal_push}
\end{table*}

To provide a more thorough and quantitative evaluation of our method, we report the results of experiments conducted on four challenging environments, AntButton, AntCircle, AntPush and AntGoal in Table~\ref{table:racecar_button_circle} and Table~\ref{table:racecar_goal_push}. The metrics compared include violation rate (\%), magnitude of violations, and average cost. Across all experiments, our ADRC method consistently outperforms baseline approaches (PID and Lagrange) in achieving lower violation rates and magnitudes, while maintaining competitive or reduced average costs.

\subsection{Velocity Control Results}
\label{sec:velocity_control_results}

To further evaluate our method’s performance in dynamic and velocity-sensitive environments, we conducted experiments on the Safety Velocity Control tasks, including SafetySwimmer and SafetyHopper. These tasks pose additional challenges by requiring agents to manage both positional constraints and velocity profiles. 

The following tables present the violation rates, violation magnitudes, average costs, and average rewards achieved by different methods. Across all settings, our ADRC Lagrangian method consistently demonstrates superior safety performance with competitive or improved final rewards compared to baseline PID Lagrangian methods.

\begin{table}[htbp]
\centering
\caption{Performance comparison on SafetySwimmer environment.}
\begin{tabular}{lcccc}
\toprule
Algorithm & Violate Rate (\%) & Magnitude & Avg Cost & Avg Reward \\
\midrule
CPPOPID & 28.33 & 1.84 & 23.64 & 32.54 \\
CPPOADRC & \textbf{6.95} & \textbf{1.56} & \textbf{18.34} & 29.07 \\
TRPOPID & 35.43 & 1.78 & 22.48 & 27.73 \\
TRPOADRC & \textbf{11.30} & 2.44 & 20.82 & \textbf{35.66} \\
\bottomrule
\end{tabular}
\label{tab:safetyswimmer}
\end{table}

\begin{table}[htbp]
\centering
\caption{Performance comparison on SafetyHopper environment.}
\begin{tabular}{lcccc}
\toprule
Algorithm & Violate Rate (\%) & Magnitude & Avg Cost & Avg Reward \\
\midrule
CPPOPID & 40.33 & 5.32 & 23.92 & \textbf{1365.60} \\
CPPOADRC & \textbf{17.40} & \textbf{1.84} & \textbf{17.35} & 1155.59 \\
TRPOPID & 39.33 & 7.52 & 24.41 & \textbf{1448.46} \\
TRPOADRC & \textbf{0.93} & \textbf{0.06} & \textbf{12.02} & 1080.80 \\
\bottomrule
\end{tabular}
\label{tab:safetyhopper}
\end{table}

These results validate that the ADRC-based methods significantly improve safety metrics (lower violation rate and cost) while maintaining comparable or strong reward performance in dynamic velocity control tasks. This further demonstrates the effectiveness and robustness of ADRC Lagrangian formulations under more complex and realistic settings.

\subsection{Comparison with State-of-the-Art Safe RL Algorithms}
\label{sec:compared_with_saferl}

To demonstrate the broader applicability and effectiveness of our ADRC-Lagrangian framework beyond traditional Lagrangian methods, we conduct comprehensive comparisons with state-of-the-art safe RL algorithms. Our evaluation includes both Lagrangian-based methods (RCPO and PDO)~\cite{rcpo,PDO} and non-Lagrangian approaches such as CUP~\cite{cup} and IPO~\cite{ipo}. This comparison validates our method's superiority across different safe RL paradigms and confirms that the benefits stem from ADRC's adaptive control principles rather than merely being artifacts of the Lagrangian framework.

We evaluate all methods on two challenging continuous control environments: HalfCheetah-Velocity and Hopper-Velocity from the Safety-Gymnasium benchmark. Each algorithm is trained with identical hyperparameters and evaluated using three random seeds. We report both training metrics (averaged over the entire training process) and final policy evaluation results to provide comprehensive performance assessment.

\begin{table}[htbp]
\centering
\small
\caption{Training performance on HalfCheetah-Velocity. Best results in \textbf{bold}, runner-up in \underline{underline}.}
\begin{tabular}{lcccc}
\toprule
\textbf{Algorithm} & \textbf{Vio. Rate (\%)} & \textbf{Magnitude} & \textbf{Avg. Cost} & \textbf{Avg. Reward} \\
\midrule
CUP       & 22.63±7.21 & 4.48±4.58 & 16.25±5.08 & 1532.23±255.05 \\
IPO       & 29.17±1.63 & 0.81±0.02 & 19.60±0.08 & 1460.56±210.91 \\
PDO       & 31.95±5.09 & 9.68±2.84 & 22.16±1.47 & \underline{1690.62±421.72} \\
RCPO      & 18.26±9.69 & 4.75±3.71 & 15.77±7.03 & 1497.99±410.94 \\
\midrule
RCPO-ADRC & \textbf{0.00±0.00} & \textbf{0.00±0.00} & \textbf{8.40±2.40} & 1329.62±293.58 \\
TRPO-ADRC & \underline{1.19±0.25} & \underline{0.09±0.10} & \underline{10.89±0.34} & \textbf{1743.09±295.33} \\
CPPO-ADRC & 8.53±12.06 & 0.55±0.78 & 15.36±2.05 & 1504.17±198.53 \\
\bottomrule
\end{tabular}
\label{tab:train_halfcheetah}
\end{table}

\begin{table}[htbp]
\centering
\small
\caption{Training performance on Hopper-Velocity. Best results in \textbf{bold}, runner-up in \underline{underline}.}
\begin{tabular}{lcccc}
\toprule
\textbf{Algorithm} & \textbf{Vio. Rate (\%)} & \textbf{Magnitude} & \textbf{Avg. Cost} & \textbf{Avg. Reward} \\
\midrule
CUP       & 37.11±3.59 & 5.47±0.59 & 21.73±1.03 & 1085.14±204.86 \\
IPO       & 51.99±8.89 & 1.68±0.18 & 24.70±0.84 & 1082.95±103.43 \\
PDO       & 27.56±9.25 & 8.42±3.48 & 20.00±7.01 & 1098.22±178.78 \\
RCPO      & 37.57±7.00 & 5.45±0.93 & 23.24±2.72 & \textbf{1247.72±292.68} \\
\midrule
RCPO-ADRC & \textbf{2.57±2.19} & \textbf{0.06±0.05} & \textbf{14.23±2.74} & 1186.87±70.94 \\
TRPO-ADRC & \underline{7.76±9.59} & \underline{0.33±0.40} & \underline{15.13±3.04} & \underline{1167.62±90.74} \\
CPPO-ADRC & 11.01±5.65 & 0.92±0.67 & 15.43±1.40 & 1083.94±68.49 \\
\bottomrule
\end{tabular}
\label{tab:train_hopper}
\end{table}
\begin{table}[htbp]
\centering
\small
\caption{Evaluation performance on HalfCheetah-Velocity. Best results in \textbf{bold}, runner-up in \underline{underline}.}
\begin{tabular}{lccc}
\toprule
\textbf{Algorithm} & \textbf{Reward} & \textbf{Cost} & \textbf{Length} \\
\midrule
CUP       & 2175.61±491.09 & 28.57±21.15 & 1000.00±0.00 \\
IPO       & 1819.75±292.45   & 16.10±12.99 & 1000.00±0.00 \\
PDO       & \textbf{2468.78±581.21} & \textbf{5.77±7.32} & 1000.00±0.00 \\
RCPO      & {2296.82±665.64} & 15.00±8.81  & 1000.00±0.00 \\
\midrule
RCPO-ADRC & 1642.22±211.39 & \underline{10.23±3.51}  & 1000.00±0.00 \\
TRPO-ADRC & \underline{2394.09±419.84} & 14.63±15.65 & 1000.00±0.00 \\
CPPO-ADRC & 2098.71±464.92 & 13.87±11.43 & 1000.00±0.00 \\
\bottomrule
\end{tabular}
\label{tab:eval_halfcheetah}
\end{table}
\begin{table}[htbp]
\centering
\small
\caption{Evaluation performance on Hopper-Velocity. Best results in \textbf{bold}, runner-up in \underline{underline}.}
\begin{tabular}{lccc}
\toprule
\textbf{Algorithm} & \textbf{Reward} & \textbf{Cost} & \textbf{Length} \\
\midrule
CUP       & 1326.84±386.22 & 26.90±10.59 & 854.40±165.37 \\
IPO       & 1216.01±129.12 & 27.57±4.95  & 797.33±118.41 \\
PDO       & 1177.19±135.02 & {18.03±25.50} & 797.57±167.85 \\
RCPO      & \textbf{1554.56±223.49} & 37.53±24.35 & 979.30±29.27 \\
\midrule
RCPO-ADRC & 1248.94±220.47 & \textbf{9.27±5.18}  & 818.10±128.67 \\
TRPO-ADRC & \underline{1470.61±152.19} & \underline{10.67±3.76}  & \textbf{1000.00±0.00} \\
CPPO-ADRC & 1322.83±157.06 & 10.43±7.12  & \underline{910.20±127.00} \\
\bottomrule
\end{tabular}
\label{tab:eval_hopper}
\end{table}

Tables~\ref{tab:train_halfcheetah} and \ref{tab:train_hopper} show that our ADRC variants markedly enhance \emph{training-time stability}: compared with existing safe RL methods, ADRC achieves consistently lower violation rates, smaller violation magnitudes, and reduced average costs. 
For example, on HalfCheetah, RCPO-ADRC eliminates violations entirely (0.00±0.00\% vs.\ 18.26±9.69\% for RCPO) and attains the lowest training cost (8.40±2.40); on Hopper, RCPO-ADRC sharply suppresses violations (2.57±2.19\%) with the smallest magnitudes (0.06±0.05) and cost (14.23±2.74). 
Crucially, this improved safety does \emph{not} come at the expense of learning quality: ADRC maintains competitive training rewards—and can be better—e.g., TRPO-ADRC attains the highest training reward on HalfCheetah (1743.09±295.33) with only 1.19±0.25\% violations, indicating stable and efficient optimization.

Tables~\ref{tab:eval_halfcheetah} and \ref{tab:eval_hopper} further examine \emph{convergence-time} performance (evaluation). 
Even without explicitly measuring constraints at evaluation, ADRC remains competitive—or superior—on task metrics: on HalfCheetah, TRPO-ADRC reaches runner-up reward (2394.09±419.84), close to the best; on Hopper, RCPO-ADRC achieves the lowest evaluation cost (9.27±5.18) and TRPO-ADRC sustains the maximum horizon (1000.00±0.00) with strong reward (1470.61±152.19). 
Together, these results confirm that ADRC improves training stability and safety while preserving (and in cases improving) final task performance and convergence behavior, offering a plug-and-play safety enhancement over existing safe RL baselines.

\subsection{Parameter Sensitivity Analysis}

\label{sec:ablation_study}
\subsubsection{Tuning parameter $k_{ap}$}
To assess the effect of the control gain $k_{ap}$ on the performance of ADRC-based Lagrangian methods, we conducted a series of ablation experiments. Specifically, we evaluated three distinct values of $k_{ad}$ ($0.01, 0.1, 1$) and compared them with existing approaches, including PID-based and classical Lagrangian methods. These experiments were carried out in two challenging environments, \textit{CarPush} and \textit{RacecarGoal}, using two reinforcement learning algorithms, \textit{PPO} and \textit{TRPO}. The results highlight ADRC's ability to dynamically adjust the control gain, demonstrating superior adaptability and improved performance with carefully selected parameter settings.
\label{sec:kap}

\begin{table*}[htbp]
\centering
\scriptsize
\caption{The proportion of constraint violations during training (Vio. Rate), the average magnitude of violations (Magnitude), and the average cost (Avg. Cost) for TRPO and PPO algorithms across CarPush and RacecarGoal environments with various $k_{ap}$ values, PID, and Lag methods. Bold values indicate better performance compared to PID.}
\resizebox{\textwidth}{!}{
\begin{tabular*}{\textwidth}{@{\extracolsep{\fill}}ccccc ccc@{}}
\toprule
\multirow[c]{2}{*}{\textbf{Algorithm}} & \multirow[c]{2}{*}{\textbf{Method}} & \multicolumn{3}{c}{\textbf{CarPush}} & \multicolumn{3}{c}{\textbf{RacecarGoal}} \\ 
\cmidrule(lr){3-5} \cmidrule(lr){6-8}
 & & \textbf{Vio. Rate (\%)} & \textbf{Magnitude} & \textbf{Avg. Cost} & \textbf{Vio. Rate (\%)} & \textbf{Magnitude} & \textbf{Avg. Cost} \\
\midrule
\multirow{5}{*}{TRPO} 
& $k_{ap}=1$ & \textbf{36.38} & \textbf{4.51}& \textbf{21.99} & \textbf{32.73} & \textbf{5.78 }& \textbf{22.27 }\\
& $k_{ap}=0.1$ & \textbf{33.08} & 5.78 & \textbf{21.22} & \textbf{29.05} & \textbf{3.44} & \textbf{18.95} \\
& $k_{ap}=0.01$ & \textbf{20.43} & \textbf{2.50} & \textbf{15.42} & \textbf{23.83} & \textbf{4.29} & \textbf{19.36} \\
& PID & 38.40 & 4.84 & 21.96 & 44.60 & 7.04 & 26.15 \\
& Lag & 39.88 & 5.38 & 23.31 & 87.33 & 37.36 & 61.53 \\
\midrule
\multirow{5}{*}{CPPO} 
& $k_{ap}=1$ & 86.08 & 24.38 & 48.31 & \textbf{69.98} & \textbf{18.62} & \textbf{39.87} \\
& $k_{ap}=0.1$ & \textbf{16.25} & \textbf{4.05} & \textbf{13.46} & \textbf{29.05} & \textbf{3.44} & \textbf{18.95} \\
& $k_{ap}=0.01$ &\textbf{ 42.83} & \textbf{5.75} &\textbf{ 23.56 }& \textbf{23.83} & \textbf{4.29} & \textbf{19.36} \\
& PID & 67.28 & 12.80 & 34.67 & 79.25 & 23.88 & 46.44 \\
& Lag & 46.43 & 6.67 & 25.90 & 84.35 & 30.16 & 53.38 \\
\bottomrule
\end{tabular*}
}

\label{tab:kap_results}
\end{table*}

As shown in Table~\ref{tab:kap_results}, we report the \textbf{Violation Rate (Vio. Rate)}, the \textbf{Magnitude} of constraint violations, and the \textbf{Average Cost (Avg. Cost)} for both the \textit{CarPush} and \textit{RacecarGoal} environments using the \textit{TRPO} and \textit{PPO} algorithms. The results demonstrate the superior performance of our ADRC approach with varying $k_{ap}$ values compared to baseline methods (PID and Lag). Specifically:
\begin{itemize}
    \item For the \textbf{CarPush} environment:
    \begin{itemize}
        \item Under \textbf{TRPO}, the configuration $k_{ap}=0.01$ achieves the lowest violation rate (20.43\%) and magnitude (2.50), alongside a significantly reduced average cost (15.42), outperforming both PID and Lag.
        \item For \textbf{CPPO}, $k_{ap}=0.1$ shows remarkable results, with a violation rate of 16.25\%, the smallest magnitude (4.05), and the lowest average cost (13.46). This highlights the adaptability of ADRC at this gain level.
    \end{itemize}
    \item For the \textbf{RacecarGoal} environment:
    \begin{itemize}
        \item With \textbf{TRPO}, $k_{ap}=0.01$ again demonstrates the best performance, achieving a violation rate of 23.83\%, a moderate magnitude (4.29), and a reduced average cost (19.36). This represents a clear improvement over both PID and Lag methods.
        \item Similarly, under \textbf{CPPO}, $k_{ap}=0.1$ achieves the best performance with a violation rate of 29.05\%, the smallest magnitude (3.44), and the lowest average cost (18.95). These results further emphasize ADRC's effectiveness.
    \end{itemize}
    \item For both environments, the baseline PID and Lag methods generally exhibit higher violation rates, magnitudes, and costs. Lag in particular performs poorly, especially in the \textit{RacecarGoal} environment, where it yields the highest violation rates and costs.
\end{itemize}

These results confirm that our ADRC method, with its dynamic control gain adjustments, consistently outperforms traditional methods, particularly when $k_{ap}=0.01$ or $k_{ap}=0.1$, demonstrating its robustness and adaptability across diverse environments and algorithms. 
\begin{figure}[htbp]
    \centering
    \includegraphics[width=\linewidth]{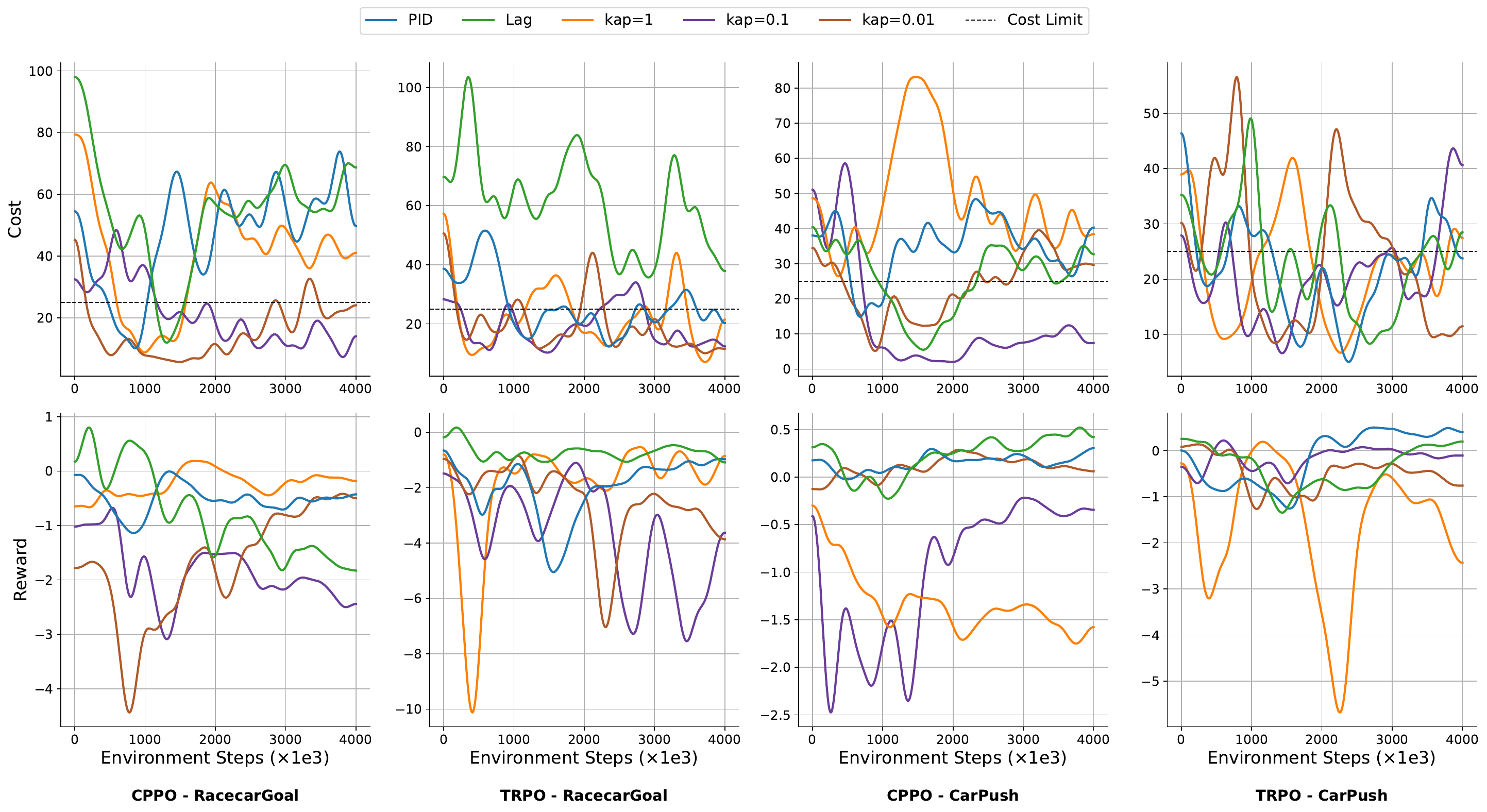}
    \caption{The training curve for TRPO and CPPO algorithms across CarPush and RacecarGoal environments with various $k_{ap}$ values, PID, and Lag methods.}
    \label{fig:kap}
\end{figure}

Figure~\ref{fig:training_c_r} provides the training curves for reward and cost across the evaluated $k_{ap}$ values, PID, and Lag methods. These curves illustrate the consistent performance improvements of our method throughout the training process. Our approach not only converges more effectively but also demonstrates a more favorable trade-off between reward maximization and cost minimization.
\subsubsection{Tuning parameter $k_{ad}$}
To assess the effect of the control gain $k_{ad}$ on the performance of ADRC-based Lagrangian methods, we conducted a series of ablation experiments. Specifically, we evaluated three distinct values of $k_{ad}$ ($0.01, 0.1, 1$) and compared them with existing approaches, including PID-based and classical Lagrangian methods. These experiments were carried out in two challenging environments, \textit{CarPush} and \textit{RacecarGoal}, using two reinforcement learning algorithms, \textit{CPPO} and \textit{TRPO}. The results highlight ADRC's ability to dynamically adjust the control gain, demonstrating superior adaptability and improved performance with carefully selected parameter settings.
\label{sec:kad}
To evaluate the impact of the tuning parameter $k_{ad}$ on ADRC Lagrangian methods' performance, we conducted ablation experiments by selecting three different values of $k_{ad}={0.01,0.1,1}$ and comparing them against existing methods, including PID Lagrangian methods and classical Lagrangian methods. The experiments were performed across two environments which are CarPush and RacecarGoal and adopt two algorithms which are CPPO and TRPO. 
\begin{table*}[htbp]
\centering
\scriptsize
\caption{The proportion of constraint violations during training (Vio. Rate), the average magnitude of violations (Magnitude), and the average cost (Avg. Cost) for TRPO and CPPO algorithms across CarPush and RacecarGoal environments with various $k_{ad}$ values, PID, and Lag methods. Bold values indicate better performance compared to PID.}
\resizebox{\textwidth}{!}{
\begin{tabular*}{\textwidth}{@{\extracolsep{\fill}}ccccc ccc@{}}
\toprule
\multirow[c]{2}{*}{\textbf{Algorithm}} & \multirow[c]{2}{*}{\textbf{Method}} & \multicolumn{3}{c}{\textbf{CarPush}} & \multicolumn{3}{c}{\textbf{RacecarGoal}} \\ 
\cmidrule(lr){3-5} \cmidrule(lr){6-8}
 & & \textbf{Vio. Rate (\%)} & \textbf{Magnitude} & \textbf{Avg. Cost} & \textbf{Vio. Rate (\%)} & \textbf{Magnitude} & \textbf{Avg. Cost} \\
\midrule
\multirow{5}{*}{TRPO} 
& $k_{ad}=1$ & \textbf{38.23} & 6.16 & 23.11 & \textbf{28.55} & \textbf{6.17} & \textbf{20.23} \\
& $k_{ad}=0.1$ & \textbf{36.68} & 6.29 & 21.99 & \textbf{38.25} & 9.12 & \textbf{23.92} \\
& $k_{ad}=0.01$ & \textbf{30.20} & \textbf{3.86} & \textbf{20.36} & \textbf{39.08} & \textbf{5.75 }& \textbf{23.33} \\
& PID & 38.40 & 4.84 & 21.96 & 44.60 & 7.04 & 26.15 \\
& Lag & 39.88 & 5.38 & 23.31 & 87.33 & 37.36 & 61.53 \\
\midrule
\multirow{5}{*}{CPPO} 
& $k_{ad}=1$ & \textbf{16.20} & \textbf{1.78} & \textbf{16.60} & \textbf{48.68} & \textbf{9.12} & \textbf{27.80} \\
& $k_{ad}=0.1$ & \textbf{15.08} & \textbf{5.06} & \textbf{15.45} & \textbf{48.55} & \textbf{8.59} & \textbf{27.80} \\
& $k_{ad}=0.01$ & \textbf{16.25} & \textbf{4.05} & \textbf{13.46} & \textbf{33.08} & \textbf{5.75} & \textbf{23.33} \\
& PID & 67.28 & 12.80 & 34.67 & 79.25 & 23.88 & 46.44 \\
& Lag & 46.43 & 6.67 & 25.90 & 84.35 & 30.16 & 53.38 \\
\bottomrule
\end{tabular*}
}

\label{tab:kad_results}
\end{table*}

As shown in Table~\ref{tab:tuning_kap_kad}, we report the \textbf{Violation Rate (Vio. Rate)}, the \textbf{Magnitude} of constraint violations, and the \textbf{Average Cost (Avg. Cost)} for both the \textit{CarPush} and \textit{RacecarGoal} environments using the \textit{TRPO} and \textit{CPPO} algorithms. The results highlight the superior performance of our ADRC approach with varying $k_{ad}$ values compared to the baseline methods (PID and Lag). Specifically:
\begin{itemize}
    \item For the \textbf{CarPush} environment:
    \begin{itemize}
        \item Under \textbf{TRPO}, $k_{ad}=0.01$ achieves the lowest violation rate (30.20\%) and the smallest magnitude (3.86), alongside a reduced average cost (20.36). This indicates better constraint satisfaction and efficiency compared to PID and Lag.
        \item For \textbf{CPPO}, $k_{ad}=0.1$ yields the best performance with the lowest violation rate (15.08\%), a moderate magnitude (5.06), and the smallest average cost (15.45). These results highlight ADRC's adaptability at this parameter setting.
    \end{itemize}
    \item For the \textbf{RacecarGoal} environment:
    \begin{itemize}
        \item With \textbf{TRPO}, $k_{ad}=1$ shows excellent performance, achieving the lowest violation rate (28.55\%) and average cost (20.23), alongside a relatively small magnitude (6.17).
        \item Under \textbf{CPPO}, $k_{ad}=0.01$ demonstrates the best results, with a low violation rate (33.08\%), a reduced magnitude (5.75), and the smallest average cost (23.33). This showcases ADRC's ability to manage constraints effectively in this challenging environment.
    \end{itemize}
    \item Across both environments, the baseline methods (PID and Lag) consistently exhibit higher violation rates, larger magnitudes, and higher costs. Lag performs particularly poorly, with significantly higher metrics, especially in the \textit{RacecarGoal} environment, where it records the highest violation rate (87.33\%) and average cost (61.53).
\end{itemize}

\begin{figure}[htbp]
    \centering
    \includegraphics[width=\linewidth]{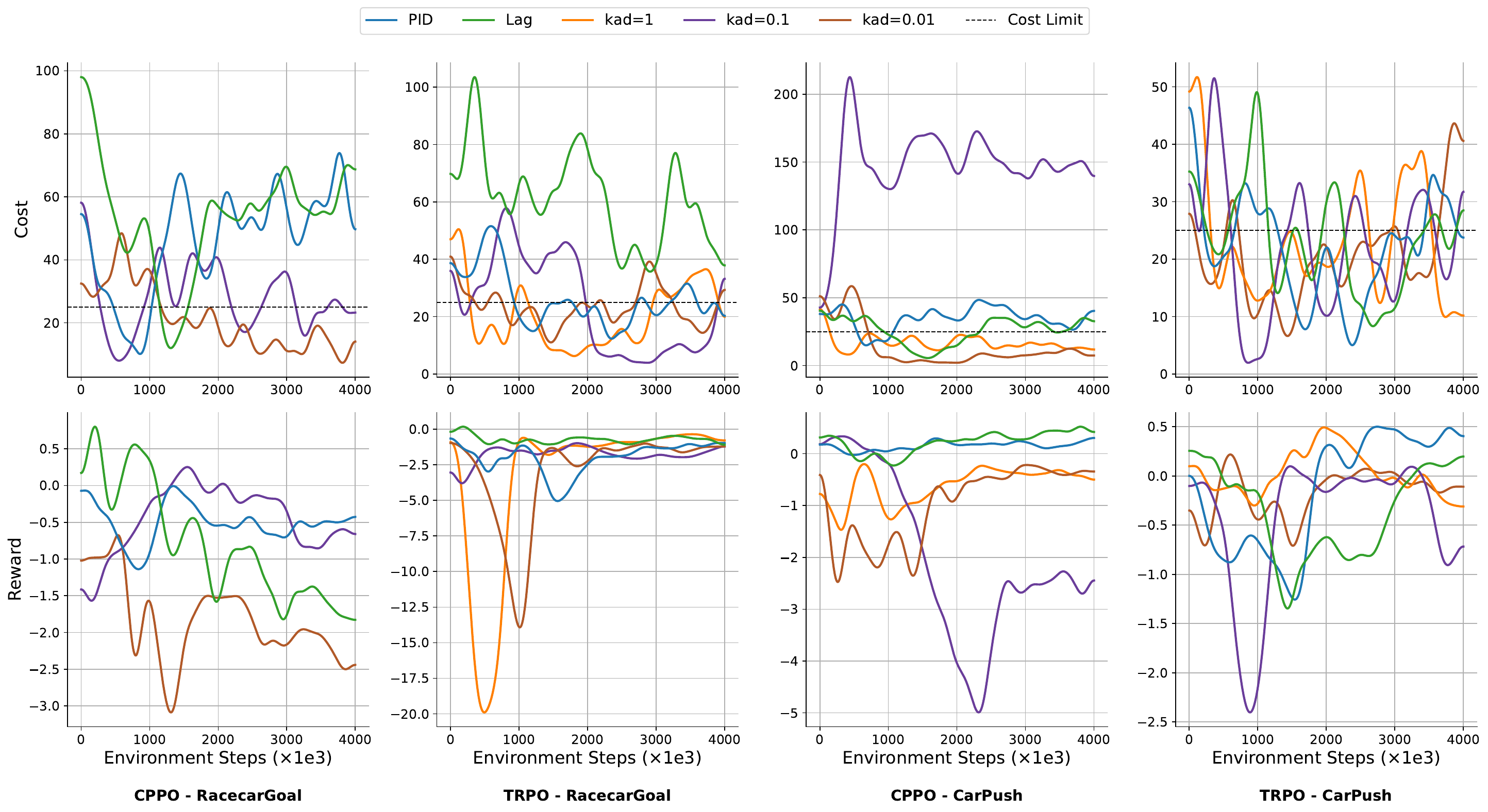}
    \caption{The training curve for TRPO and CPPO algorithms across CarPush and RacecarGoal environments with various $k_{ad}$ values, PID, and Lag methods.}
    \label{fig:kad}
\end{figure}

\subsubsection{Tuning parameter $c_{r}$}
\label{sec:appendix_tuning_c_r}

To evaluate the impact of the tuning parameter $c_r$ on ADRC Lagrangian methods' performance, we conducted ablation experiments by selecting five different values of $c_r={0.05,0.1,0.15,0.2,0.25}$ and comparing them against existing methods, including PID Lagrangian methods and classical Lagrangian methods. The experiments were performed across two environments which are CarPush and RacecarGoal and adopt two algorithms which are CPPO and TRPO. 

As shown in Table~\ref{tab:entire_c_r}, we report the \textbf{Violation Rate (Vio. Rate)}, the \textbf{Magnitude} of constraint violations, and the \textbf{Average Cost (Avg. Cost)}. The results demonstrate that our method consistently outperforms the baseline methods (PID Lagrangian methods and classical Lagrangian methods) across a majority of the $c_r$ values. Specifically:
\begin{itemize}
    \item For the {CarPush} environment, our method achieves lower violation rates and magnitudes in most cases, while maintaining competitive average costs.
    \item Similarly, in the {RacecarGoal} environment, our approach demonstrates significant improvements, particularly with $c_r = 0.1$ and $c_r = 0.2$, where it achieves the lowest violation rates and magnitudes.
\end{itemize}

\begin{table*}[htbp]
\centering
\scriptsize
\caption{The proportion of constraint violations during training (Vio. Rate), the average magnitude of violations (Magnitude), and the average cost (Avg. Cost) for TRPO and CPPO algorithms across CarPush and RacecarGoal environments with various $c_r$ values, PID, and Lag methods. Bold values indicate the better performance compared to PID.}
\resizebox{\textwidth}{!}{
\begin{tabular*}{\textwidth}{@{\extracolsep{\fill}}ccccc ccc@{}}
\toprule
\multirow[c]{2}{*}{\textbf{Algorithm}} & \multirow[c]{2}{*}{\textbf{Method}} & \multicolumn{3}{c}{\textbf{CarPush}} & \multicolumn{3}{c}{\textbf{RacecarGoal}} \\ 
\cmidrule(lr){3-5} \cmidrule(lr){6-8}
 & & \textbf{Vio. Rate (\%)} & \textbf{Magnitude} & \textbf{Avg. Cost} & \textbf{Vio. Rate (\%)} & \textbf{Magnitude} & \textbf{Avg. Cost} \\
\midrule
\multirow{7}{*}{TRPO} 
& Lag & 43.83 & 7.99 & 26.19 & 87.33 & 37.36 & 61.53 \\
& PID & 38.40 & 4.84 & 21.96 & 44.60 & 7.04 & 26.15 \\
& $c_r=0.05$ & 46.25 & 6.35 & 24.79 & \textbf{33.98} & \textbf{5.25} & \textbf{20.83} \\
& $c_r=0.1$ & \textbf{30.20} & \textbf{3.86} & \textbf{20.36} & \textbf{29.05} & \textbf{3.44} & \textbf{18.95} \\
& $c_r=0.15$ & \textbf{34.83} & 10.92 & 27.00 & \textbf{31.25} & \textbf{5.34} & \textbf{21.16} \\
& $c_r=0.2$ & \textbf{28.60} & 7.32 & \textbf{20.72} & \textbf{40.65} & \textbf{6.10} & \textbf{23.67} \\
& $c_r=0.25$ & \textbf{34.13} & 6.26 & 22.46 & \textbf{38.50} & \textbf{6.71} & \textbf{23.40} \\
\midrule
\multirow{7}{*}{CPPO} 
& Lag & 46.43 & 6.67 & 25.90 & 84.35 & 30.16 & 53.38 \\
& PID & 67.28 & 12.80 & 34.67 & 79.25 & 23.88 & 46.44 \\
& $c_r=0.05$ & \textbf{44.73} & \textbf{5.95} & \textbf{24.80} & \textbf{34.38} & \textbf{3.90} & \textbf{22.45} \\
& $c_r=0.1$ & \textbf{16.25} & \textbf{4.05} & \textbf{13.46} & \textbf{33.08} & \textbf{5.78} & \textbf{21.22} \\
& $c_r=0.15$ & \textbf{56.70} & \textbf{10.22} & \textbf{30.40} & \textbf{52.88} & \textbf{10.69} & \textbf{31.37} \\
& $c_r=0.2$ & \textbf{12.30} & \textbf{2.31} & \textbf{13.45} & \textbf{48.95} & \textbf{8.77} & \textbf{26.91 }\\
& $c_r=0.25$ & \textbf{34.00} & \textbf{6.12} & \textbf{22.09} & \textbf{62.83} & \textbf{13.37}& \textbf{33.99} \\
\bottomrule
\end{tabular*}
}

\label{tab:entire_c_r}
\end{table*}

Figure~\ref{fig:training_c_r} provides the training curves for reward and cost across the evaluated $c_r$ values, PID, and Lag methods. These curves illustrate the consistent performance improvements of our method throughout the training process. Our approach not only converges more effectively but also demonstrates a more favorable trade-off between reward maximization and cost minimization.

Overall, the results highlight the robustness of our method, as it achieves superior performance in the majority of scenarios. This indicates that the choice of $c_r$ significantly influences the balance between reward and cost, and our approach consistently outperforms existing methods under comparable conditions.
\begin{figure*}[htbp]
    \centering
    \includegraphics[width=\linewidth]{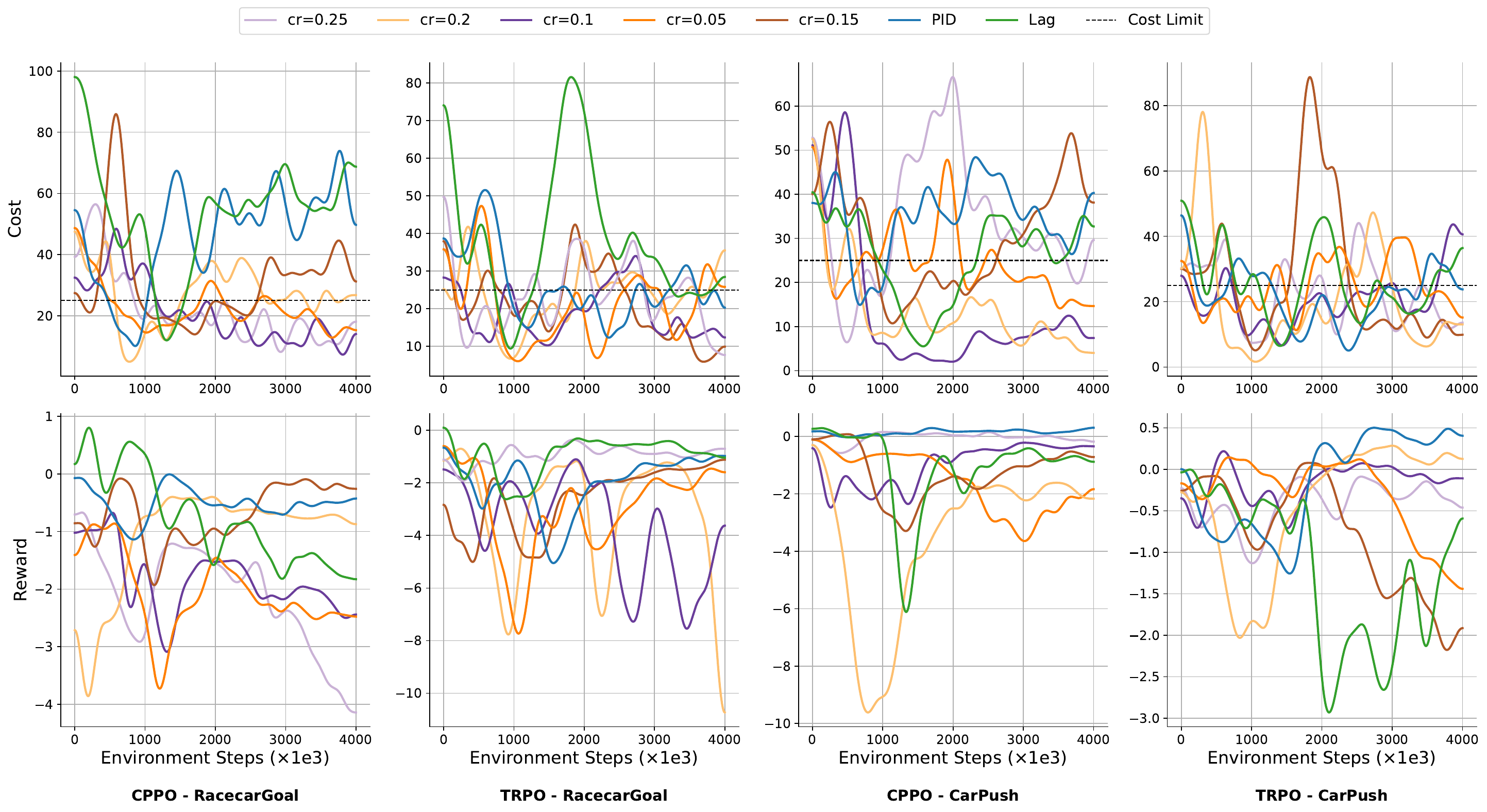}
    \caption{The training curve for TRPO and CPPO algorithms across CarPush and RacecarGoal environments with various $c_r$ values, PID, and Lag methods.}
    \label{fig:training_c_r}
\end{figure*}

\subsection{Noise Sensitivity Analysis}
\label{sec:noise_sensitivity}

\begin{table}[htbp]
\centering
\small
\caption{Training noise sensitivity on Swimmer-Velocity.}
\begin{tabular}{lcccc}
\toprule
\textbf{Method / $\sigma$} & \textbf{Vio. Rate (\%)} & \textbf{Magnitude} & \textbf{Avg. Cost} & \textbf{Avg. Reward} \\
\midrule
PID       & 35.43 & 1.77 & 22.47 & 27.73 \\
$\sigma=0$  & 11.30 & 2.44 & 20.82 & 35.66 \\
$\sigma=2$  & 10.93 & 2.55 & 21.93 & 22.41 \\
$\sigma=5$  &  \textbf{5.93} & 1.62 &  \textbf{8.04} & 16.18 \\
$\sigma=10$ & 23.95 & 3.55 & 19.93 & 12.41 \\
\bottomrule
\end{tabular}
\label{tab:train_noise_swimmer}
\end{table}

\begin{table}[htbp]
\centering
\small
\caption{Evaluation noise sensitivity on Swimmer-Velocity.}
\begin{tabular}{lcccc}
\toprule
\textbf{Method / $\sigma$} & \textbf{Avg. Cost} & \textbf{Vio. Rate (\%)} & \textbf{Magnitude} & \textbf{Avg. Reward} \\
\midrule
PID  & 23.05 & 39.79 & 6.23  & 32.19 \\
$\sigma=0$  & 18.40 & 12.07 & 11.35 & 39.55 \\
$\sigma=2$  & 18.00 & 27.60 & 14.52 & 19.72 \\
$\sigma=5$  &  \textbf{4.16} & 4.57  & 9.18  & 17.23 \\
$\sigma=10$ & 18.20 & 28.72 & 15.06 & 20.51 \\
\bottomrule
\end{tabular}
\label{tab:eval_noise_swimmer}
\end{table}

Tables~\ref{tab:train_noise_swimmer}--\ref{tab:eval_noise_swimmer} report the sensitivity of our TRPO-based ADRC approach to injected noise in the \textsc{Swimmer}-Velocity environment. The results demonstrate that TRPO-ADRC is robust to disturbances: compared with PID-type controllers, TRPO-ADRC achieves consistently lower violation rates, smaller violation magnitudes, and reduced average costs, while maintaining stable reward learning. For example, during training, TRPO-ADRC reduces the violation rate from 35.43\% (PID) to 11.30\% under $\sigma{=}0$, and further to only 5.93\% under $\sigma{=}5$, with the average cost dropping from 22.47 to 8.04. 

Importantly, these safety and stability benefits do not compromise convergence performance. At evaluation, TRPO-ADRC achieves competitive or even higher rewards while retaining robustness. Under $\sigma{=}0$, TRPO-ADRC attains higher reward than TRPO-PID (39.55 vs.\ 32.19) while simultaneously lowering both cost and violation rate. Even when the disturbance level increases ($\sigma{=}5,10$), TRPO-ADRC sustains reasonable rewards with substantially reduced safety violations compared to PID. These findings confirm that TRPO-ADRC improves training stability and safety without hindering convergence, demonstrating strong robustness to noise in the \textsc{Swimmer}-Velocity environment.

\subsection{Ablation Study}
\label{sec:ablation_in_appendix}
To evaluate the contribution of key components in the ADRC-Lagrangian framework, we conducted an ablation study by systematically modifying specific features of the proposed method. Specifically, we examined the impact of replacing the transient process \(r(t)\) with a static reference signal and fixing the dynamically adjusted compensation gain \(\omega_o\) to a constant value. These modifications simplify the framework to a configuration resembling a PID-based Lagrangian method, enabling a fair comparison of their relative contributions.

In the first ablation, the transient process \(r(t)\) is replaced with a fixed reference signal \(r(t) = d\), as used in traditional PID Lagrangian methods. While \(r(t)\) is designed in the full ADRC framework to provide a smooth transition toward the cost threshold denoted as $d$, setting \(r(t) = d\) eliminates this smoothing effect. This simplification forces the system to directly track the constant reference signal, potentially causing abrupt updates to the policy parameters and destabilizing training. The updating law is transformed into:
\begin{equation}
\lambda_t = k_{ap}(x_1 - d) + k_{ad}x_2 + \omega_o k_{ap} \int_0^t(x_1(\tau) - d) d\tau.
\end{equation}

For the second ablation, we fixed the compensation gain \(\omega_o\) to ensure its parameters match those of the PID baseline. Specifically, we solved the following equations to determine fixed values for \(\omega_o\), \(k_{ap}\), and \(k_{ad}\):
\begin{equation}
\label{eq:ablated_params}
\begin{aligned}
k_{ap} + \omega_o k_{ad} &= k_p, \\
\omega_o + k_{ad} &= k_d, \\
\omega_o k_{ap} &= k_i.
\end{aligned}
\end{equation}
The solutions to these equations yield parameter values that maintain equivalence with the PID updating law while removing the adaptivity of \(\omega_o\). With these parameters, the updating law for the Lagrangian multiplier \(\lambda_t\) reduces to:
\begin{equation}
\label{eq:ablated_law}
\lambda_t = k_p(x_1 - r) + k_{d}(x_2 - \dot{r}) + k_i \int_0^t(x_1(\tau) - r(\tau)) d\tau - \ddot{r}.
\end{equation}

The ablation experiments evaluate the importance of the dynamic transient process \(r(t)\) and the adaptive compensation gain \(\omega_o\) in the ADRC-Lagrangian framework. Specifically, we tested two simplified configurations: "delete\_r(t)," where the transient process \(r(t)\) is replaced with a static reference signal \(r(t) = d\), and "delete\_$\omega_o$," where the dynamic adjustment of \(\omega_o\) is replaced with fixed parameter values derived from PID-based methods. These modifications isolate the contributions of each component while retaining equivalent control parameters. The experiments were conducted in the CarButton and RacecarGoal environments using TRPO and CPPO as base algorithms, with the results summarized in Table~\ref{tab:ablation_results}.

\begin{table*}[htbp]
\centering
\scriptsize
\caption{The proportion of constraint violations during training (Vio. Rate), the average magnitude of violations (Magnitude), and the average cost (Avg. Cost) for TRPO and CPPO algorithms across CarButton and RacecarGoal environments with various methods. Bold values indicate better performance compared to PID.}
\resizebox{\textwidth}{!}{
\begin{tabular*}{\textwidth}{@{\extracolsep{\fill}}ccccc ccc@{}}
\toprule
\multirow[c]{2}{*}{\textbf{Algorithm}} & \multirow[c]{2}{*}{\textbf{Method}} & \multicolumn{3}{c}{\textbf{CarButton}} & \multicolumn{3}{c}{\textbf{RacecarGoal}} \\ 
\cmidrule(lr){3-5} \cmidrule(lr){6-8}
 & & \textbf{Vio. Rate (\%)} & \textbf{Magnitude} & \textbf{Avg. Cost} & \textbf{Vio. Rate (\%)} & \textbf{Magnitude} & \textbf{Avg. Cost} \\
\midrule
\multirow{5}{*}{TRPO} 
& Lag & 55.90 & 17.65 & 39.28 & 87.33 & 37.36 & 61.53 \\
& PID & 85.15 & 18.12 & 42.13 & 44.60 & 7.04 & 26.15 \\
& delete\_$\omega_o$ & \textbf{71.85} & \textbf{9.42} & \textbf{32.56} & \textbf{40.53} & \textbf{7.29} & \textbf{26.14} \\
& delete\_r(t) & \textbf{74.95} & \textbf{16.36} & \textbf{40.23} & \textbf{31.65} & \textbf{6.15} & \textbf{22.26} \\
& ADRC & \textbf{26.35} & \textbf{6.14} & \textbf{23.42} & \textbf{29.05} & \textbf{3.44} & \textbf{18.95} \\
\midrule
\multirow{5}{*}{CPPO} 
& Lag & 99.88 & 98.34 & 123.31 & 84.35 & 30.16 & 53.38 \\
& PID & 99.88 & 57.95 & 82.92 & 79.25 & 23.88 & 46.44 \\
& delete\_$\omega_o$ & 99.90 & \textbf{51.61} & \textbf{76.58} & \textbf{65.40} & \textbf{13.99} & \textbf{36.38} \\
& delete\_r(t) & \textbf{92.53} & \textbf{33.78} & \textbf{58.20} & \textbf{54.08} & \textbf{15.23} & \textbf{34.66} \\
& ADRC & \textbf{93.68} & \textbf{33.97} & \textbf{58.74} & \textbf{33.08} & \textbf{5.78} & \textbf{21.22} \\
\bottomrule
\end{tabular*}
}

\label{tab:ablation_results}
\end{table*}

Replacing \(r(t)\) with a static reference signal ("delete\_r(t)") resulted in higher violation rates and magnitudes compared to the full ADRC framework but still outperformed the PID baseline. For instance, in the RacecarGoal environment with CPPO, the violation rate decreased from 79.25\% (PID) to 54.08\% ("delete\_r(t)"), but ADRC achieved a further reduction to 33.08\%. Similarly, fixing \(\omega_o\) ("delete\_w0") impaired the system’s adaptability to environmental changes, leading to increased average costs. In the same environment, the average cost dropped from 46.44 (PID) to 36.38 ("delete\_w0") but was substantially lower with ADRC at 21.22. These results demonstrate that while both components are essential for optimal performance, even the simplified versions outperform PID, highlighting the robustness of the ADRC-Lagrangian framework.

\begin{figure}
    \centering
    \includegraphics[width=\linewidth]{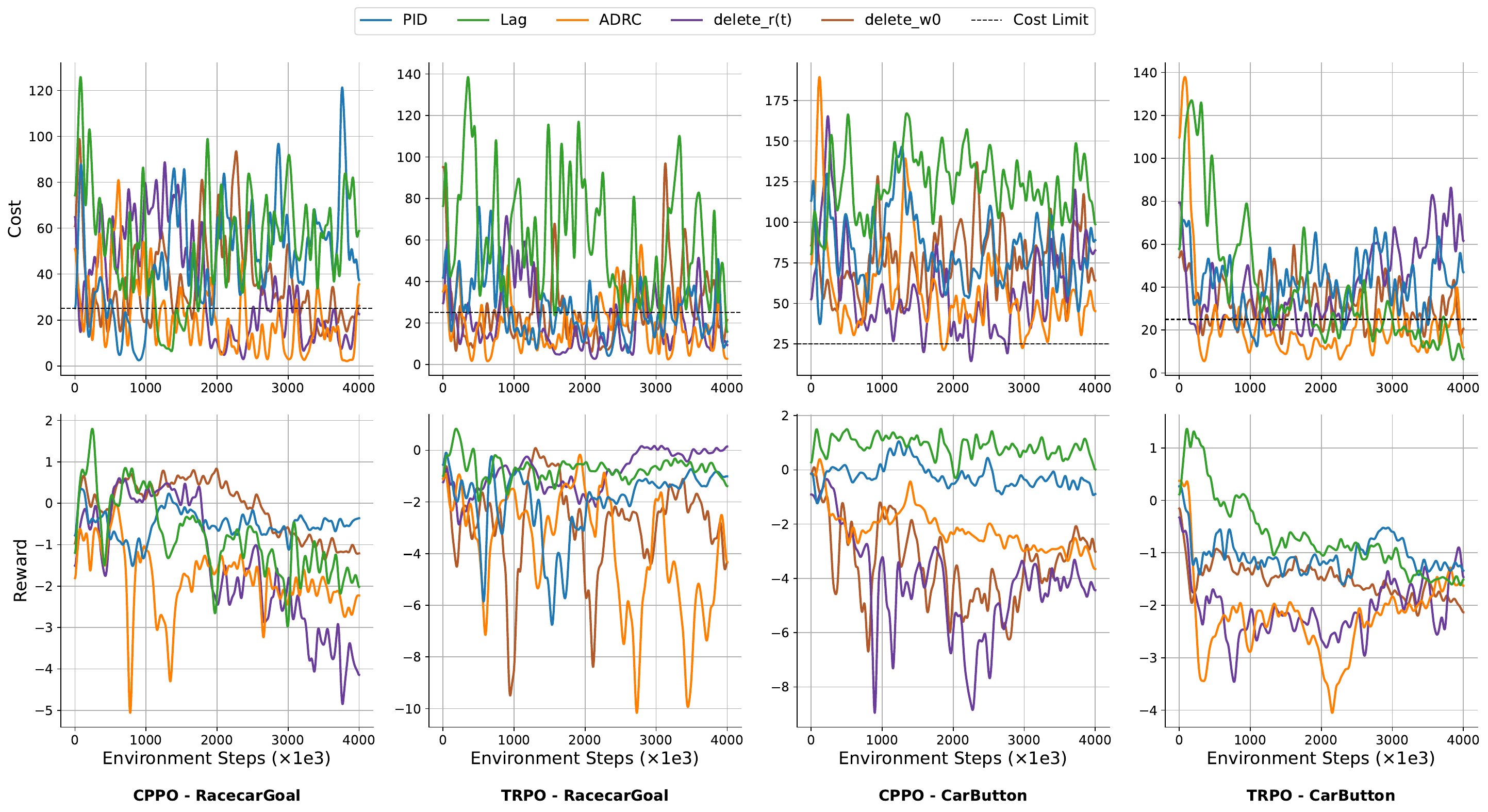}
    \caption{The training curve for TRPO and CPPO algorithms across CarButton and RacecarGoal environments with ablation study.}
    \label{fig:ablation_study}
\end{figure}

Figure~\ref{fig:ablation_study} provides the training curves for reward and cost across the ablation study. These curves illustrate the consistent performance improvements of our method throughout the training process. Our approach not only converges more effectively but also demonstrates a more favorable trade-off between reward maximization and cost minimization.

\subsubsection{Additional Ablation on CPPO}
To evaluate the effectiveness of our proposed dynamic parameter adjustment and transient process, we conducted ablation studies, with results summarized in Table~\ref{tab:ablation_CarPush_CPPO}. In this table, “Delete $r(t)$” refers to the removal of the dynamic adjustment component $r(t)$, while “Delete $\omega_o$” refers to the exclusion of the transient weight $\omega_o$ from the algorithm. The results show that removing either component results in a clear performance degradation in terms of violation rate, violation magnitude, and average cost. However, even with these removals, the performance of our approach remains superior to the baseline PID method, demonstrating the robustness of our framework. Additionally, the complete ADRC method achieves the best results across all metrics, further highlighting the significance of combining both $r(t)$ and $\omega_o$ in achieving optimal performance. For further details and results, please refer to Appendix~\ref{sec:ablation_in_appendix}.
\begin{table}[H]
\centering
\small
\caption{Ablation study of CPPO algorithm under RacecarGoal.}
\begin{tabular}{lccc}
\toprule
\textbf{Method} & \textbf{Vio. Rate(\%)} & \textbf{Magnitude} & \textbf{Avg. Cost} \\
\midrule
Delete $r(t)$ & 65.40 & 13.99 & 36.38 \\
Delete $\omega_o$  & 54.08 & 15.23 & 34.66 \\
ADRC    & \textbf{33.08} & \textbf{5.78} & \textbf{21.22  } \\
\midrule
PID           & 79.25 & 23.88 & 46.44 \\
Lag           & 84.35 & 30.16 & 53.38 \\
\bottomrule
\end{tabular}
\label{tab:ablation_CarPush_CPPO}
\end{table}


\subsection{Case Study} 

To gain a deeper understanding of how our ADRC Lagrangian methods outperform the baseline, we conduct a case study adopting the TRPO algorithm in the CarCircle-1 environment. In this environment, agents are tasked with navigating around a fixed-radius circle. The agents' goal is to maintain a smooth circular trajectory while staying within the designated circular boundary and avoiding collisions with obstacles.

The reward structure in the CarCircle-1 environment is designed to encourage agents to follow the circle boundary as closely as possible while maintaining a smooth motion. High rewards are achieved when the agent's trajectory aligns with the circle's radius, and its velocity vector aligns tangentially to the circular path. The reward for the agent is calculated using the following formula:
\begin{equation}
\label{eq:circle_reward}
\text{Reward} = \frac{\frac{-u \cdot y + v \cdot x}{r}}{1 + \left|r - R\right|} \cdot \text{reward\_factor},    
\end{equation}
where \(x, y\) represent the agent's position, \(u, v\) represent the agent's velocity, \(r\) is the distance from the agent to the circle center (\(r = \sqrt{x^2 + y^2}\)), \(R\) is the fixed radius of the circle, and \(\text{reward\_factor}\) is a scaling constant. This formula incentivizes agents to maintain a smooth and stable trajectory around the circle. 

To increase the complexity of the environment, CarCircle-1 introduces two vertical walls positioned symmetrically near the circle's boundary. These walls present an additional challenge, as agents must avoid crossing into the wall regions while navigating the circle. Costs are incurred when the agent violates safety constraints, such as exceeding the circle's radius or crossing the boundaries defined by the walls. The cost is computed using the following conditions:
\begin{equation}
\label{eq:cost_circle}
\text{Cost} = 
\begin{cases} 
1 & \text{if } |x| > \text{wall\_threshold} \text{ or } \sqrt{x^2 + y^2} > R, \\
0 & \text{otherwise},
\end{cases}    
\end{equation}
where \(\text{wall\_threshold}\) is the horizontal boundary defined by the walls, and \(R\) is the circle's radius.

In this study, we adopt the default settings in OmniSafe \citep{omnisafe}, with $R = 1.5$, \(\text{wall\_threshold} = 1.125\), and \(\text{reward\_factor} = 0.1\). The study is conducted using the TRPO algorithm, with hyperparameters detailed in Appendix~\ref{sec:parameters}. The models are trained over 4,000 episodes, with each episode consisting of 1,000 steps. After training, the final checkpoint is used for evaluation. During evaluation, we simulate a single episode of 5,000 steps, where rewards are calculated using Eqn.~\ref{eq:circle_reward}, and costs are determined by Eqn.~\ref{eq:cost_circle}. We collected the position ant the corrsponding reward and the cost of the agents and the results of this case study are presented in Figure~\ref{fig:case_study_circle}.

\begin{figure}[htbp]
    \centering
    \includegraphics[width=\linewidth]{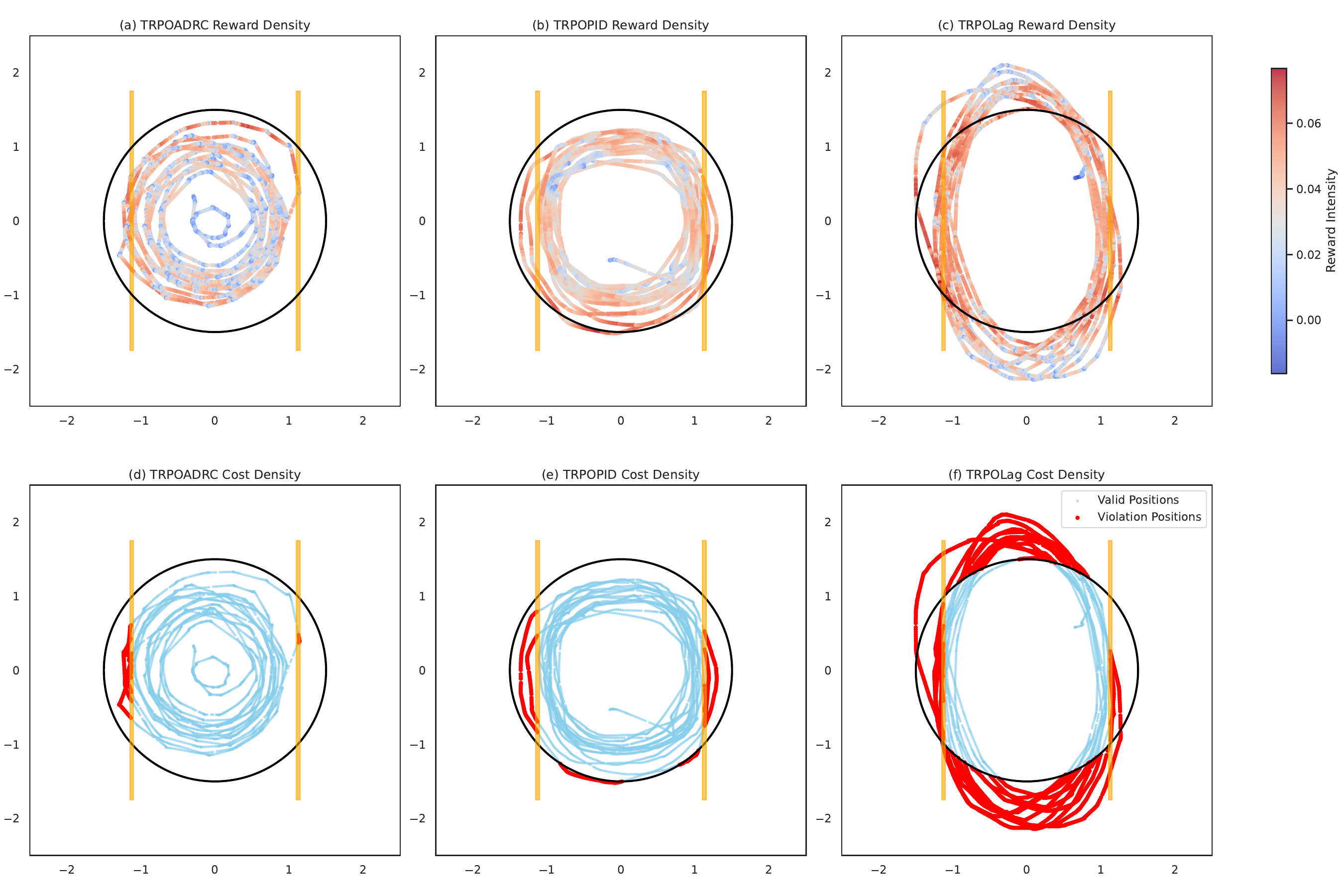}
    \caption{Reward density and cost analysis for TRPO algorithm adapting three Lagrangian methods under CarCircle environment.}
    \label{fig:case_study_circle}
\end{figure}

In Figure~\ref{fig:case_study_circle}, the color of the points represents the reward, with deeper colors indicating higher rewards. The first row visualizes the reward density at each position during the episode, while the second row illustrates the cost associated with each position. Blue points indicate safe positions, whereas red points represent unsafe ones. The results show that the classical Lagrangian method achieves the highest reward, but the trajectory deviates from a perfect circle, forming an ellipse instead. Agents trained with this method learn to avoid the walls but fail to recognize the importance of staying within the circle. To maximize rewards and bypass the walls, these agents move outside the circle, resulting in a total of 496 safety violations. In contrast, agents trained with the PID Lagrangian method demonstrate better safety awareness, recognizing that moving outside the circle is unsafe. However, to achieve higher rewards, they still cross the walls frequently, leading to 320 safety violations. Finally, agents trained with our ADRC method maintain a strict adherence to staying within the circle and exhibit only 132 safety violations, establishing a superior safety performance.

The improved results achieved by the ADRC method can be attributed to its ability to reduce phase lag and minimize oscillations, thereby enhancing the stability of training. These properties enable ADRC to maintain better control over the agent’s behavior, ensuring stricter adherence to safety constraints while still optimizing for rewards. The stability and precision offered by ADRC during training allow the agent to effectively balance the trade-off between maximizing rewards and minimizing safety violations, demonstrating the advantages of our proposed method in challenging environments.

\subsection{Final Policy Performance}
\label{sec:converge_performance}
To assess the final performance of the trained policies rather than intermediate training behavior, we conducted experiments on the Swimmer and Hopper environments from the Velocity tasks suite. The results compare ADRC-based and PID-based methods under CPPO and TRPO frameworks.

\begin{table}[h]
\centering
\caption{Performance on Swimmer environment.}
\begin{tabular}{lccc}
\toprule
Algorithm & Avg Reward & Avg Cost & Violate Rate (\%) \\
\midrule
CPPOPID & 30.10 & 22.44 & 28.02 \\
CPPOADRC & 29.39 & 16.77 & 14.16 \\
TRPOPID & 28.72 & 21.34 & 37.85 \\
TRPOADRC & 36.32 & 19.03 & 12.16 \\
\bottomrule
\end{tabular}
\label{tab:velocity_swimmer_appendix}
\end{table}

\begin{table}[h]
\centering
\caption{Performance on Hopper environment.}
\begin{tabular}{lccc}
\toprule
Algorithm & Avg Reward & Avg Cost & Violate Rate (\%) \\
\midrule
CPPOPID & 1466.47 & 48.20 & 30.00 \\
CPPOADRC & 1520.18 & 8.20 & 24.63 \\
TRPOPID & 1038.47 & 18.70 & 29.76 \\
TRPOADRC & 1384.11 & 12.90 & 10.98 \\
\bottomrule
\end{tabular}
\label{tab:velocity_hopper}
\end{table}

These results demonstrate that the ADRC-based method achieves lower constraint violation rates and costs, while maintaining or improving the overall reward compared to PID-based baselines.

\subsection{Sensitivity Analysis of PID Lagrangian Methods}

We also empirically validated the sensitivity of PID Lagrangian methods to the control gain tuning, particularly for the derivative term $k_d$. Experiments were conducted in the CarPush and CarButton environments using CPPO algorithms with varying $k_d$ values.

\begin{table*}[htbp]
\centering
\small
\caption{Sensitivity analysis of PID Lagrangian methods by varying the derivative gain $k_d$ on CarPush and CarButton environments (using CPPO).}
\resizebox{0.9\textwidth}{!}{
\begin{tabular*}{\textwidth}{@{\extracolsep{\fill}}lccccc@{}}
\toprule
\textbf{Environment} & \textbf{$k_d$ Value} & \textbf{Violate Rate (\%)} & \textbf{Violate Magnitude} & \textbf{Avg Cost} & \textbf{Avg Reward} \\
\midrule
\multirow{4}{*}{CarPush (CPPO)}
    & 1     & 50.55 & 19.17 & 35.28 & -1.37 \\
    & 0.1   & 66.00 & 24.87 & 46.81 & -2.26 \\
    & 0.01  & 67.28 & 12.80 & 34.67 & 0.15 \\
    & 0.001 & \textbf{99.80} & \textbf{96.50} & \textbf{121.45} & -0.02 \\
\midrule
\multirow{3}{*}{CarButton (CPPO)}
    & 1     & \textbf{99.88} & \textbf{89.36} & \textbf{114.33} & 0.21 \\
    & 0.1   & 99.80 & 43.32 & 68.29 & -1.89 \\
    & 0.01  & 99.88 & 64.74 & 89.72 & -0.69 \\
\bottomrule
\end{tabular*}
}
\label{tab:sensitivity_pid}
\end{table*}

These results clearly illustrate that PID Lagrangian methods are highly sensitive to the choice of the $k_d$ value. Suboptimal tuning can lead to substantial degradation in both safety and overall performance.

\section{Large Language Models Usage}

We used GPT\mbox{-}5 (OpenAI) for grammar and style editing of the paper and for debugging auxiliary code (e.g., resolving error messages and minor refactoring). All technical ideas, method designs, experiments, and conclusions were created and verified by the authors. No confidential or reviewer-only information was shared with the model.

\section{Computational Cost Analysis}
\label{sec:computational_cost}
In this section, we evaluate the computational cost of our ADRC Lagrangian method compared to the PID and classical Lagrangian methods (Lag). The evaluation includes normalized computation time during the rollout phase (interaction with the environment) and the update phase (policy updates). The analysis is conducted using the DDPG and CPPO algorithms on the RacecarButton task, which features a multi-dimensional action space.

\begin{table*}[h!]
\centering
\caption{Normalized Computation Time for DDPG and CPPO under RacecarButton task.}
\begin{tabular}{lcccccc}
\hline
\textbf{Metric}       & \multicolumn{3}{c}{\textbf{DDPG}} & \multicolumn{3}{c}{\textbf{CPPO}} \\
\cline{2-4} \cline{5-7}
                      & \textbf{PID} & \textbf{Lag} & \textbf{ADRC} & \textbf{PID} & \textbf{Lag} & \textbf{ADRC} \\
\hline
Rollout Time          & 0.95         & 1.00         & 0.95          & 0.95         & 1.00         & 0.95         \\
Update Time           & 0.94         & 1.00         & 1.00          & 0.94         & 1.00         & 1.00         \\
Total Time       & 0.95         & 1.00         & 0.97          & 0.95         & 1.00         & 0.97         \\
\hline
\end{tabular}

\label{tab:normalized_times}
\end{table*}

Compared with PID, at each episode, our ADRC method introduces minimal additional computation. Specifically, ADRC calculates the reference signal $r(t)$, solves the equation defined by Eqn.\ref{eq:manifold}, and determines the optimal parameters $\omega_o$ based on Eqn.\ref{eq:lower_bound_of_omega}. These operations involve fixed and lightweight calculations that do not scale with the problem size, ensuring no additional time complexity is introduced.

The results in Table~\ref{tab:normalized_times} confirm that ADRC achieves comparable computation times to the baselines. The rollout time of ADRC matches that of PID and slightly outperforms Lag, demonstrating efficiency in policy adjustments. During the update phase, ADRC incurs no additional cost compared to PID and Lag, aligning with the theoretical analysis that these computations are efficiently integrated into the training process.

All experiments were conducted on a machine equipped with an NVIDIA RTX 3090 GPU with 24GB of memory. Each task was trained for 4 million steps. For TRPO and CPPO, each training run takes approximately 10 GPU-hours, while DDPG and TD3 require about 18 GPU-hours per run.

\end{document}